\documentclass[10pt]{article}%
\usepackage[margin=1in]{geometry}
\usepackage{microtype}
\usepackage{graphicx}

\usepackage{amsmath,amsfonts,bm,amsthm,amssymb,mathrsfs,algorithm,algorithmic,xcolor,url}
\usepackage{hyperref,url,booktabs,enumitem,multicol,graphicx,caption,subcaption,pifont,tikz,listings,wrapfig,arydshln}

\hypersetup{
    colorlinks,
    linkcolor={blue!50!black},
    citecolor={blue!50!black},
    urlcolor={red!50!black}
}

\theoremstyle{plain}
\newtheorem{theorem}{Theorem} 

\newtheorem{lemma}[theorem]{Lemma}

\theoremstyle{definition}
\newtheorem{definition}[theorem]{Definition}
\newtheorem{assumption}[theorem]{Assumption}
\theoremstyle{remark}
\newtheorem{remark}[theorem]{Remark}
\newtheorem*{theorem*}{Theorem}

\renewcommand{\algorithmicrequire}{\textbf{Input:}}
\renewcommand{\algorithmicensure}{\textbf{Output:}}

\def\Secref#1{Section~\ref{#1}}

\def\eqref#1{(\ref{#1})}

\def\1{\bm{1}}

\def\vzero{{\bm{0}}}

\def\vmu{{\bm{\mu}}}
\def\vtheta{{\bm{\theta}}}
\def\va{{\bm{a}}}
\def\vb{{\bm{b}}}
\def\vc{{\bm{c}}}
\def\vd{{\bm{d}}}
\def\ve{{\bm{e}}}

\def\vs{{\bm{s}}}

\def\vv{{\bm{v}}}
\def\vw{{\bm{w}}}
\def\vx{{\bm{x}}}
\def\vy{{\bm{y}}}
\def\vz{{\bm{z}}}
\def\vphi{{\bm{\phi}}}
\def\vepsilon{{\bm{\epsilon}}}
\def\vmu{{\bm{\mu}}}
\def\vsigma{{\bm{\sigma}}}

\def\mA{{\bm{A}}}

\def\mI{{\bm{I}}}

\def\mM{{\bm{M}}}

\def\mT{{\bm{T}}}

\DeclareMathAlphabet{\mathsfit}{\encodingdefault}{\sfdefault}{m}{sl}
\SetMathAlphabet{\mathsfit}{bold}{\encodingdefault}{\sfdefault}{bx}{n}

\def\gB{{\mathcal{B}}}

\def\gL{{\mathcal{L}}}

\def\gN{{\mathcal{N}}}

\def\gP{{\mathcal{P}}}

\def\sA{{\mathbb{A}}}

\def\sD{{\mathbb{D}}}

\def\sS{{\mathbb{S}}}

\def\cM{{\mathcal{M}}}

\def\cP{{\mathcal{P}}}

\newcommand{\E}{\mathbb{E}}

\newcommand{\R}{\mathbb{R}}

\newcommand{\abs}[1]{\left\vert#1\right\rvert}
\newcommand{\norm}[1]{\left\Vert#1\right\Vert}
\newcommand{\cbr}[1]{\left\{#1\right\}}
\newcommand{\br}[1]{\left(#1\right)}
\newcommand{\sbr}[1]{\left[#1\right]}
\newcommand{\given}{\,|\,}

\usepackage{relsize}

\newcommand{\eg}{\textit{e.g.}}
\newcommand{\wrt}{\textit{w.r.t.}}

\usepackage{natbib}
\setcitestyle{authoryear,round,citesep={;},aysep={,},yysep={;}}

\title{
\bfseries A Behavior Regularized Implicit Policy for Offline \\ Reinforcement Learning
}
\author{
Shentao Yang$^{1,*,\dagger}$, Zhendong Wang$^{1,*}$, Huangjie Zheng$^1$, Yihao Feng$^2$, and Mingyuan Zhou$^{1,\dagger}$\\
\vspace{0mm}\\
$^1$The University of Texas at Austin, 
$^2$Salesforce Research 
}
\date{\vspace{-5ex}}

\begin{document}

\maketitle
\def\thefootnote{*}\footnotetext{Equal contribution.}
\def\thefootnote{$\dagger$}
\footnotetext{Correspondence: Shentao Yang \texttt{shentao.yang@mccombs.utexas.edu}, Mingyuan Zhou \texttt{mingyuan.zhou@mccombs.utexas.edu}.}
\def\thefootnote{\arabic{footnote}}

\begin{abstract}
Offline reinforcement learning enables learning from a fixed dataset, without further interactions with the environment. The lack of environmental interactions makes the policy training vulnerable to state-action pairs far from the training dataset and prone to missing rewarding actions. For training more effective agents, we propose a framework that supports learning a flexible yet well-regularized fully-implicit policy. We further propose a simple modification to the classical policy-matching methods for regularizing with respect to the dual form of the Jensen--Shannon divergence and the integral probability metrics. We theoretically show the correctness of the policy-matching approach, and the correctness and a good finite-sample property of our modification. An effective instantiation of our framework through the GAN structure is provided, together with techniques to explicitly smooth the state-action mapping for robust generalization beyond the static dataset. Extensive experiments and ablation study on the D4RL benchmark validate our framework and the effectiveness of our algorithmic designs.
\end{abstract}

\section{Introduction} \label{sec:intro}

Offline reinforcement learning (offline RL), also known as batch RL, aims at training agents from
fixed datasets that are typically large and heterogeneous, with a special emphasis on no environmental interactions during training \citep{treerl2005, batchrl2012, bcq2019, bear2019, brac2019, rem2020, abm2020, crr2020}.
This paradigm extends the applicability of RL to  where the environmental interactions are costly or even potentially dangerous, such as healthcare \citep{rllungcancer2017, rlhealth2018, whentotreat2019}, autonomous driving \citep{autodrivesurvey2020}, and recommendation systems \citep{operecom2017, offlineabtest2018}. 
While (online) off-policy RL algorithms \citep{ddpg2016,td32018,sac2018}
could be directly adopted into offline settings, their application can be unsuccessful \citep{bcq2019, bear2019}, especially on high-dimensional continuous control tasks, where function approximations are inevitable and data samples are non-exhaustive. 
Such failures may be attributed to the discrepancy between the state-action visitation frequency induced by the current policy and that by the data-collecting behavior policy,
which results in possibly uncontrollable extrapolation errors \citep{bcq2019, bear2019}.
In this regard, one approach to offline RL is to control the difference between the observed and policy-induced visitations, so that the current policy mostly generates state-action pairs that are close to the offline dataset.

Previous work in this line of research typically 
\textbf{(1)} regularizes the current policy to be close to behavior policy during training, \textit{i.e.}, policy (state-conditional action distribution) matching;
\textbf{(2)} uses a Gaussian policy class with a learnable mean and diagonal covariance matrix \citep{bear2019,brac2019}. See Appendix~\ref{sec:related_work} for a detailed review.
However, at any given state $\vs$, the underlying action-value function may possess multiple local maxima over the action space.  
A deterministic or uni-modal stochastic policy may only capture one of the local optima and neglect lots of rewarding actions.
An even worse situation occurs when such stochastic policy exhibits a strong mode-covering behavior, artificially inflating the probability density around the average of multiple rewarding actions that itself may be inferior. 

Previous work under the policy-matching theme mainly takes two approaches.
The first approach, \textit{e.g.}, \citet{bear2019}, resorts to a two-step strategy: 
First, fit a generative model $\widehat \pi_b(\va \given \vs)$ to clone the behavior policy; 
Second, estimate the distance between the fitted behavior policy $\widehat \pi_b(\va \given \vs)$ and the current policy, and minimize that distance as a way to regularize.
While this approach is able to accurately estimate the distance between the current policy and the cloned behavior, its success relies heavily on how well the inferred behavior-cloning generative model mimics the true behavior policy.
On tasks with large or continuous state space or on datasets collected by a mixture of policies, however, accurately estimating the behavior policy is known to be hard \citep{cql2020}. 
In particular, some prior work uses conditional VAE (CVAE, \citet{cvae2015}) 
to clone the possibly-multimodal behavior policy, which further suffers to the problem that CVAE may exhibit a strong mode-covering behavior. 
The second approach in the policy-matching theme directly estimates the divergence between the state-conditional actions distributions \citep{brac2019}.
However, on tasks with continuous state space, with probability one,
for each observed state $\vs_i$, the offline dataset has only one corresponding action $\va_i$ from the behavior policy. 
Thus, unlike the first approach, at each state one is only able to use a single data-point to assess whether the current policy is close to the behavior policy, which may not well reflect the true divergence between the two conditional distributions.

To address these concerns, we are motivated to develop a framework that not only supports an flexible policy,
but also well regularizes this expressive policy towards the data-collecting behavior policy.
Specifically, \textbf{(1)} instead of using the classical deterministic or uni-modal Gaussian policy, we train a fully implicit policy for its flexibility to capture multiple modes in the action-value function;
\textbf{(2)} to avoid the additional difficulty and complexity in modeling the behavior policy, we base our framework on the second approach in the policy-matching theme.
On top of that, we propose a \emph{simple modification} to the estimate of the regularization term for improved matching \textit{w.r.t.} the dual form of the Jensen--Shannon divergence (JSD, \citet{jsd1991}) and the integral probability metrics (IPM, \citet{ipm1997}). 
On the theoretical side, we show in \Secref{sec:theory} the correctness of the policy-matching approach that it matches the undiscounted state-action visitations, from which the offline dataset is sampled.
We also show the correctness and a good finite-sample property of our proposed modification.
Similar notion in offline RL of matching the state-action visitations is taken by the DICE family \citep{algaedice2019, optidice2021}, but they either use a Gaussian policy or a mixture of Gaussian policies with a per-dataset tuned number of mixtures.
Besides, these algorithms have high computational complexity, which, together with inflexible policies and intensive hyperparameter tuning, limit their practical applicability. 

We instantiate our framework with a  generative adversarial network (GAN) \cite{gan2014} based structure that approximately minimizes the JSD between the current and the behavior policies.
Furthermore, we design techniques to explicitly encourage robust behavior of our policy at states not included in the static dataset.
We conduct ablation study on several components of our algorithm and analyze their contributions.
With these considerations, our full algorithm achieves competitive performance on various tasks from the D4RL benchmark \citep{fu2021d4rl}.

\section{Background and Motivation} \label{sec:prelim}

We first present background information and then introduce a toy example to illustrate the motivations of the proposed framework for offline RL.

\subsection{Offline RL} \label{sec:offlinerl_prelim}
Following the classic RL setting \citep{rlintro2018}, the interaction between the agent and environment is modeled as a Markov decision process (MDP), specified by the tuple $\cM = \br{\sS, \sA, \cP, r, \gamma}$, where $\sS$ denotes the state space, $\sA$ the action space,  $\cP(\vs' \given \vs, \va): \sS \times  \sS \times \sA \rightarrow [0,1]$ the environmental dynamics, $r(\vs, \va):\sS \times \sA \rightarrow \sbr{R_{\min}, R_{\max}}$ the reward function, and $\gamma \in (0,1]$ the discount factor. 
The goal of RL is to learn a policy $\pi_\vphi(\va_t \given \vs_t)$, parametrized by $\vphi$, that maximizes the expected cumulative
discounted reward $\E\sbr{\sum_{t=0}^{\infty}\gamma^t r(\vs_t, \va_t)}.$ 

In offline RL \citep{offlinetutorial2020}, the agent only has access to a fixed dataset $\sD \triangleq \cbr{(\vs, \va, r, \vs')}$, consisting of transition tuples from rollouts of some behavior policies $\pi_{b}(\va \given \vs)$ on $\cP$. We denote the undiscounted state-action visitation frequency induced by the behavior policy $\pi_b$ as $d_{b}(\vs,\va)$ and its state-marginal
as $d_{b}(\vs)$. 
The counterparts for the current policy $\pi_\vphi$ are $d_{\vphi}(\vs, \va)$ and $d_{\vphi}(\vs)$.
Here, $d_{b}(\vs,\va) = d_{b}(\vs)\pi_{b}(\va \given \vs)$ and following the literature, \textit{e.g.}, \citet{breakingcurse2018}, we have $\sD \sim d_{b}(\vs,\va)$ (discussed further in Appendix~\ref{sec:related_work}). 
The visitation frequencies in the dataset are denoted as $d_{\sD}(\vs, \va)$ and $d_{\sD}(\vs)$, which are discrete approximations to $d_{b}(\vs,\va)$ and $d_{b}(\vs)$, respectively.

\subsection{Actor-Critic Algorithm} \label{sec:ac_prelim}
Denote the action-value function as $Q^{\pi}(\vs, \va) = \E_{\pi, \gP} [ \sum_{t=0}^{\infty}\gamma^t r(\vs_t, \va_t) \given \vs_0=\vs, \va_0=\va]$.
In the actor-critic scheme \citep{rlintro2018}, the critic $Q^{\pi}(\vs, \va)$ is often approximated by a neural network $Q_{\vtheta}\br{\vs, \va}$, parametrized by $\vtheta$  and trained by the Bellman operator \citep{ddpg2016, sac2018, bcq2019}. 

The actor $\pi_{\vphi}$ aims at maximizing the expected value of $Q_{\vtheta}$, and in offline RL its learning objective is commonly expressed as maximizing \textit{w.r.t.} $\vphi$
\begin{equation}\textstyle \label{eq:ac_actor_approx}
     J\br{\pi_{\vphi}} \approx \E_{\vs \sim d_b(\vs),\, \va \sim \pi_{\vphi}(\cdot \given \vs)}\sbr{Q_{\vtheta}\br{\vs, \va}},
\end{equation}
where sampling from $d_b(\vs)$ can be implemented as sampling from the offline dataset $\sD$ \citep{diagbottleneck2019, offlinetutorial2020}.

\subsection
{Generative Adversarial Nets} \label{sec:gan_prelim}
GAN \citep{gan2014} provides a framework to train deep generative models, with two neural networks  trained jointly in an adversarial manner:
a generator $G_\vphi$, parametrized by $\vphi$, that fits the data distribution and a discriminator $D_\vw$, parametrized by $\vw$, that outputs the probability of a sample coming from the training data rather than $G_\vphi$.
Sampling $\vx$ from the generator's distribution $d_\vphi\br{\vx}$ can be realized with  $\vz \sim p_\vz(\vz), \vx = G_\vphi\br{\vz}$, where $p_\vz(\vz)$ is some noise distribution.
Denote $d_\sD\br{\cdot}$ as the data distribution, 
both $G_\vphi$ and $D_\vw$ are trained via a two-player min-max game as 
\begin{equation}\label{eq:gan_obj}\textstyle 
\begin{split}
    \min_\vphi \max_\vw \mbox{ } V\br{D_\vw,G_\vphi} = \E_{\vy \sim d_\sD\br{\cdot}}\sbr{\log D_\vw \br{\vy}} + \E_{\vz\sim p_\vz \br{\cdot}} \sbr{\log\br{1-D_\vw(G_\vphi(\vz))}}.
\end{split}
\end{equation}
Given the optimal discriminator $D_G^*$ at $G_\vphi$, the training objective of $G_\vphi$
is determined by  the JSD between $d_\sD$ and $d_\vphi$ as
$
    V\br{D_G^*, G_\vphi}
    = -\log 4 + 2 \cdot \mathrm{JSD} \br{d_\sD \Vert d_\vphi}
$,
with the global minimum achieved if and only if $d_\vphi = d_\sD$.
Therefore, one may view GAN as a distributional matching framework that approximately minimizes the JSD between the generator distribution and data distribution.

\subsection{Motivations} \label{sec:motivation}
To illustrate our motivations of training an expressive policy under an appropriate regularization,
we conduct a toy experiment of behavior cloning, as shown in Figure~\ref{fig:toy}, where
we use the $x$- and $y$-axis values to represent the state and action, respectively. 
Figure~\ref{fig:toy_truth} shows the state-action joint distribution of the behavior policy that we try to mimic. 
For Figures~\ref{fig:toy_cvae}-\ref{fig:toy_gan}, we use the same test-time state distribution, consisting of an equal mixture of the behavior policy's state distribution and a uniform state distribution between $-1.5$ and $1.5$. 
If the inferred policy well approaches the behavior policy, we expect \textbf{(1)} clear concentration on the eight centers and \textbf{(2)} smooth interpolation between centers, which implies a good and smooth fit to the behavior policy.
We start with fitting a CVAE model, a representative behavior-cloning method, to the dataset. 
As shown in Figure~\ref{fig:toy_cvae}, CVAE exhibits a mode-covering behavior that covers the data density modes at the expense of overestimating unwanted low data-density regions.
Hence, the regularization ability is questionable of using CVAE as a proxy for the behavior policy in some prior work.
Replacing CVAE with the conditional GAN (CGAN, \citet{cgan2014}), \textit{i.e.}, replacing the KL loss with the JSD loss, but adopting the Gaussian policy popular in prior offline RL work partially alleviates the mode-covering issues but drops necessary modes, as shown in Figure~\ref{fig:toy_gcgan}.
This shows the inflexibility of Gaussian policies.
Replacing the Gaussian policy in CGAN with an implicit policy, and training CGAN via the classical policy-matching approach, improves the capability of capturing multiple modes, as shown in Figure~\ref{fig:toy_cgan}.
Finally, training the implicit-policy CGAN via our proposed modification (Section~\ref{sec:policy_matching_reg}) also leads to good capture of the behavior policy.
As shown in Figure~\ref{fig:toy_gan}, it concentrates clearly on the eight centers and interpolates smoothly between the seen states. 
Based on this toy example, training a fully-implicit policy with the policy-matching strategy and \textit{w.r.t.} the GAN-style JSD minimization can be an effective way to learn a flexible yet well-regularized policy in offline reinforcement learning.

\begin{figure*}[tb]
     \centering
     \begin{subfigure}[b]{0.19\textwidth}
         \centering
         \includegraphics[width=\textwidth]{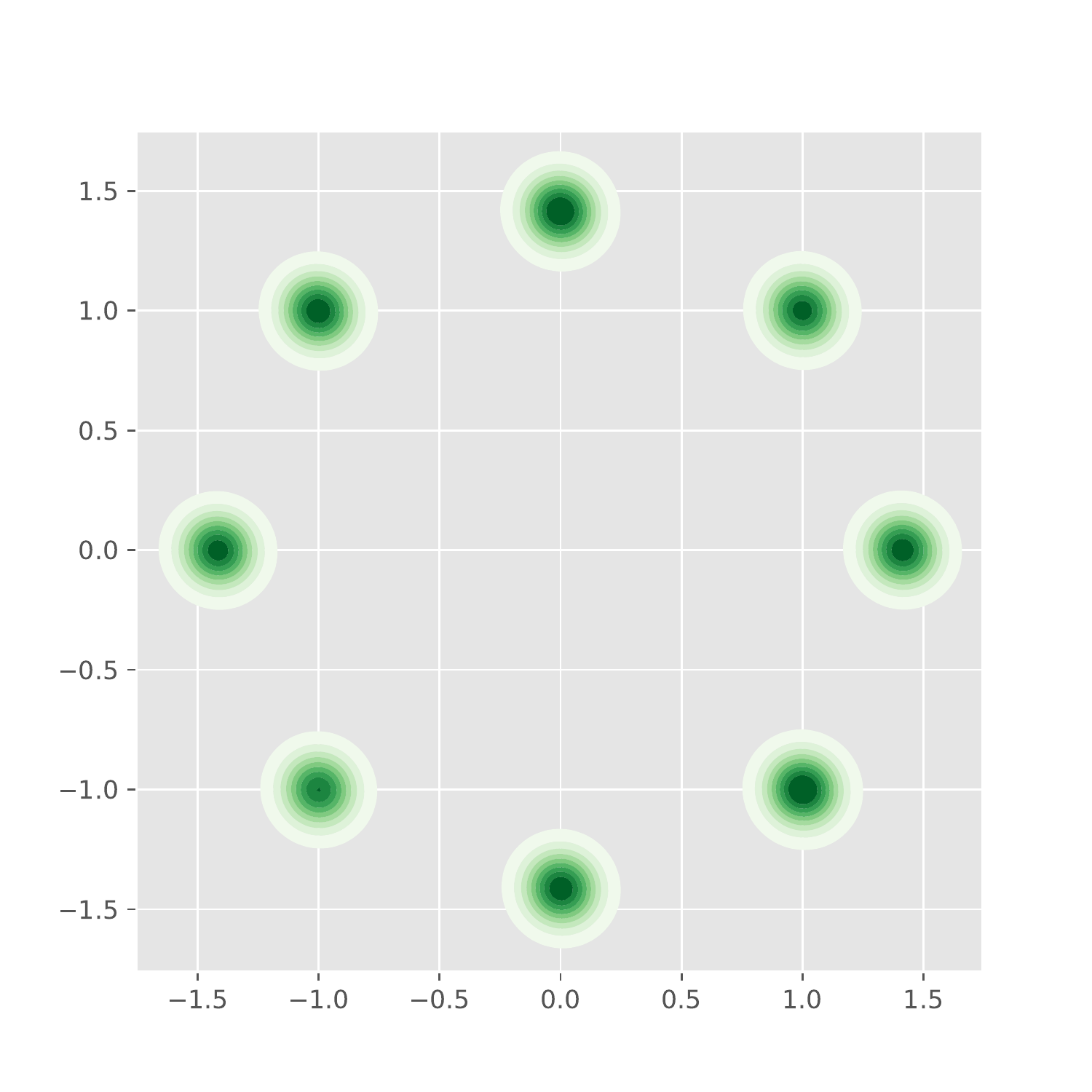}
         \caption{Truth}
         \label{fig:toy_truth}
     \end{subfigure}
     \hfill
     \begin{subfigure}[b]{0.19\textwidth}
         \centering
         \includegraphics[width=\textwidth]{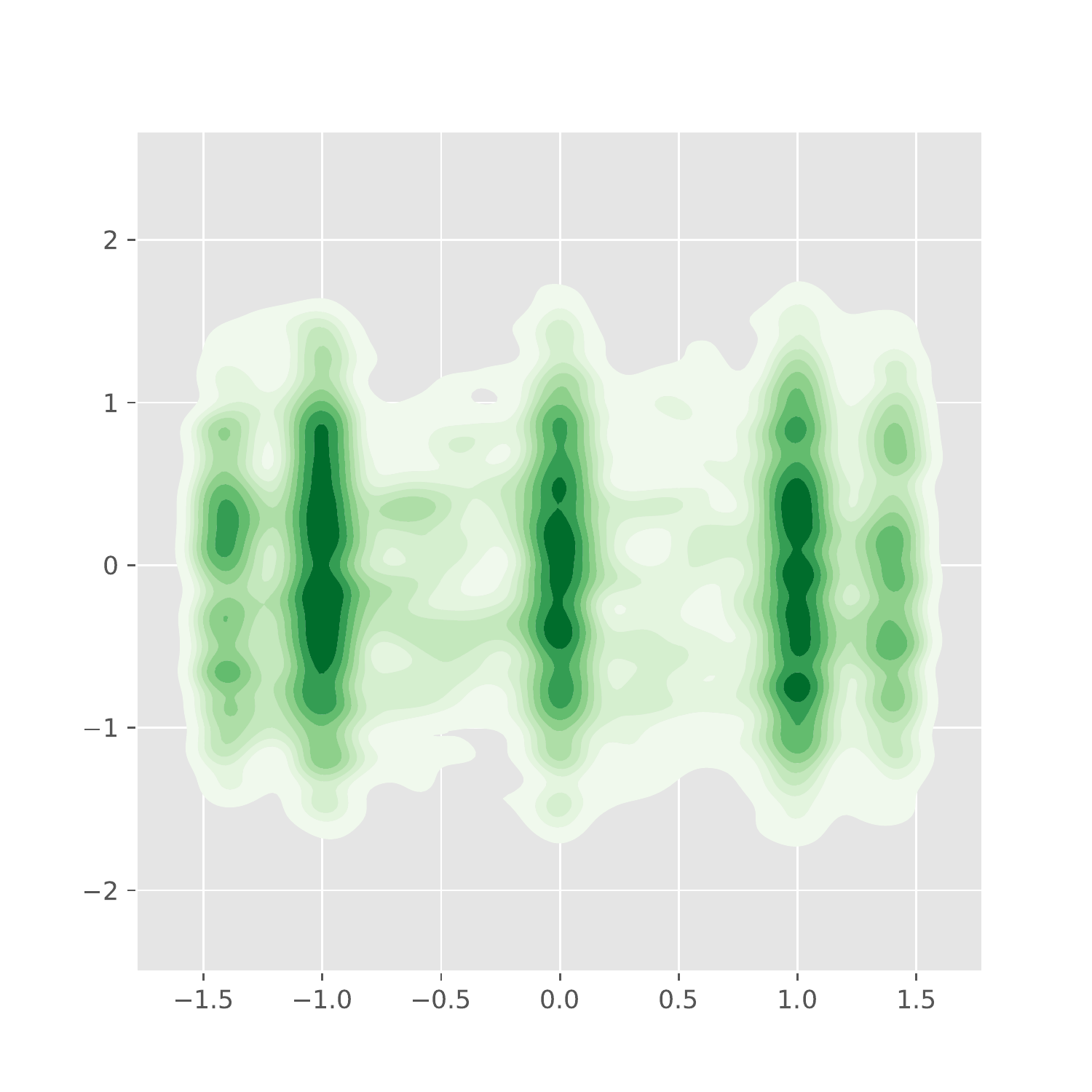}
         \caption{CVAE}
         \label{fig:toy_cvae}
     \end{subfigure}
     \hfill
     \begin{subfigure}[b]{0.19\textwidth}
         \centering
         \includegraphics[width=\textwidth]{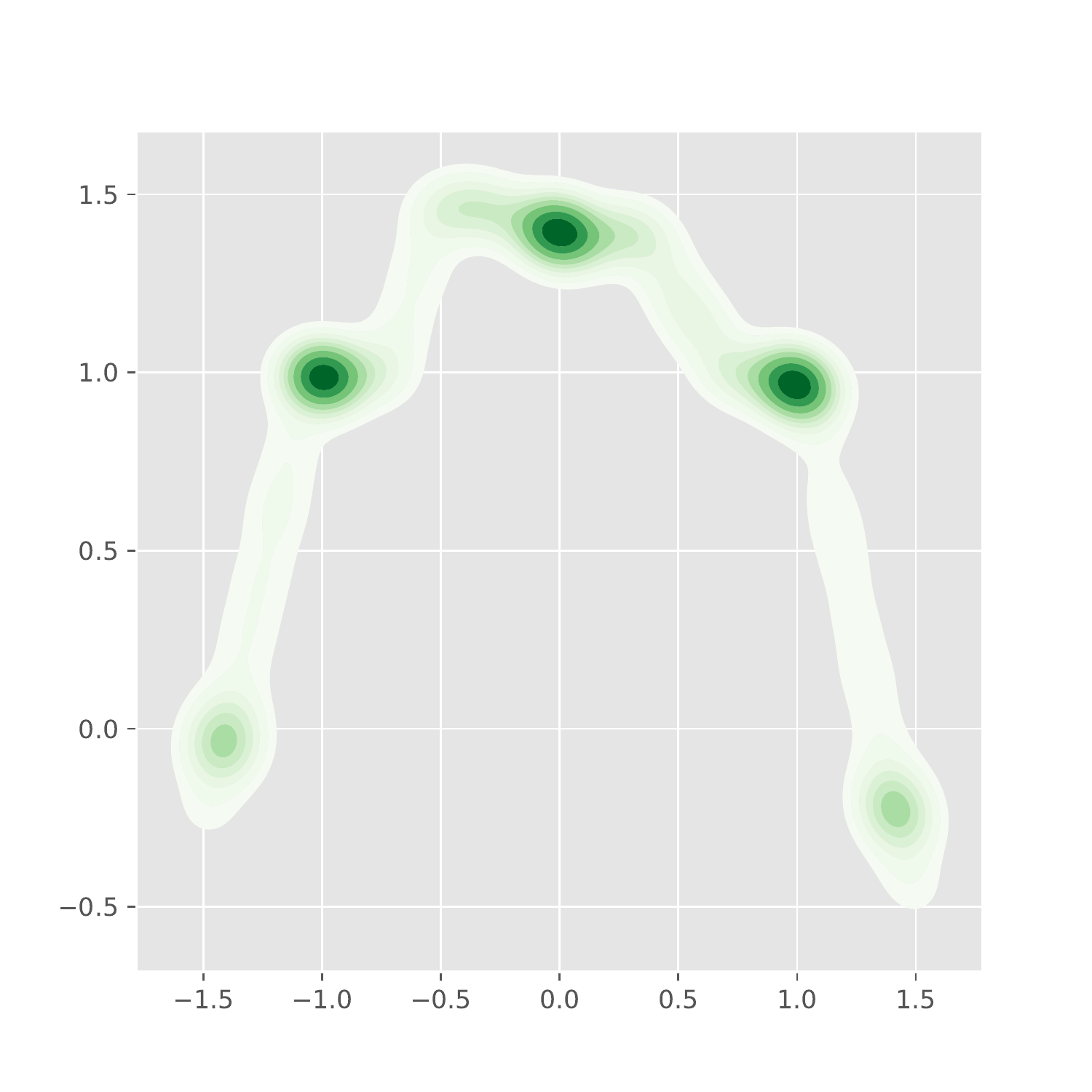}
         \caption{G-CGAN}
         \label{fig:toy_gcgan}
     \end{subfigure}
     \hfill
     \begin{subfigure}[b]{0.19\textwidth}
         \centering
         \includegraphics[width=\textwidth]{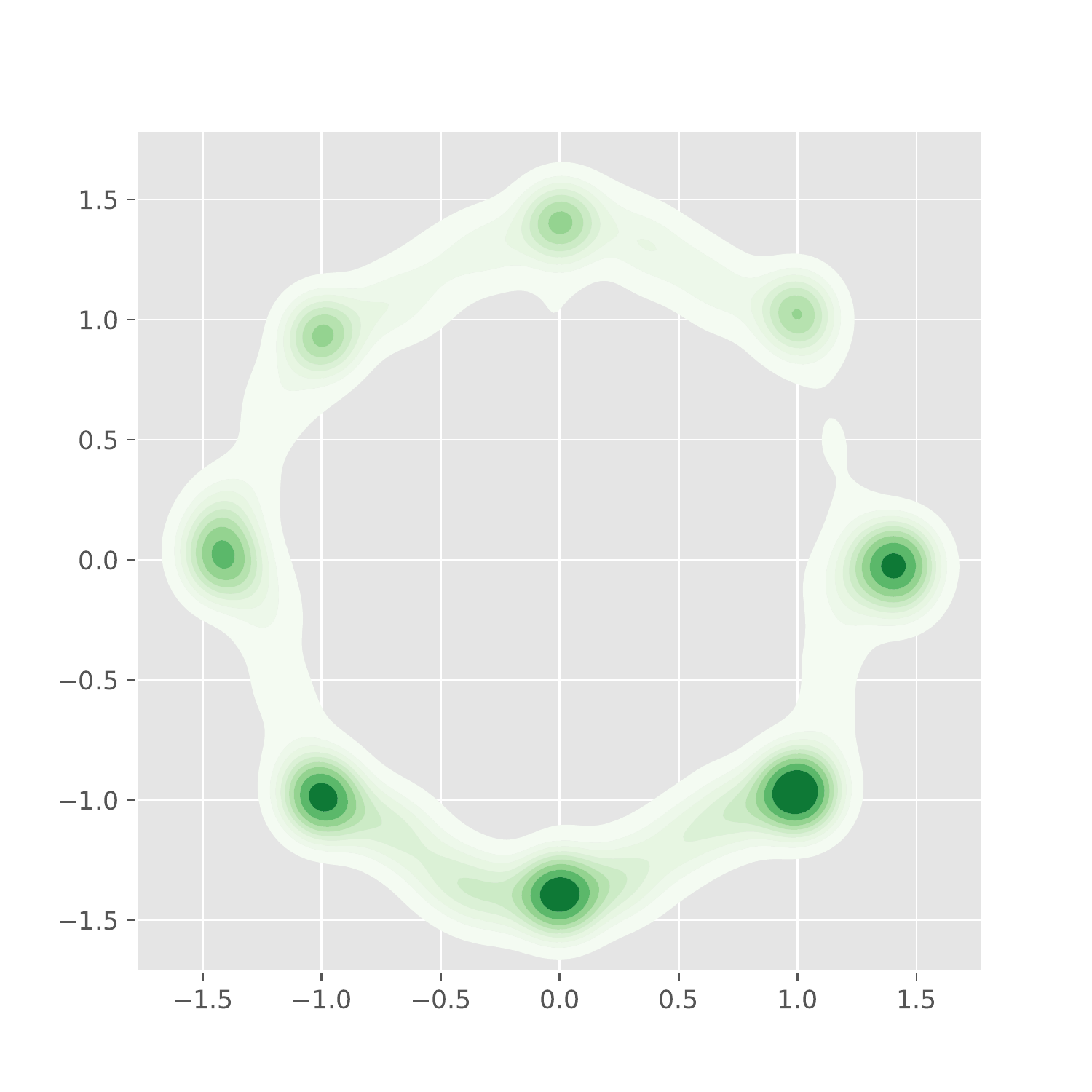}
         \caption{CGAN}
         \label{fig:toy_cgan}
     \end{subfigure}
     \hfill
     \begin{subfigure}[b]{0.19\textwidth}
         \centering
         \includegraphics[width=\textwidth]{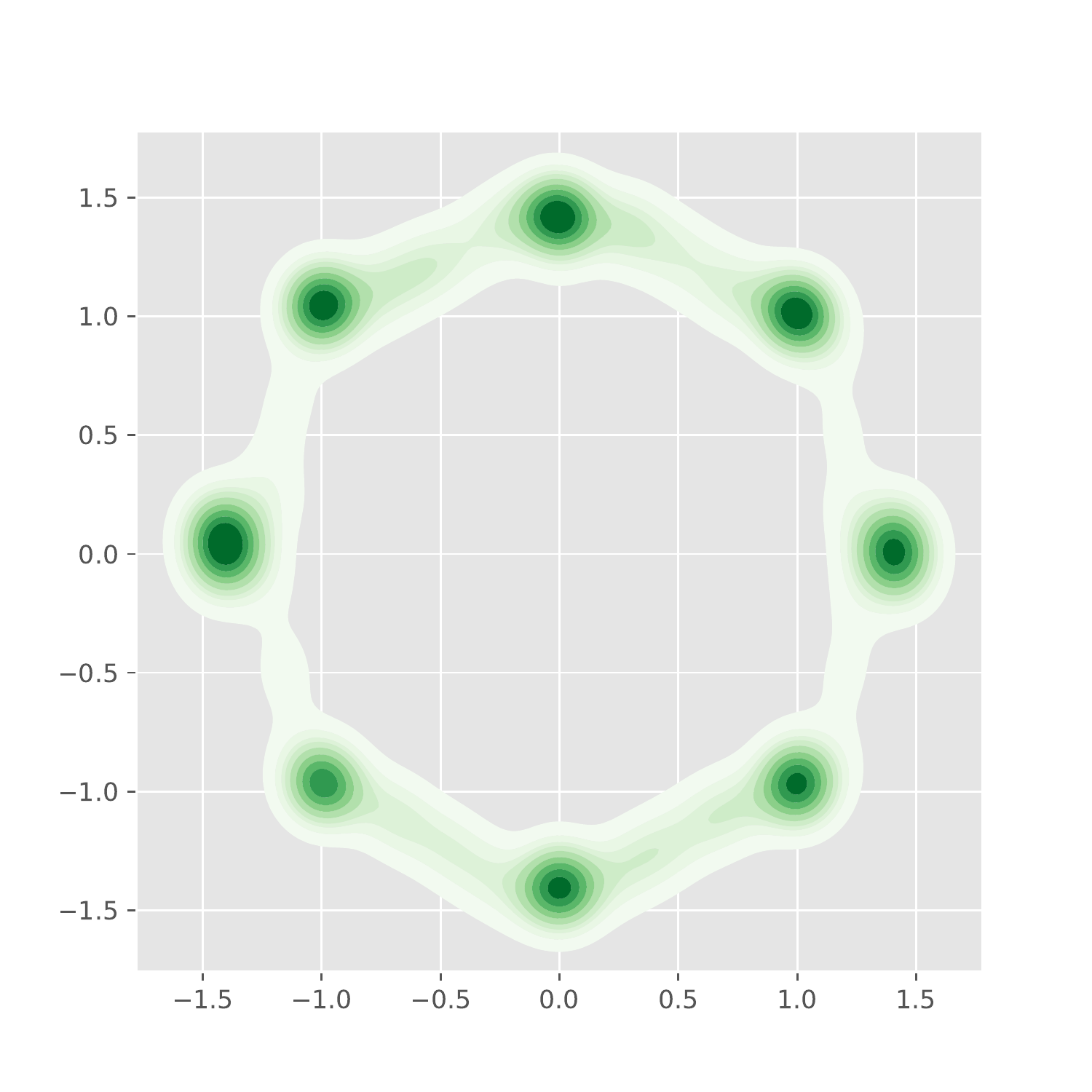}
         \caption{GAN}
         \label{fig:toy_gan}
     \end{subfigure}
    \captionsetup{font=small}
        \caption{ \small 
        Performance of approximating the behavior policy on the eight-Gaussian dataset. 
        A conditional VAE (``CVAE''), a conditional GAN (``CGAN''), and a Gaussian-generator conditional GAN (“G-CGAN”) are fitted using the classical policy matching approach.
        A conditional GAN (``GAN'') is fitted using the basic state-action joint-matching strategy (\Secref{sec:policy_matching_reg}).
        Details are in Appendix \ref{sec:toy_detail}.
        }
        \label{fig:toy}
\end{figure*}

\section{State-Action Joint Regularized Implicit Policy}\label{sec:method}

In this section we discuss an instance of our framework that will be used in our empirical study in \Secref{sec:experiment}.
Concretely, for sample-based policy-matching, we train a fully implicit policy via a GAN structure to approximately minimize the JSD. 
Our basic algorithm is discussed in \Secref{sec:basic}, followed by two enhancing components presented in \Secref{sec:enhance_comp} to build up our full algorithm.
This instantiation manifests three facets we consider important in offline RL:
\textbf{(1)} the flexibility of the policy class,
\textbf{(2)} an effective sample-based regularization without explicitly modelling the behavior policy,
and \textbf{(3)} the smoothness of the learned policy.

\subsection{Basic Algorithm} \label{sec:basic}
Motivated by the standard actor-critic and GAN frameworks,
our basic algorithm consists of a critic $Q_\vtheta$, an actor $\pi_\vphi$, and a discriminator $D_\vw$.
For training stability, we follow the double Q-learning \citep{doubleq2010} 
to train a pair of critics $Q_{\vtheta_1}, Q_{\vtheta_2}$ and maintain the target networks $Q_{\vtheta'_1}, Q_{\vtheta'_2}, \pi_{\vphi'}$.

We follow prior work \citep[\eg,][]{bcq2019,bear2019} to use the critic-training target
\begin{equation} \label{eq:critic_target}
         \widetilde{Q}\br{\vs, \va} \triangleq  r(\vs, \va) + \gamma \E_{\va'\sim \pi_{\vphi'}\br{\boldsymbol{\cdot}\given \vs'}}\sbr{\lambda \min_{j=1,2}Q_{\vtheta'_j}\br{\vs', \va'} + (1-\lambda) \max_{j=1,2}Q_{\vtheta'_j}\br{\vs', \va'}}, %
\end{equation}
with hyperparameter $\lambda \in \sbr{0,1}$. 
Both critic networks are trained to minimize the mean-squared-error between their respective action-value estimates $Q_{\vtheta_j}\br{\vs, \va}$ and $\widetilde{Q}\br{\vs, \va}$.

Actor training has three parts: implicit policy, policy-matching regularization, and conservative target.

\subsubsection{Implicit Policy}
As discussed in Sections \ref{sec:intro} and \ref{sec:motivation}, a deterministic or Gaussian policy may miss important rewarding actions, or even concentrate on inferior ``average actions.'' 
For online off-policy RL, \citet{idac2020} shows the benefit of an implicit distribution mixed Gaussian policy.  
Generalizing this idea to offline RL, we train a fully implicit policy, which transforms a given noise distribution into the state-conditional action distribution via a neural network, in reminiscent of the generator in CGAN. 
Specifically, with a deterministic function $\pi_\vphi$ and some noise distribution $p_\vz(\vz)$, given state~$\vs$, 
\begin{equation} \label{eq:implicit_policy} \textstyle 
    \va \sim \pi_\vphi(\boldsymbol{\cdot}\given \vs) = \pi_\vphi(\vs, \vz), \quad \vz \overset{iid}{\sim} p_\vz(\vz).
\end{equation}
As shown in Figures~\ref{fig:toy_cgan} and \ref{fig:toy_gan}, an implicit policy can be stronger to learn multi-modality, if needed.

\subsubsection{Policy-Matching Regularization} \label{sec:policy_matching_reg}
Our goal is to efficiently match the current policy with the behavior policy \textit{w.r.t.} sample-based estimate of some statistical divergence, such as the JSD or IPM.
For the JSD, empirically studied in Section~\ref{sec:experiment}, as in \citet{brac2019}, the classical policy-matching objective is to minimize 
\begin{equation} \label{eq:cond_gan_objective} \textstyle
    \E_{\vs \sim d_b\br{\vs}}\left[\mathrm{JSD}\br{\pi_b(\va \given \vs}, \pi_\vphi(\va \given \vs)) \right].
\end{equation}
Using the notations in GAN, the generator sample $\vx$ and the data sample $\vy$ for policy-matching are 
\begin{equation} \label{eq:cond_true_fake_sample} \textstyle
    \vy \triangleq (\vs, \va) \sim \sD, \; \vx \triangleq (\vs, \Tilde{\va}), \; \Tilde{\va} \sim \pi_\vphi(\boldsymbol{\cdot}\given \vs),
\end{equation}
where \emph{the same} $\vs$ is used in both $\vx$ and $\vy$.

In this paper, we propose to minimize an equivalent form of Eq.~\eqref{eq:cond_gan_objective} as 
\begin{equation} \label{eq:joint_gan_objective} \textstyle
    \mathrm{JSD}\sbr{\pi_b\br{\va \given \vs} d_b\br{\vs}, \pi_\vphi\br{\va\given \vs} d_b\br{\vs}},
\end{equation}
which we dub as ``state-action joint-matching.''
The intuition for a benefit of this objective is discussed below, and the equivalence between Eqs.~\eqref{eq:joint_gan_objective} and \eqref{eq:cond_gan_objective} together with a theoretical benefit of the objective Eq.~\eqref{eq:joint_gan_objective} is discussed in Theorem~\ref{thm:jsd_lower_bound}.
The generator sample $\vx$ and the data sample $\vy$ are now 
\begin{equation} \label{eq:true_fake_sample} \textstyle
    \vy \triangleq (\vs, \va) \sim \sD; \; \vx \triangleq (\Tilde{\vs}, \Tilde{\va}), \; \Tilde{\vs} \sim \sD, \; \Tilde{\va} \sim \pi_\vphi(\boldsymbol{\cdot}\given \Tilde{\vs}),
\end{equation}
where $\Tilde{\vs}$ \emph{is resampled} and thus \emph{is independent of} $\vs$.

For both policy-matching objectives Eq.~\eqref{eq:cond_gan_objective} and Eq.~\eqref{eq:joint_gan_objective}, we constrain the statistical divergence, named the generator loss $\gL_g(\vphi)$, in the training of actor.
In this instantiation of approximately minimizing JSD via GAN, with the discriminator $D_\vw$, we have $\gL_g(\vphi) \triangleq \E_{\vx} \sbr{\log\br{1-D_\vw(\vx)}}$.

Intuitively, our proposal of minimizing the policy-matching objective Eq.~\eqref{eq:joint_gan_objective}, instead of the classical one Eq.~\eqref{eq:cond_gan_objective}, circumvents the problem of matching each state-conditional action distribution on only one data point. 
The state-action pairs $\br{\vs_i, \va_i}$ in the offline dataset are all viewed as samples from $\pi_b\br{\va \given \vs} d_b\br{\vs}$, instead of each pair being separately viewed as one sample from the state-conditional distribution, \textit{i.e.}, $\va_i \sim \pi_b\br{\cdot \given \vs_i}$.
Besides, the state-action joint-matching objective implicitly encourages the smoothness of the state-action mapping, namely, similar states should have similar actions.
This is because, for example, the discriminator in GAN can easily decide as ``fake'' a generator sample $\vx$ should it has state similar to a data sample but action very differently from.
This smoothness feature helps a reliable generalization of our policy to unseen states.

\subsubsection{ 
Actor-Training Target} 
We follow \citet{bear2019} to train the policy \wrt\mbox{} a conservative estimate of the action-values.
For the ease of optimization, we use the Lagrange form of the constrained optimization problem and penalize the generator loss $\gL_g(\vphi)$ while improving the policy.
Our policy-training target is
\begin{equation} \label{eq:policy_target} \textstyle
    \min_{\vphi} - \E_{\vs \sim \sD} \E_{\va \sim \pi_\vphi(\boldsymbol{\cdot}\given \vs)}\sbr{\min_{j=1,2} Q_{\vtheta_j}(\vs, \va)} + \alpha \cdot \gL_g(\vphi),
\end{equation}
where $\alpha$ is a fixed Lagrange multiplier.
At test time, we follow prior work \citep[\eg,][]{bcq2019,bear2019} to first sample $10$ actions from $\pi_\vphi$ and then execute the action that maximizes $Q_{\vtheta_1}$.

The discriminator is trained to better distinguish generator and data samples. 
It aids the policy-matching through outputting $\gL_g(\vphi)$.
As an example, for approximately minimizing JSD via Eq.~\eqref{eq:joint_gan_objective}, the discriminator outputs the probability that the input, either the $\vx$ or $\vy$ in Eq.~\eqref{eq:true_fake_sample}, comes from $d_b(\vs, \va)$. 
In this case, the discriminator is trained to minimize the error in assigning $\vx$ as ``fake'' and $\vy$ as ``true,''  which is the inner maximization of Eq.~\eqref{eq:gan_obj}.

\subsection{Enhancing Components} \label{sec:enhance_comp}
In this section we present two components to further improve the basic algorithm in Section~\ref{sec:basic}.


\textbf{State-smoothing at Bellman Backup.} 
Due to the stochastic nature of environmental dynamics, multiple next states $\vs'$ are possible after taking action $\va$ at state $\vs$, while the offline dataset $\sD$ only contains one such $\vs'$.
Since the agent is unable to interact with the environment to collect more data in offline RL, local exploration \citep{s4rl2021} in the state-space appears as an effective strategy to regularize the Bellman backup by considering states close to the records in the offline dataset.
We assume that: \textbf{(1)} a small transformation to a state results in states physically plausible in the underlying environment (as in \citet{s4rl2021}); 
\textbf{(2)} when the state space is continuous, the transition kernel $\gP\br{\boldsymbol{\cdot}\given \vs, \va}$ is locally continuous and centered at the recorded $\vs'$ in the dataset. 

With these assumptions, we propose to fit $Q_\vtheta(\vs, \va)$ on the value of a small region around the recorded next state $\vs'$. 
Specifically, with a pre-specified standard deviation $\sigma_B$, we sample around $\vs'$ as $\hat\vs = \vs' + \vepsilon, \vepsilon \sim \gN(\vzero, \sigma_B^2 \mI)$, and modify Eq.~\eqref{eq:critic_target} as 
\begin{equation} \label{eq:critic_smooth_target}
         \widetilde{Q}\br{\vs, \va} \triangleq r(\vs, \va) +
         \gamma \E_{\hat\vs} \E_{\hat\va \sim \pi_{\vphi'}\br{\boldsymbol{\cdot}\given \hat\vs}}\sbr{\lambda \min_{j=1,2}Q_{\vtheta'_j}\br{\hat\vs, \hat\va} + (1-\lambda) \max_{j=1,2}Q_{\vtheta'_j}\br{\hat\vs, \hat\va}},%
\end{equation}
where $N_B$ $\hat\vs$ are sampled to estimate the expectation.
This strategy is equivalent to using a Gaussian distribution centered at $\vs'$ to approximate the otherwise non-smooth $\delta_{\vs'}$ transition kernel manifested in the offline dataset.
Similar technique is also considered as the target policy smoothing regularization in \citet{td32018}, though smoothing therein is applied on the target action.

\textbf{State-smoothing at Policy-matching.} 
In optimizing the policy-matching objective Eq.~\eqref{eq:joint_gan_objective}, we substitute $d_\sD(\vs)$ for $d_b(\vs)$.
However, $d_\sD(\vs)$ is in essence discrete and the idea of smoothing the discrete state-distribution can be applied again to provide a better coverage of the state space. 
This design explicitly encourages a predictable and smooth behavior at states unseen in the offline dataset. 
Specifically, with some pre-specified $\sigma_J^2$, we modify the sampling scheme of $\Tilde{\vs}$ in Eq.~\ref{eq:true_fake_sample} as
\begin{equation} \label{eq:resample_state} \textstyle
    \Tilde{\vs} \sim \sD, \, \vepsilon \sim \gN\br{\vzero, \sigma_J^2\mI}, \, \Tilde{\vs} \leftarrow \Tilde{\vs} + \vepsilon.
\end{equation}
Our strategy is akin to sampling from a kernel density approximation \citep{nonparamstat2006} of $d_b(\vs)$ with data points $\vs \in \sD$ and with radial basis kernel of bandwidth $\sigma_J$.

Algorithm~\ref{alg:simple} shows the main steps of our full algorithm, instantiated by approximately minimizing JSD via GAN, and dubbed as ``GAN-Joint.''
A detailed listing of our algorithm is provided in Appendix \ref{sec:full_algo}. 

\begin{algorithm}[H]
\captionsetup{font=small}
\caption{\small GAN-Joint, Main Steps}
\begin{algorithmic}
\label{alg:simple}
\STATE Initialize policy network $\pi_{\vphi}$, critic network $Q_{\vtheta_1}$ and $Q_{\vtheta_2}$, discriminator network $D_\vw$.
\FOR{each iteration}
\STATE Sample transition mini-batch $\gB = \cbr{(\vs, \va, r, \vs')} \sim \sD$. 
\STATE Train the critics by Eq.~\eqref{eq:critic_smooth_target}, $\forall \, j = 1,2$, $\arg\min_{\vtheta_j}( Q_{\vtheta_j}\br{\vs, \va} - \widetilde{Q}(\vs, \va) )^2 $ over $\br{\vs, \va} \in \gB$.
\STATE Get generator loss $\gL_g$ using $D_\vw$ and  the $\vx$, $\vy$ in Eq.~\eqref{eq:true_fake_sample}, apply state-smoothing in \Secref{sec:enhance_comp}.
\STATE Optimize policy network $\pi_{\vphi}$ by Eq.~\eqref{eq:policy_target}.
\STATE Optimize discriminator $D_\vw$ to maximize $\E_{\vy \sim d_\sD\br{\cdot}}\sbr{\log D_\vw \br{\vy}} + \E_{\vx} \sbr{\log\br{1-D_\vw(\vx)}}$.
\ENDFOR
\end{algorithmic}
\end{algorithm}

\section{Theoretical Analysis}\label{sec:theory}

As discussed in \Secref{sec:prelim}, the offline dataset $\sD$ is typically sampled from the \emph{undiscounted} state-action visitation frequency induced by the behavior policy $\pi_b$.
Recall that in this paper we adopt the common strategy 
of controlling the distance between the behavior policy and the current policy during the training process.
In this section, we first prove that this approach, in essence, controls the corresponding undiscounted state-action visitations.
As a consequence, the issue of uncontrollable extrapolation errors in the action-value function estimate can be mitigated.

\begin{theorem}[Informal]\label{thm:occup_single_informal}
When the current policy is close to the behavior policy, the total-variation distance between the corresponding undiscounted state-action visitation frequencies are small.
\end{theorem}

A formal statement and the proof of Theorem~\ref{thm:occup_single_informal} is on Theorem~\ref{thm:occup_single} provided in Appendix~\ref{sec:proof}. 

We notice that similar analysis has been given in the prior work of bounding $D_{\mathrm{KL}}\br{d_\vphi(\vs) \| d_{b}(\vs)}$ by $O\br{\epsilon / (1-\gamma)^2}$ \citep{trpo2015,offlinetutorial2020}. 
However, that prior work deals with (unnormalized) \emph{discounted} visitation frequencies while our bound is devoted to \emph{undiscounted} visitation frequencies, since neither the data collection (\textit{i.e.}, policy rollout) nor the proposed state-action joint-matching scheme (\Secref{sec:policy_matching_reg}) involve the discount factor.  
In short, the definitions of $d_\vphi(\vs)$ and $d_b(\vs)$ in our work are different from the prior work.
Note that this prior bound depends on $1-\gamma$ in the denominator and hence cannot be applied to the undiscounted case where the discount factor $\gamma = 1$.

In practice, the offline dataset $\sD$ often consists of samples collected by a mixture of policies.
Equivalently, the behavior policy $\pi_b(\boldsymbol{\cdot}\given\vs)$ is a mixture of single policies. 
Theorem~\ref{thm:occup_single_informal} can be extended into the mixture of policies case as in Theorem \ref{thm:occup_mix} provided in Appendix~\ref{sec:proof}.

We now show the correctness of our proposed state-action joint-matching scheme (Eq.~\eqref{eq:true_fake_sample}) in IPM.

\begin{definition}[Integral Probability Metric]\label{def:ipm}
The integral probability metrics (IPM) $D_{\mathcal{G}}$ for the probability measures $\mathcal{P}, \mathcal{Q}$ \textit{w.r.t.} some function class $\mathcal{G}$ is defined as \citep{ipm1997, mmdgangp2018}
    \begin{equation*}
        D_{\mathcal{G}}\br{\mathcal{P}, \mathcal{Q}} = \sup_{g\in \mathcal{G}}\abs{\E_{X\sim \mathcal{P}} \sbr{g(X)} - \E_{Y \sim \mathcal{Q}} \sbr{g(Y)}}.
    \end{equation*}
\end{definition}

\begin{theorem}[Informal] \label{thm:ipm_undiscount_informal}
    Our proposed state-action joint-matching scheme (Eq.~\eqref{eq:true_fake_sample}) and the classical policy-matching scheme (Eq.~\eqref{eq:cond_true_fake_sample}) minimize the IPM between undiscounted state-action visitations.
\end{theorem}

A formal statement and the proof of Theorem~\ref{thm:ipm_undiscount_informal} is on Theorem~\ref{thm:ipm_undiscount} provided in Appendix~\ref{sec:proof}.
Interestingly, from the last two equalities in Theorem~\ref{thm:ipm_undiscount}, only state-samples from the offline dataset are needed to minimize the IPM between undiscounted state-action visitations.

Though out of the scope of this paper, for completeness we note that for the discounted visitation frequency defined for a policy $\pi$ as 
$
    d^\pi\br{\vs, \va} \triangleq (1-\gamma) \sum_{t=0}^\infty \gamma^t \Pr_\pi\br{\vs_t = \vs, \va_t = \va},
$
our scheme also matches the IPM between the discounted visitation frequencies $D_{\mathcal{G}}\br{d_\vphi\br{\vs, \va}, d_b\br{\vs, \va}}$, where we reuse the notations $d_\vphi$ and $d_b$ which denote undiscounted visitation frequencies elsewhere.
This is shown in Theorem~\ref{thm:ipm_discount} provided in Appendix~\ref{sec:proof}.

Note that since both IPM and JSD are valid probability metrics, 
in theory, we consider IPM and JSD as comparable for distribution matching.
Empirical successes of approximate JSD matching via GAN are abundant, however, JSD is hard to analyze in theory \cite{fedus2018many}. 
IPM is much easier to analyze, but requires the discriminator to be within some specific function class, which is hard to enforce in practice \citep{mescheder2018training}.
We thus conduct theoretical analysis under IPM, but adopt GAN for coding. 
Indeed, based on our preliminary study discussed in \Secref{sec:abaltion_study} \textbf{(e)}, the JSD-matching via GAN provides both better results and an easier hyperparameter reference from the literature (discussed in Appendix~\ref{sec:tech_gan}).
We henceforth focus on approximately minimizing JSD via GAN.

At the population level, the objectives for our proposed state-action joint-matching scheme (Eq.~\eqref{eq:true_fake_sample}) and the classical policy-matching scheme (Eq.~\eqref{eq:cond_true_fake_sample}) are the same.
However, in theory the classical policy-matching requires many samples from $\pi_b(\va \given \vs)$ while our proposed scheme only requires many samples from $\sD$, as in the discussed case of JSD.
We now verify the equivalence of Eqs.~\eqref{eq:joint_gan_objective} and \eqref{eq:cond_gan_objective}.
Further, while both are valid methods in theory, our method has better property in practice.

\begin{theorem} \label{thm:jsd_lower_bound}
    \textbf{(1)} $
    \mathrm{JSD}\sbr{\pi_b\br{\va \given \vs} d_b\br{\vs}, \pi_\vphi\br{\va\given \vs} d_b\br{\vs}} = \E_{\vs \sim d_b\br{\vs}}\left[\mathrm{JSD}\br{\pi_b(\va \given \vs}, \pi_\vphi(\va \given \vs)) \right]
    $ . \\
    \textbf{(2)} Under the state-action joint-matching scheme, the discriminator is optimized towards estimating the desired JSD; while under the classical policy-matching scheme, the discriminator is optimized towards estimating a lower bound of the desired JSD.
\end{theorem}
Proof of Theorem~\ref{thm:jsd_lower_bound} is in Appendix~\ref{pf:jsd_lower_bound}. 
In fact, part (2) of Theorem~\ref{thm:jsd_lower_bound} holds not just for JSD, but also for IPM, which we state and prove in Theorem~\ref{thm:fd_lower_bound} in Appendix~\ref{sec:proof}.

\section{Experiments}\label{sec:experiment}
In this section we test an instantiation of our framework on the continuous-control RL tasks.
Specifically, we first show the effectiveness of the implicit policy, the state-action joint-matching scheme, and the state-smoothing techniques (\Secref{sec:main_result}).
We then show in ablation study (\Secref{sec:abaltion_study}) the contributions of several components.
Finally we discuss the complexity of our method (\Secref{sec:compute_complex}).

{\bfseries Instantiation.}
We use GAN to approximately control the JSD between the behavior policy and the current policy. 
We use a simple GAN structure with generator (RL policy) and discriminator having two hidden-layers of sizes 400 and 300,
with the loss and hyperparameter choices following the literature \citep{gan2014, dcgan2016}.
To mimic a hyperparameter-agnostic setting, we minimize hyperparameter tuning across datasets.  
Implementation details and hyperparameters 
is in Appendix \ref{sec:tech_gan}.

\subsection{Main Results}\label{sec:main_result}
To validate the effectiveness of our framework, we test four implementations of the GAN instantiation:
\textbf{(1)} basic algorithm (\Secref{sec:basic}) regularized by the classical policy-matching scheme Eqs.~\eqref{eq:cond_gan_objective} and \eqref{eq:cond_true_fake_sample}  (``GAN-Cond:Basic''),
\textbf{(2)} basic algorithm regularized by the proposed state-action joint-matching  (``GAN-Joint:Basic''),
\textbf{(3)} full algorithm, which adds state-smoothing techniques onto ``GAN-Joint:Basic'' (``GAN-Joint''),
\textbf{(4)} full algorithm, with the same construction of the regularization coefficient as in TD3+BC (\citet{td3bc2021}, ``GAN-Joint-$\alpha$'').
For our ``GAN-Joint-$\alpha$'' variant detailed in Appendix~\ref{sec:tech_alpha}, we 
\emph{unify the hyperparameter setting across all tested datasets}.

We compare our implementations with two policy-matching baselines BEAR \citep{bear2019} and BRAC \citep{brac2019}; and state-of-the-art (SOTA) offline-RL algorithms: CQL \citep{cql2020}, FisherBRC \citep{fisherbrc2021}, TD3+BC, EDAC \citep{edac2021}, and OptiDICE \citep{optidice2021}.
We re-run CQL (details in Appendix~\ref{sec:tech_cql}), FisherBRC, TD3+BC, EDAC, and OptiDICE using the official source codes.
Our evaluation protocol is discussed in Appendix~\ref{sec:rl_details}.
Results for other baselines are from \citet{fu2021d4rl}. 
Table~\ref{table:main} presents the results.

Both versions of our full algorithm, ``GAN-Joint-$\alpha$'' and ``GAN-Joint'' on average outperform the baseline algorithms, and their results are relatively stable across datasets that possess diverse nature.
Our full algorithms especially perform robustly and comparatively-well on the high-dimensional Adroit tasks and the Maze2D tasks that are collected by non-Markovian policies, both of which are traditionally considered as hard in offline RL.
On the MuJoCo domain, our full algorithms show their abilities to learn from datasets collected by a mixture of behavior policies, and from medium-quality examples. 
Further, the comparison with OptiDICE may show an overall benefit of our methods over directly matching stationary state-action distributions via behavior cloning and DICE.
These results support our design of an implicit policy, state-action joint-matching, and explicit state-smoothing. 

Comparing ``GAN-Cond:Basic'' with the baseline algorithms, especially BEAR and BRAC that also use policy-matching regularization but with Gaussian policies, we see that an implicit policy does in general help the performance.
This aligns with our intuition in Sections \ref{sec:intro} and \ref{sec:motivation} of the incapability of the uni-modal Gaussian policy in capturing multiple action-modes.

To verify the gain of our state-action joint-matching scheme over the classical policy-matching, apart from the comparison between our full algorithms with BEAR and BRAC, two classical policy-matching methods, we further compare ``GAN-Joint:Basic'' with ``GAN-Cond:Basic.''
On $11$ out of $16$ datasets, ``GAN-Joint:Basic'' wins ``GAN-Cond:Basic,'' while results on other datasets are close.
This empirical gain may be related to the advantage of the state-action joint-matching scheme, \textit{e.g.}, better finite-sample property (Theorem~\ref{thm:jsd_lower_bound}) and smoothness in the state-action mapping (\Secref{sec:policy_matching_reg}).

Comparing ``GAN-Joint'' with ``GAN-Joint:Basic,'' we see that our state-smoothing techniques in general help the performance.
This gain may be related to a smoother action-choice at states not covered by the offline dataset, and a more regularized Bellman backup. 
Note that the smoothing strength here is  unified across all datasets. 
In fact, Table~\ref{table:sigma_sweep}, when viewed row-wise, shows that the gain of our smoothing techniques could be further boosted if allowing per-dataset tuning.

\begin{table}[tb] 
\captionsetup{font=footnotesize}
\caption{
\footnotesize 
Normalized returns for experiments on the D4RL tasks.
High average score and low average rank are desirable.
Here, ``hcheetah'' denotes ``halfcheetah'', ``med'' denotes ``medium'', ``rep'' denotes ``replay'', ``exp'' denotes ``expert'', ``GAN-Joint:B'' denotes ``GAN-Joint:Basic'', and ``GAN-Cond:B'' denotes ``GAN-Cond:Basic''.
} 
\label{table:main} 
\def\arraystretch{1.2}
\resizebox{\textwidth}{!}{
\centering 
\begin{tabular}{l||ccccccc|cccc}
\toprule
                 Task Name &   BEAR  &   BRAC  &                               CQL &                        FisherBRC &                           TD3+BC &                              EDAC &                          OptiDICE &                GAN-Joint-$\alpha$ &                         GAN-Joint &                   GAN-Joint:B &                    GAN-Cond:B \\
\midrule
maze2d-large &     4.6 &      40.6 &   43.7 $\pm$ {\footnotesize 18.6} &   -2.1 $\pm$ {\footnotesize 0.4} &  84.3 $\pm$ {\footnotesize 18.1} &    -0.1 $\pm$ {\footnotesize 8.5} &  130.7 $\pm$ {\footnotesize 56.1} &  200.5 $\pm$ {\footnotesize 23.6} &   63.5 $\pm$ {\footnotesize 21.2} &   57.2 $\pm$ {\footnotesize 16.5} &   36.9 $\pm$ {\footnotesize 17.9} \\
maze2d-med &    29.0 &      33.8 &    30.7 $\pm$ {\footnotesize 9.8} &   4.6 $\pm$ {\footnotesize 20.4} &  47.2 $\pm$ {\footnotesize 41.5} &   25.7 $\pm$ {\footnotesize 10.7} &  140.8 $\pm$ {\footnotesize 44.0} &   72.8 $\pm$ {\footnotesize 21.8} &   74.3 $\pm$ {\footnotesize 25.5} &    44.6 $\pm$ {\footnotesize 9.1} &   42.6 $\pm$ {\footnotesize 21.4} \\
              maze2d-umaze &     3.4 &     -16.0 &    50.5 $\pm$ {\footnotesize 7.9} &  -2.3 $\pm$ {\footnotesize 17.9} &  -0.5 $\pm$ {\footnotesize 15.6} &    19.8 $\pm$ {\footnotesize 3.1} &  107.6 $\pm$ {\footnotesize 33.1} &   58.8 $\pm$ {\footnotesize 22.7} &   47.1 $\pm$ {\footnotesize 18.8} &   50.8 $\pm$ {\footnotesize 15.1} &   56.6 $\pm$ {\footnotesize 22.2} \\
        hcheetah-med &    41.7 &      46.3 &    39.0 $\pm$ {\footnotesize 0.8} &   41.1 $\pm$ {\footnotesize 0.6} &   42.8 $\pm$ {\footnotesize 0.2} &    50.6 $\pm$ {\footnotesize 1.3} &    38.2 $\pm$ {\footnotesize 0.5} &    44.0 $\pm$ {\footnotesize 0.2} &    44.0 $\pm$ {\footnotesize 0.2} &    43.8 $\pm$ {\footnotesize 0.4} &    43.7 $\pm$ {\footnotesize 0.4} \\
           walker2d-med &    59.1 &      81.1 &   60.2 $\pm$ {\footnotesize 30.8} &   78.4 $\pm$ {\footnotesize 1.8} &   78.8 $\pm$ {\footnotesize 3.2} &    84.0 $\pm$ {\footnotesize 1.3} &   14.3 $\pm$ {\footnotesize 15.0} &    69.9 $\pm$ {\footnotesize 6.4} &    69.3 $\pm$ {\footnotesize 8.8} &    66.8 $\pm$ {\footnotesize 4.9} &    66.8 $\pm$ {\footnotesize 7.4} \\
             hopper-med &    52.1 &      31.1 &   34.5 $\pm$ {\footnotesize 11.7} &   99.2 $\pm$ {\footnotesize 0.3} &   99.6 $\pm$ {\footnotesize 0.7} &    29.7 $\pm$ {\footnotesize 0.1} &   92.3 $\pm$ {\footnotesize 16.9} &   86.4 $\pm$ {\footnotesize 10.9} &   66.1 $\pm$ {\footnotesize 24.0} &   69.1 $\pm$ {\footnotesize 20.7} &   67.5 $\pm$ {\footnotesize 21.3} \\
 hcheetah-med-rep &    38.6 &      47.7 &    43.4 $\pm$ {\footnotesize 0.8} &   43.2 $\pm$ {\footnotesize 1.3} &   42.8 $\pm$ {\footnotesize 1.3} &    50.6 $\pm$ {\footnotesize 0.6} &    39.8 $\pm$ {\footnotesize 0.8} &    33.4 $\pm$ {\footnotesize 2.4} &    33.0 $\pm$ {\footnotesize 1.8} &    31.3 $\pm$ {\footnotesize 2.9} &    32.3 $\pm$ {\footnotesize 2.2} \\
    walker2d-med-rep &    19.2 &       0.9 &    16.4 $\pm$ {\footnotesize 6.6} &  38.4 $\pm$ {\footnotesize 16.6} &   22.5 $\pm$ {\footnotesize 5.3} &    15.2 $\pm$ {\footnotesize 2.1} &    20.2 $\pm$ {\footnotesize 5.8} &     6.7 $\pm$ {\footnotesize 2.2} &     9.3 $\pm$ {\footnotesize 2.0} &    10.1 $\pm$ {\footnotesize 1.9} &     7.8 $\pm$ {\footnotesize 3.2} \\
      hopper-med-rep &    33.7 &       0.6 &    29.5 $\pm$ {\footnotesize 2.3} &   33.4 $\pm$ {\footnotesize 2.8} &   31.3 $\pm$ {\footnotesize 3.1} &    27.1 $\pm$ {\footnotesize 0.2} &    29.0 $\pm$ {\footnotesize 4.9} &    30.9 $\pm$ {\footnotesize 3.2} &    30.0 $\pm$ {\footnotesize 2.9} &    33.6 $\pm$ {\footnotesize 7.9} &    26.7 $\pm$ {\footnotesize 1.7} \\
 hcheetah-med-exp &    53.4 &      41.9 &   34.5 $\pm$ {\footnotesize 15.8} &   92.5 $\pm$ {\footnotesize 8.5} &   87.5 $\pm$ {\footnotesize 7.8} &   31.9 $\pm$ {\footnotesize 13.0} &   91.2 $\pm$ {\footnotesize 16.6} &   72.6 $\pm$ {\footnotesize 11.1} &   72.8 $\pm$ {\footnotesize 11.2} &   70.5 $\pm$ {\footnotesize 11.1} &   72.8 $\pm$ {\footnotesize 10.4} \\
    walker2d-med-exp &    40.1 &      81.6 &   79.8 $\pm$ {\footnotesize 22.7} &  98.2 $\pm$ {\footnotesize 13.1} &  94.1 $\pm$ {\footnotesize 18.8} &   98.3 $\pm$ {\footnotesize 26.2} &   67.1 $\pm$ {\footnotesize 30.2} &    79.6 $\pm$ {\footnotesize 1.9} &   75.3 $\pm$ {\footnotesize 12.1} &   67.4 $\pm$ {\footnotesize 13.5} &   59.9 $\pm$ {\footnotesize 16.5} \\
      hopper-med-exp &    96.3 &       0.8 &  103.5 $\pm$ {\footnotesize 20.2} &  112.3 $\pm$ {\footnotesize 0.3} &  112.0 $\pm$ {\footnotesize 0.3} &   111.5 $\pm$ {\footnotesize 0.3} &  101.8 $\pm$ {\footnotesize 18.5} &   71.1 $\pm$ {\footnotesize 10.7} &   86.4 $\pm$ {\footnotesize 19.0} &   76.3 $\pm$ {\footnotesize 21.3} &   68.5 $\pm$ {\footnotesize 22.1} \\
                 pen-human &    -1.0 &       0.6 &    2.1 $\pm$ {\footnotesize 13.7} &    0.0 $\pm$ {\footnotesize 3.9} &   -3.8 $\pm$ {\footnotesize 0.6} &   17.8 $\pm$ {\footnotesize 30.2} &    -0.1 $\pm$ {\footnotesize 5.6} &   71.0 $\pm$ {\footnotesize 23.2} &   57.5 $\pm$ {\footnotesize 22.6} &   61.0 $\pm$ {\footnotesize 16.6} &   52.9 $\pm$ {\footnotesize 16.5} \\
                pen-cloned &    26.5 &      -2.5 &     1.5 $\pm$ {\footnotesize 6.2} &   -2.0 $\pm$ {\footnotesize 0.8} &   -3.5 $\pm$ {\footnotesize 0.5} &   47.1 $\pm$ {\footnotesize 21.4} &     1.4 $\pm$ {\footnotesize 6.8} &    27.6 $\pm$ {\footnotesize 7.1} &   23.2 $\pm$ {\footnotesize 14.2} &   23.6 $\pm$ {\footnotesize 16.7} &   22.0 $\pm$ {\footnotesize 17.6} \\
                pen-exp &   105.9 &      -3.0 &   95.9 $\pm$ {\footnotesize 18.1} &  31.6 $\pm$ {\footnotesize 24.4} &  22.4 $\pm$ {\footnotesize 16.9} &  103.0 $\pm$ {\footnotesize 16.5} &    -1.1 $\pm$ {\footnotesize 4.7} &  134.5 $\pm$ {\footnotesize 10.8} &  140.2 $\pm$ {\footnotesize 12.9} &  131.1 $\pm$ {\footnotesize 13.2} &  126.8 $\pm$ {\footnotesize 14.1} \\
               door-exp &   103.4 &      -0.3 &   87.9 $\pm$ {\footnotesize 21.6} &  57.6 $\pm$ {\footnotesize 37.7} &   -0.3 $\pm$ {\footnotesize 0.0} &   86.0 $\pm$ {\footnotesize 14.9} &   87.9 $\pm$ {\footnotesize 25.8} &   102.2 $\pm$ {\footnotesize 4.5} &   103.5 $\pm$ {\footnotesize 0.9} &   103.0 $\pm$ {\footnotesize 3.4} &   101.8 $\pm$ {\footnotesize 5.1} \\
               \midrule
             Average Score &    44.1 &      24.1 &                              47.1 &                             45.3 &                             47.3 &                              49.9 &                              60.1 &                              \textbf{72.6} &                              \textbf{62.2} &                              58.8 &                              55.4 \\
              Average Rank &     6.8 &       7.8 &                               6.7 &                              5.8 &                              5.4 &                               5.7 &                               5.9 &                               \textbf{4.4} &                               \textbf{4.9} &                               5.4 &                               6.8 \\
\bottomrule
\end{tabular}
}
\end{table}

\subsection{Ablation Study}\label{sec:abaltion_study}  

The ablation study serves to understand the contributions of several algorithmic designs.
Unless stated otherwise, hyperparameters for all algorithmic variants on all datasets are in Table~\ref{table:gan_param}.\\ 

{\bfseries (a):} \textit{Is implicit policy better than the Gaussian policy under our state-action joint-matching scheme?} \\

Table \ref{table:implicit_normal} compares the results of our basic joint-matching algorithm, ``GAN-Joint:Basic,'' with its counterpart where the implicit policy therein is replaced by a Gaussian policy. 
To make a fair comparison, the experimental settings remain the same. Technical details are on Appendix~\ref{sec:tech_gaussian_policy}.  

On $11$ out of $16$ datasets, our basic joint-matching algorithm has higher average return than the Gaussian policy variant. 
This empirical result coincides with our intuition in \Secref{sec:basic} and results in \Secref{sec:main_result} that a Gaussian policy is less flexible to capture all the rewarding actions, of which an implicit policy is likely to be capable.
Appendix~\ref{sec:compare_implicit_gauss} further discusses this comparison and shows in plots that \emph{a Gaussian policy does leave out action modes} in the ``maze2d-umaze'' dataset. \\

{\bfseries (b):} \textit{Does state-smoothing at policy-matching help?}\\

Table \ref{table:smoothing_joint_matching} compares our two full algorithms with their variants of no state-smoothing in the state-action joint-matching scheme.
Our full algorithms overall perform better than the no state-smoothing variants.
The gain may be related to a better coverage of the state-space by the smoothed state-distribution (\Secref{sec:enhance_comp}), which can lead to a more robust action choice at unseen states. \\

{\bfseries (c):} \textit{Does state-smoothing at Bellman backup matter?}\\

Table~\ref{table:smoothing_bellman} compares our two full algorithms with their variants of no state-smoothing in Bellman backup.
Again, overall, our full algorithms perform better than the no state-smoothing versions, showing the benefit of smoothing the empirical transition kernel $\delta_{\vs'}$ (\Secref{sec:enhance_comp}), \textit{e.g.}, taking the stochasticity of state-transitions into account.
In this and the above ablation {\bfseries (b)}, we use the same smoothing strength across all datasets, while a per-dataset tuning may sharpen the comparisons.\\

{\bfseries (d):} \textit{How important is the standard deviation of the Gaussian noise injected in state-smoothing? }\\

To ease hyperparameter tuning, \emph{in practice we fix $\sigma_B = \sigma_J \triangleq \sigma$} (see Appendix~\ref{sec:tech_gan}).
Table~\ref{table:sigma_sweep} tests the robustness of our full algorithm ``GAN-Joint'' to the $\sigma$ hyperparameter, where 
$\sigma$ sweeps over $\cbr{1 \times 10^{-2}, 3 \times 10^{-3}, 1 \times 10^{-3}, 3 \times 10^{-4}, 1 \times 10^{-4}, 0}$.
We see that our method is relatively insensitive to the choice of $\sigma$, especially in the range $\sigma \in \left[1 \times 10^{-4}, 1 \times 10^{-3} \right)$, where the overall performance varies little with $\sigma$.
A too-small $\sigma$ cannot provide enough smoothing to the state distributions while a too-large $\sigma$ may highly distort the information contained in the offline dataset, such as the state-transition kernel.
In both cases, a degradation in the overall performance is expected. \\

{\bfseries (e):} \textit{Does approximately matching the JSD empirically perform better than matching the IPM?}\\

In our preliminary study, we try a variant of our ``GAN-Joint'' that approximately minimizes the dual form of the Wasserstein-$1$ distance, an instance of IPM, by changing only the GAN structure therein into the WGAN-GP (\citet{wgangp2017}, dubbed as ``W1-Joint'').
Table~\ref{table:results_medium_expert} compares our ``GAN-Joint'' with ``W1-Joint'' under varying Lipschitz-1 constraint $\lambda_{\mathrm{GP}}$ on four MoJoCo datasets.
Though ``W1-Joint'' does not fail on these datasets, its results are mediocre, likely because we have not found for it suitable unified hyperparameter and network structure.
We leave further investigation on W1-Joint and other instances of IPM, \textit{e.g.}, the Maximum Mean Discrepancy \citep{kerneltwosampletest2012}, as future work.

\vspace{-1.5em}
\begin{multicols}{2}
  \begin{minipage}{0.49\textwidth}
    \begin{table}[H]
\captionsetup{font=scriptsize}
\caption{\scriptsize 
Preliminary results on four MuJoCo datasets over $3$ random seeds.
``GAN'' denotes ``GAN-Joint'', ``W1'' denotes ``W1-Joint''.
}
\label{table:results_medium_expert}
\centering 
\def\arraystretch{1.}
\resizebox{\textwidth}{!}{
\begin{tabular}{lcccc}
\toprule
                 Task Name &  GAN &  W1 ($\lambda_{\mathrm{GP}}$=$0.1$) &  W1 ($\lambda_{\mathrm{GP}}$=$1$) &  W1 ($\lambda_{\mathrm{GP}}$=$10$) \\
\midrule
 halfcheetah-med-exp &       75.8 &                                  32.2 &                                26.8 &                                 30.7 \\
    walker2d-med-exp &       71.2 &                                  65.2 &                                53.0 &                                 19.7 \\
      hopper-med-exp &       99.9 &                                  22.7 &                                40.4 &                                 25.1 \\
        halfcheetah-med &       44.1 &                                  43.0 &                                45.1 &                                 42.5 \\
\bottomrule
\end{tabular}
}
\end{table}
  \end{minipage}
  
\begin{minipage}{0.49\textwidth}
\begin{table}[H]
\captionsetup{font=scriptsize}
\caption{\scriptsize (Approximate) GPU memory and total training time of our ``GAN-Joint-$\alpha$'' and some baselines on D4RL MuJoCo datasets.}
\label{table:complexity}
\centering 
\def\arraystretch{1.}
\resizebox{\textwidth}{!}{
\begin{tabular}{lccccc}
\toprule
                       & {\bfseries Our}  & CQL & FisherBRC & EDAC  & OptiDICE \\ \midrule
Mem (GB)     & 1.4 & 1.5 & 1.6   & 1.4     &  2 \\
Time (Hour) & 9  & 13  & 8    & 16      &   9 \\ \bottomrule
\end{tabular}
}
\end{table}
\end{minipage}
\end{multicols}

\subsection{Complexity of the Purposed Method} \label{sec:compute_complex}

Table~\ref{table:complexity} compares the computational complexity of our ``GAN-Joint-$\alpha$''  with some baselines.
Note that we use a small GAN structure with discriminator having two hidden-layers of sizes $(400, 300)$, which only adds a small overhead to the vanilla actor-critic algorithm.
We note that CQL and FisherBRC use larger network sizes for actor and critic.
CQL and EDAC require more training steps. 
FisherBRC and OptiDICE need a cloned behavior policy, and OptiDICE uses Gaussian mixture policy with several mixture components for behavior cloning.

As shown on Table~\ref{table:gan_param}, our ``GAN-Joint-$\alpha$'' adds only two more hyperparameters to the classical policy-matching methods, {\itshape i.e.}, $\sigma$ and $N_B$.
Note that ``GAN-Joint-$\alpha$'' achieves good results despite fixing all hyperparameters across all tested datasets, \textit{i.e.}, no per-dataset tuning.
Hence, this default setting can serve as a good starting point for new datasets.

\section{Conclusion} \label{sec:conclusion}
In this paper, we develop a framework that supports learning a flexible yet well-regularized policy in offline RL.
Specifically, we train a fully-implicit policy via regularizing 
the difference between the current policy and the behavior policy during the training process. 
An effective instantiation of our framework through the GAN structure is provided for approximately minimizing the JSD between the current and the behavior policies.
Other divergence metrics, such as the IPM, may also be applied and are left for future work.
We further propose a simple modification to the classical policy-matching scheme for a better regularizing \textit{w.r.t.} the dual form of JSD and IPM.
Moreover, we augment our policy-matching method with explicit state-smoothing techniques to enhance its generalizability on states beyond the dataset.
On the theoretical side, we show the correctness of the policy-matching scheme in matching the underlying undiscounted state-action visitations, and the correctness and a good finite-sample property of our proposed modification.
We validate the efficacy of our framework and implementations through experiments and ablation study on the D4RL benchmark.


\bibliographystyle{icml2022}
\bibliography{joint_matching_icml2022}

\begin{thebibliography}{80}
\providecommand{\natexlab}[1]{#1}
\providecommand{\url}[1]{\texttt{#1}}
\expandafter\ifx\csname urlstyle\endcsname\relax
  \providecommand{\doi}[1]{doi: #1}\else
  \providecommand{\doi}{doi: \begingroup \urlstyle{rm}\Url}\fi

\bibitem[Agarwal et~al.(2020)Agarwal, Schuurmans, and Norouzi]{rem2020}
Agarwal, R., Schuurmans, D., and Norouzi, M.
\newblock An optimistic perspective on offline reinforcement learning.
\newblock In \emph{International Conference on Machine Learning}, pp.\
  104--114. PMLR, 2020.

\bibitem[An et~al.(2021)An, Moon, Kim, and Song]{edac2021}
An, G., Moon, S., Kim, J.-H., and Song, H.~O.
\newblock {Uncertainty-based offline reinforcement learning with diversified
  q-ensemble}.
\newblock \emph{{Advances in Neural Information Processing Systems}}, 34, 2021.

\bibitem[Arjovsky et~al.(2017)Arjovsky, Chintala, and Bottou]{wgan2017}
Arjovsky, M., Chintala, S., and Bottou, L.
\newblock {Wasserstein GAN}.
\newblock \emph{ArXiv}, abs/1701.07875, 2017.

\bibitem[Baxendale(2011)]{harris2011}
Baxendale, P.
\newblock {T. E. Harris's Contributions to Recurrent Markov Processes and
  Stochastic Flows}.
\newblock \emph{The Annals of Probability}, 39\penalty0 (2):\penalty0 417--428,
  2011.
\newblock ISSN 00911798.

\bibitem[Bellemare et~al.(2017)Bellemare, Danihelka, Dabney, Mohamed,
  Lakshminarayanan, Hoyer, and Munos]{cramerdist2017}
Bellemare, M.~G., Danihelka, I., Dabney, W., Mohamed, S., Lakshminarayanan, B.,
  Hoyer, S., and Munos, R.
\newblock {The Cramer Distance as a Solution to Biased Wasserstein Gradients}.
\newblock \emph{ArXiv}, abs/1705.10743, 2017.

\bibitem[Binkowski et~al.(2018)Binkowski, Sutherland, Arbel, and
  Gretton]{mmdgangp2018}
Binkowski, M., Sutherland, D.~J., Arbel, M., and Gretton, A.
\newblock {Demystifying MMD GANs}.
\newblock \emph{ArXiv}, abs/1801.01401, 2018.

\bibitem[Cang et~al.(2021)Cang, Rajeswaran, Abbeel, and Laskin]{mabe2021}
Cang, C., Rajeswaran, A., Abbeel, P., and Laskin, M.
\newblock {Behavioral Priors and Dynamics Models: Improving Performance and
  Domain Transfer in Offline RL}.
\newblock \emph{ArXiv}, abs/2106.09119, 2021.

\bibitem[Chen et~al.(2021)Chen, Lu, Rajeswaran, Lee, Grover, Laskin, Abbeel,
  Srinivas, and Mordatch]{decisiontrans2021}
Chen, L., Lu, K., Rajeswaran, A., Lee, K., Grover, A., Laskin, M., Abbeel, P.,
  Srinivas, A., and Mordatch, I.
\newblock {Decision Transformer: Reinforcement Learning via Sequence Modeling}.
\newblock \emph{ArXiv}, abs/2106.01345, 2021.

\bibitem[Ernst et~al.(2005)Ernst, Geurts, and Wehenkel]{treerl2005}
Ernst, D., Geurts, P., and Wehenkel, L.
\newblock {Tree-Based Batch Mode Reinforcement Learning.}
\newblock \emph{{J. Mach. Learn. Res.}}, 6:\penalty0 503--556, 2005.

\bibitem[Fedus et~al.(2018)Fedus, Rosca, Lakshminarayanan, Dai, Mohamed, and
  Goodfellow]{fedus2018many}
Fedus, W., Rosca, M., Lakshminarayanan, B., Dai, A.~M., Mohamed, S., and
  Goodfellow, I.
\newblock Many paths to equilibrium: {GAN}s do not need to decrease a
  divergence at every step.
\newblock In \emph{International Conference on Learning Representations}, 2018.

\bibitem[Fu et~al.(2019)Fu, Kumar, Soh, and Levine]{diagbottleneck2019}
Fu, J., Kumar, A., Soh, M., and Levine, S.
\newblock {Diagnosing Bottlenecks in Deep Q-learning Algorithms}.
\newblock In \emph{{International Conference on Machine Learning}}, 2019.

\bibitem[Fu et~al.(2020)Fu, Kumar, Nachum, Tucker, and Levine]{fu2021d4rl}
Fu, J., Kumar, A., Nachum, O., Tucker, G., and Levine, S.
\newblock D4rl: Datasets for deep data-driven reinforcement learning.
\newblock \emph{arXiv preprint arXiv:2004.07219}, 2020.

\bibitem[Fujimoto \& Gu(2021)Fujimoto and Gu]{td3bc2021}
Fujimoto, S. and Gu, S.~S.
\newblock {A Minimalist Approach to Offline Reinforcement Learning}.
\newblock \emph{{ArXiv}}, abs/2106.06860, 2021.

\bibitem[Fujimoto et~al.(2018)Fujimoto, van Hoof, and Meger]{td32018}
Fujimoto, S., van Hoof, H., and Meger, D.
\newblock {Addressing Function Approximation Error in Actor-Critic Methods}.
\newblock In Dy, J. and Krause, A. (eds.), \emph{{Proceedings of the 35th
  International Conference on Machine Learning}}, volume~80 of
  \emph{Proceedings of Machine Learning Research}, pp.\  1587--1596. PMLR,
  10--15 Jul 2018.

\bibitem[Fujimoto et~al.(2019)Fujimoto, Meger, and Precup]{bcq2019}
Fujimoto, S., Meger, D., and Precup, D.
\newblock {Off-Policy Deep Reinforcement Learning without Exploration}.
\newblock In Chaudhuri, K. and Salakhutdinov, R. (eds.), \emph{{Proceedings of
  the 36th International Conference on Machine Learning}}, volume~97 of
  \emph{Proceedings of Machine Learning Research}, pp.\  2052--2062. PMLR,
  09--15 Jun 2019.

\bibitem[Gilotte et~al.(2018)Gilotte, Calauz{\`e}nes, Nedelec, Abraham, and
  Doll{\'e}]{offlineabtest2018}
Gilotte, A., Calauz{\`e}nes, C., Nedelec, T., Abraham, A., and Doll{\'e}, S.
\newblock {Offline A/B Testing for Recommender Systems}.
\newblock \emph{Proceedings of the Eleventh ACM International Conference on Web
  Search and Data Mining}, 2018.

\bibitem[Goodfellow(2016)]{gantutorial2017}
Goodfellow, I.
\newblock Nips 2016 tutorial: Generative adversarial networks.
\newblock \emph{arXiv preprint arXiv:1701.00160}, 2016.

\bibitem[Goodfellow et~al.(2014)Goodfellow, Pouget-Abadie, Mirza, Xu,
  Warde-Farley, Ozair, Courville, and Bengio]{gan2014}
Goodfellow, I., Pouget-Abadie, J., Mirza, M., Xu, B., Warde-Farley, D., Ozair,
  S., Courville, A., and Bengio, Y.
\newblock {Generative Adversarial Nets}.
\newblock In Ghahramani, Z., Welling, M., Cortes, C., Lawrence, N., and
  Weinberger, K.~Q. (eds.), \emph{{Advances in Neural Information Processing
  Systems}}, volume~27. Curran Associates, Inc., 2014.

\bibitem[Gottesman et~al.(2018)Gottesman, Johansson, Meier, Dent, Lee,
  Srinivasan, Zhang, Ding, Wihl, Peng, Yao, Lage, Mosch, wei H.~Lehman,
  Komorowski, Faisal, Celi, Sontag, and Doshi-Velez]{rlhealth2018}
Gottesman, O., Johansson, F.~D., Meier, J., Dent, J., Lee, D., Srinivasan, S.,
  Zhang, L., Ding, Y., Wihl, D., Peng, X., Yao, J., Lage, I., Mosch, C., wei
  H.~Lehman, L., Komorowski, M., Faisal, A., Celi, L., Sontag, D., and
  Doshi-Velez, F.
\newblock {Evaluating Reinforcement Learning Algorithms in Observational Health
  Settings}.
\newblock \emph{ArXiv}, abs/1805.12298, 2018.

\bibitem[Gretton et~al.(2012)Gretton, Borgwardt, Rasch, Sch{\"o}lkopf, and
  Smola]{kerneltwosampletest2012}
Gretton, A., Borgwardt, K., Rasch, M., Sch{\"o}lkopf, B., and Smola, A.
\newblock {A Kernel Two-Sample Test}.
\newblock \emph{J. Mach. Learn. Res.}, 13:\penalty0 723--773, 2012.

\bibitem[Gulcehre et~al.(2021)Gulcehre, Colmenarejo, ziyu wang, Sygnowski,
  Paine, Zolna, Chen, Hoffman, Pascanu, and de~Freitas]{addressingextra2021}
Gulcehre, C., Colmenarejo, S.~G., ziyu wang, Sygnowski, J., Paine, T., Zolna,
  K., Chen, Y., Hoffman, M., Pascanu, R., and de~Freitas, N.
\newblock {Addressing Extrapolation Error in Deep Offline Reinforcement
  Learning}.
\newblock 2021.

\bibitem[Gulrajani et~al.(2017)Gulrajani, Ahmed, Arjovsky, Dumoulin, and
  Courville]{wgangp2017}
Gulrajani, I., Ahmed, F., Arjovsky, M., Dumoulin, V., and Courville, A.~C.
\newblock {Improved Training of Wasserstein GANs}.
\newblock In \emph{{Advances in neural information processing systems}}, 2017.

\bibitem[Haarnoja et~al.(2017)Haarnoja, Tang, Abbeel, and
  Levine]{rlenergypolicy2017}
Haarnoja, T., Tang, H., Abbeel, P., and Levine, S.
\newblock {Reinforcement Learning with Deep Energy-Based Policies}.
\newblock In \emph{{International Conference on Machine Learning}}, 2017.

\bibitem[Haarnoja et~al.(2018{\natexlab{a}})Haarnoja, Zhou, Abbeel, and
  Levine]{sac2018}
Haarnoja, T., Zhou, A., Abbeel, P., and Levine, S.
\newblock {Soft Actor-Critic: Off-Policy Maximum Entropy Deep Reinforcement
  Learning with a Stochastic {Actor}}.
\newblock In Dy, J. and Krause, A. (eds.), \emph{{Proceedings of the 35th
  International Conference on Machine Learning}}, volume~80 of
  \emph{Proceedings of Machine Learning Research}, pp.\  1861--1870. PMLR,
  10--15 Jul 2018{\natexlab{a}}.

\bibitem[Haarnoja et~al.(2018{\natexlab{b}})Haarnoja, Zhou, Hartikainen,
  Tucker, Ha, Tan, Kumar, Zhu, Gupta, Abbeel, and Levine]{sacnew2018}
Haarnoja, T., Zhou, A., Hartikainen, K., Tucker, G., Ha, S., Tan, J., Kumar,
  V., Zhu, H., Gupta, A., Abbeel, P., and Levine, S.
\newblock {Soft Actor-Critic Algorithms and Applications}.
\newblock \emph{ArXiv}, abs/1812.05905, 2018{\natexlab{b}}.

\bibitem[Hasselt(2010)]{doubleq2010}
Hasselt, H.~V.
\newblock {Double Q-learning}.
\newblock In \emph{{Advances in neural information processing systems}}, 2010.

\bibitem[Ho \& Ermon(2016)Ho and Ermon]{gail2016}
Ho, J. and Ermon, S.
\newblock {Generative Adversarial Imitation Learning}.
\newblock In Lee, D., Sugiyama, M., Luxburg, U., Guyon, I., and Garnett, R.
  (eds.), \emph{{Advances in Neural Information Processing Systems}},
  volume~29. Curran Associates, Inc., 2016.

\bibitem[Jaques et~al.(2019)Jaques, Ghandeharioun, Shen, Ferguson, Lapedriza,
  Jones, Gu, and Picard]{offlinerldialog2019}
Jaques, N., Ghandeharioun, A., Shen, J.~H., Ferguson, C., Lapedriza, {\`A}.,
  Jones, N.~J., Gu, S., and Picard, R.~W.
\newblock {Way Off-Policy Batch Deep Reinforcement Learning of Implicit Human
  Preferences in Dialog}.
\newblock \emph{ArXiv}, abs/1907.00456, 2019.

\bibitem[Kallus \& Zhou(2020)Kallus and Zhou]{confoundingrobust2020}
Kallus, N. and Zhou, A.
\newblock {Confounding-robust policy evaluation in infinite-horizon
  reinforcement learning}.
\newblock \emph{{Advances in Neural Information Processing Systems}},
  33:\penalty0 22293--22304, 2020.

\bibitem[Kingma \& Ba(2014)Kingma and Ba]{adam2014}
Kingma, D.~P. and Ba, J.
\newblock {Adam: A Method for Stochastic Optimization}.
\newblock In \emph{{International Conference on Learning Representations}},
  2014.

\bibitem[Kingma \& Welling(2013)Kingma and Welling]{vae2013}
Kingma, D.~P. and Welling, M.
\newblock Auto-encoding variational bayes.
\newblock \emph{arXiv preprint arXiv:1312.6114}, 2013.

\bibitem[Kostrikov et~al.(2021{\natexlab{a}})Kostrikov, Nair, and
  Levine]{iql2021}
Kostrikov, I., Nair, A., and Levine, S.
\newblock {Offline Reinforcement Learning with Implicit Q-Learning}.
\newblock \emph{{ArXiv}}, abs/2110.06169, 2021{\natexlab{a}}.

\bibitem[Kostrikov et~al.(2021{\natexlab{b}})Kostrikov, Tompson, Fergus, and
  Nachum]{fisherbrc2021}
Kostrikov, I., Tompson, J., Fergus, R., and Nachum, O.
\newblock {Offline Reinforcement Learning with Fisher Divergence Critic
  Regularization}.
\newblock In \emph{{International Conference on Machine Learning}},
  2021{\natexlab{b}}.

\bibitem[Kumar et~al.(2019)Kumar, Fu, Soh, Tucker, and Levine]{bear2019}
Kumar, A., Fu, J., Soh, M., Tucker, G., and Levine, S.
\newblock {Stabilizing Off-Policy Q-Learning via Bootstrapping Error
  Reduction}.
\newblock In \emph{{Advances in Neural Information Processing Systems}},
  volume~32. Curran Associates, Inc., 2019.

\bibitem[Kumar et~al.(2020)Kumar, Zhou, Tucker, and Levine]{cql2020}
Kumar, A., Zhou, A., Tucker, G., and Levine, S.
\newblock {Conservative Q-Learning for Offline Reinforcement Learning}.
\newblock In Larochelle, H., Ranzato, M., Hadsell, R., Balcan, M.~F., and Lin,
  H. (eds.), \emph{{Advances in Neural Information Processing Systems}},
  volume~33, pp.\  1179--1191. Curran Associates, Inc., 2020.

\bibitem[Kuznetsov et~al.(2020)Kuznetsov, Shvechikov, Grishin, and
  Vetrov]{tqc2020}
Kuznetsov, A., Shvechikov, P., Grishin, A., and Vetrov, D.
\newblock {Controlling Overestimation Bias with Truncated Mixture of Continuous
  Distributional Quantile Critics}.
\newblock \emph{ArXiv}, abs/2005.04269, 2020.

\bibitem[Lange et~al.(2012)Lange, Gabel, and Riedmiller]{batchrl2012}
Lange, S., Gabel, T., and Riedmiller, M.
\newblock \emph{{Batch Reinforcement Learning}}, pp.\  45--73.
\newblock Springer Berlin Heidelberg, Berlin, Heidelberg, 2012.
\newblock ISBN 978-3-642-27645-3.
\newblock \doi{10.1007/978-3-642-27645-3_2}.

\bibitem[Langville \& Meyer(2004)Langville and Meyer]{deeppagerank2004}
Langville, A.~N. and Meyer, C.~D.
\newblock {Deeper inside {P}age{R}ank}.
\newblock \emph{Internet Mathematics}, 1\penalty0 (3):\penalty0 335--380, 2004.
\newblock ISSN 1542-7951.

\bibitem[Laroche \& Trichelair(2019)Laroche and Trichelair]{sibb2019}
Laroche, R. and Trichelair, P.
\newblock {Safe Policy Improvement with Baseline Bootstrapping}.
\newblock In \emph{{International Conference on Machine Learning}}, 2019.

\bibitem[Lee et~al.(2021{\natexlab{a}})Lee, Jeon, Lee, Pineau, and
  Kim]{optidice2021}
Lee, J., Jeon, W., Lee, B.-J., Pineau, J., and Kim, K.-E.
\newblock {OptiDICE: Offline Policy Optimization via Stationary Distribution
  Correction Estimation}.
\newblock \emph{ArXiv}, abs/2106.10783, 2021{\natexlab{a}}.

\bibitem[Lee et~al.(2021{\natexlab{b}})Lee, Laskin, Srinivas, and
  Abbeel]{sunrise2021}
Lee, K., Laskin, M., Srinivas, A., and Abbeel, P.
\newblock {SUNRISE: A Simple Unified Framework for Ensemble Learning in Deep
  Reinforcement Learning}.
\newblock In \emph{{International Conference on Machine Learning}},
  2021{\natexlab{b}}.

\bibitem[Levine et~al.(2020)Levine, Kumar, Tucker, and Fu]{offlinetutorial2020}
Levine, S., Kumar, A., Tucker, G., and Fu, J.
\newblock Offline reinforcement learning: Tutorial, review, and perspectives on
  open problems.
\newblock \emph{arXiv preprint arXiv:2005.01643}, 2020.

\bibitem[Li et~al.(2017)Li, Chang, Cheng, Yang, and
  P{\'o}czos]{mmdganweightclip2017}
Li, C.-L., Chang, W.-C., Cheng, Y., Yang, Y., and P{\'o}czos, B.
\newblock {MMD GAN: Towards Deeper Understanding of Moment Matching Network}.
\newblock In \emph{{Advances in neural information processing systems}}, 2017.

\bibitem[Lillicrap et~al.(2016)Lillicrap, Hunt, Pritzel, Heess, Erez, Tassa,
  Silver, and Wierstra]{ddpg2016}
Lillicrap, T., Hunt, J.~J., Pritzel, A., Heess, N., Erez, T., Tassa, Y.,
  Silver, D., and Wierstra, D.
\newblock {Continuous Control with Deep Reinforcement Learning}.
\newblock \emph{CoRR}, abs/1509.02971, 2016.

\bibitem[Lin(1991)]{jsd1991}
Lin, J.
\newblock {Divergence Measures Based on the Shannon Entropy}.
\newblock \emph{IEEE Transactions on Information theory}, 37:\penalty0
  145--151, 1991.

\bibitem[Lin(1992)]{replaybuffer1992}
Lin, L.-J.
\newblock {Self-Improving Reactive Agents Based on Reinforcement Learning,
  Planning and Teaching}.
\newblock \emph{Machine Learning}, 8\penalty0 (3--4):\penalty0 293--321, 1992.

\bibitem[Liu et~al.(2018)Liu, Li, Tang, and Zhou]{breakingcurse2018}
Liu, Q., Li, L., Tang, Z., and Zhou, D.
\newblock {Breaking the Curse of Horizon: Infinite-Horizon Off-Policy
  Estimation}.
\newblock In \emph{{Advances in neural information processing systems}}, 2018.

\bibitem[Matsushima et~al.(2021)Matsushima, Furuta, Matsuo, Nachum, and
  Gu]{bremen2021}
Matsushima, T., Furuta, H., Matsuo, Y., Nachum, O., and Gu, S.~S.
\newblock {Deployment-Efficient Reinforcement Learning via Model-Based Offline
  Optimization}.
\newblock \emph{{International Conference on Learning Representations}},
  abs/2006.03647, 2021.

\bibitem[Mescheder et~al.(2018)Mescheder, Geiger, and
  Nowozin]{mescheder2018training}
Mescheder, L., Geiger, A., and Nowozin, S.
\newblock Which training methods for {GAN}s do actually converge?
\newblock In \emph{International conference on machine learning}, pp.\
  3481--3490. PMLR, 2018.

\bibitem[Meyer(2000)]{matrixanalysis2000}
Meyer, C.~D.
\newblock \emph{{Matrix Analysis and Applied Linear Algebra}}.
\newblock Society for Industrial and Applied Mathematics, USA, 2000.
\newblock ISBN 0898714540.

\bibitem[Mirza \& Osindero(2014)Mirza and Osindero]{cgan2014}
Mirza, M. and Osindero, S.
\newblock {Conditional Generative Adversarial Nets}.
\newblock \emph{ArXiv}, abs/1411.1784, 2014.

\bibitem[Miyato et~al.(2018)Miyato, Kataoka, Koyama, and
  Yoshida]{spectralnorm2018}
Miyato, T., Kataoka, T., Koyama, M., and Yoshida, Y.
\newblock {Spectral Normalization for Generative Adversarial Networks}.
\newblock \emph{ArXiv}, abs/1802.05957, 2018.

\bibitem[Mnih et~al.(2013)Mnih, Kavukcuoglu, Silver, Graves, Antonoglou,
  Wierstra, and Riedmiller]{dqn2013}
Mnih, V., Kavukcuoglu, K., Silver, D., Graves, A., Antonoglou, I., Wierstra,
  D., and Riedmiller, M.
\newblock Playing atari with deep reinforcement learning.
\newblock \emph{arXiv preprint arXiv:1312.5602}, 2013.

\bibitem[Mousavi et~al.(2020)Mousavi, Li, Liu, and Zhou]{blackope2020}
Mousavi, A., Li, L., Liu, Q., and Zhou, D.
\newblock {Black-box Off-policy Estimation for Infinite-Horizon Reinforcement
  Learning}.
\newblock \emph{ArXiv}, abs/2003.11126, 2020.

\bibitem[M{\"u}ller(1997)]{ipm1997}
M{\"u}ller, A.
\newblock {Integral Probability Metrics and Their Generating Classes of
  Functions}.
\newblock \emph{Advances in Applied Probability}, 29:\penalty0 429--443, 1997.

\bibitem[Nachum et~al.(2019)Nachum, Dai, Kostrikov, Chow, Li, and
  Schuurmans]{algaedice2019}
Nachum, O., Dai, B., Kostrikov, I., Chow, Y., Li, L., and Schuurmans, D.
\newblock {AlgaeDICE: Policy Gradient from Arbitrary Experience}.
\newblock \emph{ArXiv}, abs/1912.02074, 2019.

\bibitem[Nie et~al.(2019)Nie, Brunskill, and Wager]{whentotreat2019}
Nie, X., Brunskill, E., and Wager, S.
\newblock {Learning When-to-Treat Policies}.
\newblock \emph{Journal of the American Statistical Association}, 116:\penalty0
  392 -- 409, 2019.

\bibitem[Page et~al.(1998)Page, Brin, Motwani, and Winograd]{pagerank1998}
Page, L., Brin, S., Motwani, R., and Winograd, T.
\newblock {The PageRank Citation Ranking: Bringing order to the Web}.
\newblock In \emph{{Proceedings of the 7th International World Wide Web
  Conference}}, pp.\  161--172, Brisbane, Australia, 1998.

\bibitem[Puterman(2014)]{markovdp1994}
Puterman, M.~L.
\newblock \emph{Markov decision processes: discrete stochastic dynamic
  programming}.
\newblock John Wiley \& Sons, 2014.

\bibitem[Radford et~al.(2016)Radford, Metz, and Chintala]{dcgan2016}
Radford, A., Metz, L., and Chintala, S.
\newblock {Unsupervised Representation Learning with Deep Convolutional
  Generative Adversarial Networks}.
\newblock \emph{CoRR}, abs/1511.06434, 2016.

\bibitem[Rajeswaran et~al.(2018)Rajeswaran, Kumar, Gupta, Schulman, Todorov,
  and Levine]{adroit2018}
Rajeswaran, A., Kumar, V., Gupta, A., Schulman, J., Todorov, E., and Levine, S.
\newblock {Learning Complex Dexterous Manipulation with Deep Reinforcement
  Learning and Demonstrations}.
\newblock \emph{ArXiv}, abs/1709.10087, 2018.

\bibitem[Salimans et~al.(2016)Salimans, Goodfellow, Zaremba, Cheung, Radford,
  and Chen]{improvetechgan2016}
Salimans, T., Goodfellow, I., Zaremba, W., Cheung, V., Radford, A., and Chen,
  X.
\newblock {Improved Techniques for Training GANs}.
\newblock In \emph{{Advances in neural information processing systems}}, 2016.

\bibitem[Schulman et~al.(2015)Schulman, Levine, Abbeel, Jordan, and
  Moritz]{trpo2015}
Schulman, J., Levine, S., Abbeel, P., Jordan, M.~I., and Moritz, P.
\newblock {Trust Region Policy Optimization}.
\newblock \emph{ArXiv}, abs/1502.05477, 2015.

\bibitem[Sejdinovic et~al.(2013)Sejdinovic, Sriperumbudur, Gretton, and
  Fukumizu]{energydistmmd2013}
Sejdinovic, D., Sriperumbudur, B., Gretton, A., and Fukumizu, K.
\newblock {Equivalence of Distance-based and RKHS-based Statistics in
  Hypothesis Testing}.
\newblock \emph{The Annals of Statistics}, 41\penalty0 (5):\penalty0
  2263--2291, 2013.
\newblock ISSN 00905364, 21688966.

\bibitem[Siegel et~al.(2020)Siegel, Springenberg, Berkenkamp, Abdolmaleki,
  Neunert, Lampe, Hafner, Heess, and Riedmiller]{abm2020}
Siegel, N.~Y., Springenberg, J.~T., Berkenkamp, F., Abdolmaleki, A., Neunert,
  M., Lampe, T., Hafner, R., Heess, N., and Riedmiller, M.
\newblock Keep doing what worked: Behavioral modelling priors for offline
  reinforcement learning.
\newblock \emph{arXiv preprint arXiv:2002.08396}, 2020.

\bibitem[Sinha et~al.(2022)Sinha, Mandlekar, and Garg]{s4rl2021}
Sinha, S., Mandlekar, A., and Garg, A.
\newblock S4rl: Surprisingly simple self-supervision for offline reinforcement
  learning in robotics.
\newblock In \emph{Conference on Robot Learning}, pp.\  907--917. PMLR, 2022.

\bibitem[Sohn et~al.(2015)Sohn, Lee, and Yan]{cvae2015}
Sohn, K., Lee, H., and Yan, X.
\newblock {Learning Structured Output Representation using Deep Conditional
  Generative Models}.
\newblock In \emph{{Advances in neural information processing systems}}, 2015.

\bibitem[Sutton \& Barto(2018)Sutton and Barto]{rlintro2018}
Sutton, R.~S. and Barto, A.~G.
\newblock \emph{Reinforcement learning: An introduction}.
\newblock MIT press, 2018.

\bibitem[Swaminathan et~al.(2017)Swaminathan, Krishnamurthy, Agarwal,
  Dud{\'i}k, Langford, Jose, and Zitouni]{operecom2017}
Swaminathan, A., Krishnamurthy, A., Agarwal, A., Dud{\'i}k, M., Langford, J.,
  Jose, D., and Zitouni, I.
\newblock {Off-policy Evaluation for Slate Recommendation}.
\newblock In \emph{{Advances in neural information processing systems}}, 2017.

\bibitem[Tseng et~al.(2017)Tseng, Luo, Cui, Chien, Haken, and
  Naqa]{rllungcancer2017}
Tseng, H., Luo, Y., Cui, S., Chien, J.-T., Haken, R.~T., and Naqa, I.
\newblock {Deep Reinforcement Learning for Automated Radiation Adaptation in
  Lung Cancer}.
\newblock \emph{Medical Physics}, 44:\penalty0 6690?6705, 2017.

\bibitem[Urp\'i et~al.(2021)Urp\'i, Curi, and Krause]{oraac2021}
Urp\'i, N.~A., Curi, S., and Krause, A.
\newblock {Risk-Averse Offline Reinforcement Learning}.
\newblock \emph{ArXiv}, abs/2102.05371, 2021.

\bibitem[Wang et~al.(2020)Wang, Novikov, Zolna, Merel, Springenberg, Reed,
  Shahriari, Siegel, Gulcehre, Heess, and de~Freitas]{crr2020}
Wang, Z., Novikov, A., Zolna, K., Merel, J.~S., Springenberg, J.~T., Reed,
  S.~E., Shahriari, B., Siegel, N., Gulcehre, C., Heess, N., and de~Freitas, N.
\newblock {Critic Regularized Regression}.
\newblock In Larochelle, H., Ranzato, M., Hadsell, R., Balcan, M.~F., and Lin,
  H. (eds.), \emph{{Advances in Neural Information Processing Systems}},
  volume~33, pp.\  7768--7778. Curran Associates, Inc., 2020.

\bibitem[Wasserman(2006)]{nonparamstat2006}
Wasserman, L.
\newblock \emph{{All of Nonparametric Statistics}}.
\newblock Springer, 2006.

\bibitem[White(2016)]{samplegenerative2016}
White, T.
\newblock Sampling generative networks.
\newblock \emph{arXiv preprint arXiv:1609.04468}, 2016.

\bibitem[Wu et~al.(2019)Wu, Tucker, and Nachum]{brac2019}
Wu, Y., Tucker, G., and Nachum, O.
\newblock Behavior regularized offline reinforcement learning.
\newblock \emph{arXiv preprint arXiv:1911.11361}, 2019.

\bibitem[Wu et~al.(2021)Wu, Zhai, Srivastava, Susskind, Zhang, Salakhutdinov,
  and Goh]{uwac2021}
Wu, Y., Zhai, S., Srivastava, N., Susskind, J., Zhang, J., Salakhutdinov, R.,
  and Goh, H.
\newblock {Uncertainty Weighted Actor-Critic for Offline Reinforcement
  Learning}.
\newblock In \emph{{International Conference on Machine Learning}}, 2021.

\bibitem[Yu et~al.(2021)Yu, Kumar, Rafailov, Rajeswaran, Levine, and
  Finn]{combo2021}
Yu, T., Kumar, A., Rafailov, R., Rajeswaran, A., Levine, S., and Finn, C.
\newblock {COMBO: Conservative Offline Model-Based Policy Optimization}.
\newblock \emph{ArXiv}, abs/2102.08363, 2021.

\bibitem[Yue et~al.(2020)Yue, Wang, and Zhou]{idac2020}
Yue, Y., Wang, Z., and Zhou, M.
\newblock {Implicit Distributional Reinforcement Learning}.
\newblock In Larochelle, H., Ranzato, M., Hadsell, R., Balcan, M., and Lin, H.
  (eds.), \emph{{Advances in Neural Information Processing Systems 33: Annual
  Conference on Neural Information Processing Systems 2020, NeurIPS 2020,
  December 6-12, 2020, virtual}}, 2020.

\bibitem[Yurtsever et~al.(2020)Yurtsever, Lambert, Carballo, and
  Takeda]{autodrivesurvey2020}
Yurtsever, E., Lambert, J., Carballo, A., and Takeda, K.
\newblock {A Survey of Autonomous Driving: Common Practices and Emerging
  Technologies}.
\newblock \emph{IEEE Access}, 8:\penalty0 58443--58469, 2020.

\bibitem[Zhang et~al.(2020)Zhang, Dai, Li, and Schuurmans]{gendice2020}
Zhang, R., Dai, B., Li, L., and Schuurmans, D.
\newblock {GenDICE: Generalized Offline Estimation of Stationary Values}.
\newblock \emph{ArXiv}, abs/2002.09072, 2020.

\end{thebibliography}


\clearpage
\appendix

\begin{center}
\Large
\textbf{Appendix}
\end{center}

\section{Related Work} \label{sec:related_work}
\textbf{Offline Reinforcement Learning.}
Three major themes currently exist in offline-RL research.
The first focuses on more robustly estimating the action-value function \citep{rem2020,addressingextra2021} or providing a conservative estimate of the Q-values \citep{cql2020,combo2021,s4rl2021}, which may better guide the policy optimization process.
The second research theme aims at designing a tactful behavior-cloning scheme so as to learn only from ``good'' actions in the offline dataset \citep{crr2020,decisiontrans2021}.
In this paper we adopt the third line of research that tries to constrain the current policy to be close to the behavior policy during the training process, under the notion that Q-value estimates at unfamiliar state-action pairs can be pathologically worse due to a lack of supervised training.
Specifically, \citet{bear2019} and \citet{uwac2021} use conditional variational autoencoder (CVAE) \citep{vae2013, cvae2015} to train a behavior cloning policy to sample multiple actions at each state for calculating the MMD constraint.
\citet{brac2019}, \citet{abm2020}, and \citet{mabe2021} fit a (Gaussian) behavioral prior to the offline dataset trained by (weighted) maximum likelihood objective.
\citet{offlinerldialog2019} consider a pre-trained generative prior of human dialog data before applying KL-control to the current policy.
Note that these works essentially constrain the distance between the current policy and the \emph{cloned} behavior policy, where the latter may deviate from the \emph{true} behavior.
\citet{sibb2019} assume a known stochastic data-collecting behavior policy.
\citet{bcq2019} and \citet{oraac2021} implicitly control the state-conditional action distribution by decomposing action into a behavior cloning component, trained by fitting a CVAE onto the offline data, and a perturbation component, trained to optimize the (risk-averse) returns.
Besides, some work, such as \citet{brac2019}, directly estimates and regularizes the divergence between the state-conditional action distributions, implementing the regularization as Eqs.~\eqref{eq:cond_gan_objective} and \eqref{eq:cond_true_fake_sample}.
Further, we notice that most of the existing offline RL work use deterministic or uni-modal Gaussian policy, whose flexibility is limited, as discussed in Sections~\ref{sec:basic}, \ref{sec:main_result} and \ref{sec:abaltion_study} \textbf{(a)}.
In the paper, we \textbf{(1)} develop a framework to train a flexible fully-implicit policy; and \textbf{(2)} propose a simple modification for improved matching \textit{w.r.t.} the dual form of JSD and IPM, which avoids using a single point to estimate the divergence between two distributions and removes the need for a good approximator of the behavior policy (\Secref{sec:basic}).

\textbf{Online Off-policy RL.} A large class of modern online off-policy deep RL algorithms trains the policy using experience replay buffer \citep{replaybuffer1992}, which is a storage of the rollouts of past policies encountered in the training process \citep{dqn2013,ddpg2016,rlenergypolicy2017,td32018,sacnew2018,tqc2020,idac2020,sunrise2021}.
This approach essentially use the state-visitation frequency of past policies to approximate that of the current policy (Eq.~\ref{eq:ac_actor_approx}).
This notion is adopted in policy-matching offline-RL algorithms in both the policy improvement step and in the implementation of the policy-matching regularization, since ideally one would like to match the undiscounted state-action visitation induced by the current policy with the offline dataset.
We note that in the standard implementation of off-policy RL algorithms, the discount factor does not act on the collection and the utilization of the replay buffer.
Hence, the replay buffer can be viewed as samples from the \emph{undiscounted} state-action visitation induced by the current and past policies.
Similar to our work, a GAN structure is adopted by GAIL \citep{gail2016} and its follow-ups. However, these works target imitation learning and require online interactions with the environment, and thus do not follow the offline RL setting.

\textbf{Computational Distribution Matching.} Many computationally efficient algorithms exist to match two probability distributions with respective to some statistical divergence.
GAN \citep{gan2014} approximately minimizes the Jensen--Shannon divergence between the the model's distribution and the data-generating distribution.
A similar adversarial training strategy is applied to estimate a class of statistical divergence, termed the Integral Probability Metrics \citep{ipm1997}, in a sample-based manner.
For example, \citet{wgan2017,wgangp2017,spectralnorm2018} estimate the Wasserstein-1 distance by enforcing the Lipschitz norm of the witness function to be bounded by $1$.
\citet{mmdganweightclip2017,mmdgangp2018} consider Maximum Mean Discrepancy (MMD) \citep{kerneltwosampletest2012} with learnable kernels.
\citet{cramerdist2017} study the energy distance, an instance of the MMD \citep{energydistmmd2013}.
In this paper, we focus on the classical GAN structure to approximately control the JSD between the current and behavior policies, since the GAN structure is simple, effective and well-studied.
Furthermore, we propose a simple modification to improve policy-matching \textit{w.r.t.} the dual form of JSD and IPM.
Other divergence metrics may also be applicable to our framework and are left for future work.

\section{Additional Tables}\label{sec:additional_tables}

Table~\ref{table:implicit_normal} - \ref{table:sigma_sweep} correspond to the results of our ablation study in \Secref{sec:abaltion_study}.

    \begin{table}[H] 
    \captionsetup{font=small}
\caption{\small Normalized returns for comparing the implicit and the Gaussian policy on the basic algorithm (\Secref{sec:basic}) on the D4RL suite of tasks. The reported number are the means and standard deviations of the normalized returns of the last five rollouts across five random seeds $\cbr{0,1,2,3,4}$.} 
\label{table:implicit_normal} 
\centering 
\begin{tabular}{l||cc}
\toprule
                 Task Name &                  GAN-Joint: Basic & GAN-Joint: Basic, Gaussian Policy \\
\midrule
              maze2d-umaze &   50.8 $\pm$ {\footnotesize 15.1} &   24.0 $\pm$ {\footnotesize 10.7} \\
             maze2d-medium &    44.6 $\pm$ {\footnotesize 9.1} &    -0.2 $\pm$ {\footnotesize 6.7} \\
              maze2d-large &   57.2 $\pm$ {\footnotesize 16.5} &    5.4 $\pm$ {\footnotesize 10.9} \\
        halfcheetah-medium &    43.8 $\pm$ {\footnotesize 0.4} &    43.7 $\pm$ {\footnotesize 0.3} \\
           walker2d-medium &    66.8 $\pm$ {\footnotesize 4.9} &   53.1 $\pm$ {\footnotesize 12.0} \\
             hopper-medium &   69.1 $\pm$ {\footnotesize 20.7} &   78.0 $\pm$ {\footnotesize 14.9} \\
 halfcheetah-medium-replay &    31.3 $\pm$ {\footnotesize 2.9} &    31.2 $\pm$ {\footnotesize 2.1} \\
    walker2d-medium-replay &    10.1 $\pm$ {\footnotesize 1.9} &     9.2 $\pm$ {\footnotesize 1.3} \\
      hopper-medium-replay &    33.6 $\pm$ {\footnotesize 7.9} &    25.2 $\pm$ {\footnotesize 2.2} \\
 halfcheetah-medium-expert &   70.5 $\pm$ {\footnotesize 11.1} &    75.5 $\pm$ {\footnotesize 9.3} \\
    walker2d-medium-expert &   67.4 $\pm$ {\footnotesize 13.5} &   58.8 $\pm$ {\footnotesize 16.1} \\
      hopper-medium-expert &   76.3 $\pm$ {\footnotesize 21.3} &   82.8 $\pm$ {\footnotesize 13.9} \\
                 pen-human &   61.0 $\pm$ {\footnotesize 16.6} &   52.9 $\pm$ {\footnotesize 18.9} \\
                pen-cloned &   23.6 $\pm$ {\footnotesize 16.7} &   37.2 $\pm$ {\footnotesize 14.7} \\
                pen-expert &  131.1 $\pm$ {\footnotesize 13.2} &  118.0 $\pm$ {\footnotesize 12.3} \\
               door-expert &   103.0 $\pm$ {\footnotesize 3.4} &   38.9 $\pm$ {\footnotesize 21.0} \\
               \midrule
             Average Score &                              \textbf{58.8} &                              45.9 \\
\bottomrule
\end{tabular}
\end{table}
  
\begin{table}[H] 
\captionsetup{font=small}
\caption{\small Normalized returns for comparing our full algorithms with their counterpart of no state-smoothing in the state-action joint-matching scheme. The reported number are the means and standard deviations of the normalized returns of the last five rollouts across five random seeds $\cbr{0,1,2,3,4}$. ``Joint'' denotes ``GAN-Joint'', ``Joint-$\alpha$'' denotes ``GAN-Joint-$\alpha$''.} 
\label{table:smoothing_joint_matching} 
\centering 
\def\arraystretch{1.}
\resizebox{\textwidth}{!}{
\begin{tabular}{l||cc|cc}
\toprule
                 Task Name &          Joint: Full & Joint: No Smoothing &          Joint-$\alpha$: Full & Joint-$\alpha$: No Smoothing \\
\midrule
              maze2d-umaze &   47.1 $\pm$ {\footnotesize 18.8} &                    47.3 $\pm$ {\footnotesize 10.5} &   58.8 $\pm$ {\footnotesize 22.7} &                    35.0 $\pm$ {\footnotesize 20.5} \\
             maze2d-medium &   74.3 $\pm$ {\footnotesize 25.5} &                    41.4 $\pm$ {\footnotesize 15.5} &   72.8 $\pm$ {\footnotesize 21.8} &                    56.5 $\pm$ {\footnotesize 35.7} \\
              maze2d-large &   63.5 $\pm$ {\footnotesize 21.2} &                    63.0 $\pm$ {\footnotesize 26.5} &  200.5 $\pm$ {\footnotesize 23.6} &                   114.9 $\pm$ {\footnotesize 72.4} \\
        halfcheetah-medium &    44.0 $\pm$ {\footnotesize 0.2} &                     43.9 $\pm$ {\footnotesize 0.4} &    44.0 $\pm$ {\footnotesize 0.2} &                     44.1 $\pm$ {\footnotesize 0.4} \\
           walker2d-medium &    69.3 $\pm$ {\footnotesize 8.8} &                    63.5 $\pm$ {\footnotesize 11.1} &    69.9 $\pm$ {\footnotesize 6.4} &                     62.3 $\pm$ {\footnotesize 8.8} \\
             hopper-medium &   66.1 $\pm$ {\footnotesize 24.0} &                    65.2 $\pm$ {\footnotesize 15.3} &   86.4 $\pm$ {\footnotesize 10.9} &                    78.1 $\pm$ {\footnotesize 18.9} \\
 halfcheetah-medium-replay &    33.0 $\pm$ {\footnotesize 1.8} &                     32.3 $\pm$ {\footnotesize 2.5} &    33.4 $\pm$ {\footnotesize 2.4} &                     31.3 $\pm$ {\footnotesize 1.7} \\
    walker2d-medium-replay &     9.3 $\pm$ {\footnotesize 2.0} &                     10.1 $\pm$ {\footnotesize 2.7} &     6.7 $\pm$ {\footnotesize 2.2} &                      6.1 $\pm$ {\footnotesize 3.4} \\
      hopper-medium-replay &    30.0 $\pm$ {\footnotesize 2.9} &                     29.6 $\pm$ {\footnotesize 2.0} &    30.9 $\pm$ {\footnotesize 3.2} &                     32.7 $\pm$ {\footnotesize 5.0} \\
 halfcheetah-medium-expert &   72.8 $\pm$ {\footnotesize 11.2} &                    69.7 $\pm$ {\footnotesize 10.3} &   72.6 $\pm$ {\footnotesize 11.1} &                    70.2 $\pm$ {\footnotesize 12.5} \\
    walker2d-medium-expert &   75.3 $\pm$ {\footnotesize 12.1} &                    72.3 $\pm$ {\footnotesize 18.6} &    79.6 $\pm$ {\footnotesize 1.9} &                    74.1 $\pm$ {\footnotesize 11.8} \\
      hopper-medium-expert &   86.4 $\pm$ {\footnotesize 19.0} &                    74.1 $\pm$ {\footnotesize 15.3} &   71.1 $\pm$ {\footnotesize 10.7} &                    72.8 $\pm$ {\footnotesize 22.8} \\
                 pen-human &   57.5 $\pm$ {\footnotesize 22.6} &                    55.7 $\pm$ {\footnotesize 18.3} &   71.0 $\pm$ {\footnotesize 23.2} &                    62.2 $\pm$ {\footnotesize 23.9} \\
                pen-cloned &   23.2 $\pm$ {\footnotesize 14.2} &                    22.4 $\pm$ {\footnotesize 15.4} &    27.6 $\pm$ {\footnotesize 7.1} &                    28.0 $\pm$ {\footnotesize 12.2} \\
                pen-expert &  140.2 $\pm$ {\footnotesize 12.9} &                    137.6 $\pm$ {\footnotesize 9.9} &  134.5 $\pm$ {\footnotesize 10.8} &                   134.3 $\pm$ {\footnotesize 14.4} \\
               door-expert &   103.5 $\pm$ {\footnotesize 0.9} &                    101.5 $\pm$ {\footnotesize 4.3} &   102.2 $\pm$ {\footnotesize 4.5} &                    100.8 $\pm$ {\footnotesize 5.0} \\
               \midrule
             Average Score &                              \textbf{62.2} &                                               58.1 &                              \textbf{72.6} &                                               62.7 \\
\bottomrule
\end{tabular}
}
\end{table}


\begin{table}[H] 
\captionsetup{font=small}
\caption{
\small
Normalized returns for comparing our full algorithms with their counterpart of no state-smoothing in the Bellman backup. The reported number are the means and standard deviations of the normalized returns of the last five rollouts across five random seeds $\cbr{0,1,2,3,4}$. ``Joint'' denotes ``GAN-Joint'', ``Joint-$\alpha$'' denotes ``GAN-Joint-$\alpha$''.} 
\label{table:smoothing_bellman} 
\def\arraystretch{1.}
\resizebox{\textwidth}{!}{
\centering 
\begin{tabular}{l||cc|cc}
\toprule
                 Task Name &                   Joint: Full & Joint: No Smoothing  &          Joint-$\alpha$: Full & Joint-$\alpha$: No Smoothing \\
\midrule
              maze2d-umaze &   47.1 $\pm$ {\footnotesize 18.8} &                  50.8 $\pm$ {\footnotesize 16.8} &   58.8 $\pm$ {\footnotesize 22.7} &                    51.6 $\pm$ {\footnotesize 28.0} \\
             maze2d-medium &   74.3 $\pm$ {\footnotesize 25.5} &                  53.8 $\pm$ {\footnotesize 20.5} &   72.8 $\pm$ {\footnotesize 21.8} &                    57.5 $\pm$ {\footnotesize 27.3} \\
              maze2d-large &   63.5 $\pm$ {\footnotesize 21.2} &                   55.8 $\pm$ {\footnotesize 6.7} &  200.5 $\pm$ {\footnotesize 23.6} &                   132.2 $\pm$ {\footnotesize 66.7} \\
        halfcheetah-medium &    44.0 $\pm$ {\footnotesize 0.2} &                   44.0 $\pm$ {\footnotesize 0.3} &    44.0 $\pm$ {\footnotesize 0.2} &                     44.1 $\pm$ {\footnotesize 0.3} \\
           walker2d-medium &    69.3 $\pm$ {\footnotesize 8.8} &                   62.6 $\pm$ {\footnotesize 9.3} &    69.9 $\pm$ {\footnotesize 6.4} &                    63.6 $\pm$ {\footnotesize 12.0} \\
             hopper-medium &   66.1 $\pm$ {\footnotesize 24.0} &                  63.6 $\pm$ {\footnotesize 16.4} &   86.4 $\pm$ {\footnotesize 10.9} &                    72.8 $\pm$ {\footnotesize 18.8} \\
 halfcheetah-medium-replay &    33.0 $\pm$ {\footnotesize 1.8} &                   32.0 $\pm$ {\footnotesize 1.7} &    33.4 $\pm$ {\footnotesize 2.4} &                     34.8 $\pm$ {\footnotesize 1.0} \\
    walker2d-medium-replay &     9.3 $\pm$ {\footnotesize 2.0} &                    9.5 $\pm$ {\footnotesize 1.9} &     6.7 $\pm$ {\footnotesize 2.2} &                      8.3 $\pm$ {\footnotesize 2.4} \\
      hopper-medium-replay &    30.0 $\pm$ {\footnotesize 2.9} &                   28.6 $\pm$ {\footnotesize 2.1} &    30.9 $\pm$ {\footnotesize 3.2} &                     30.9 $\pm$ {\footnotesize 2.9} \\
 halfcheetah-medium-expert &   72.8 $\pm$ {\footnotesize 11.2} &                   69.3 $\pm$ {\footnotesize 8.6} &   72.6 $\pm$ {\footnotesize 11.1} &                     71.4 $\pm$ {\footnotesize 9.6} \\
    walker2d-medium-expert &   75.3 $\pm$ {\footnotesize 12.1} &                   83.7 $\pm$ {\footnotesize 9.1} &    79.6 $\pm$ {\footnotesize 1.9} &                     72.1 $\pm$ {\footnotesize 4.7} \\
      hopper-medium-expert &   86.4 $\pm$ {\footnotesize 19.0} &                  73.2 $\pm$ {\footnotesize 12.0} &   71.1 $\pm$ {\footnotesize 10.7} &                    81.4 $\pm$ {\footnotesize 20.0} \\
                 pen-human &   57.5 $\pm$ {\footnotesize 22.6} &                  45.7 $\pm$ {\footnotesize 25.2} &   71.0 $\pm$ {\footnotesize 23.2} &                    53.3 $\pm$ {\footnotesize 27.4} \\
                pen-cloned &   23.2 $\pm$ {\footnotesize 14.2} &                  23.4 $\pm$ {\footnotesize 12.9} &    27.6 $\pm$ {\footnotesize 7.1} &                    23.6 $\pm$ {\footnotesize 15.7} \\
                pen-expert &  140.2 $\pm$ {\footnotesize 12.9} &                 131.3 $\pm$ {\footnotesize 12.3} &  134.5 $\pm$ {\footnotesize 10.8} &                   130.6 $\pm$ {\footnotesize 13.9} \\
               door-expert &   103.5 $\pm$ {\footnotesize 0.9} &                  101.8 $\pm$ {\footnotesize 1.0} &   102.2 $\pm$ {\footnotesize 4.5} &                    101.4 $\pm$ {\footnotesize 4.3} \\
               \midrule
             Average Score &                              \textbf{62.2} &                                             58.1 &                              \textbf{72.6} &                                               64.4 \\
\bottomrule
\end{tabular}
}
\end{table}
\begin{table}[H] 
\captionsetup{font=small}
\caption{
\small
Normalized returns under several values of $\sigma \triangleq \sigma_B = \sigma_J$ (Appendix \ref{sec:tech_gan}) in the full algorithm ``GAN-joint''. The reported number are the means and standard deviations of the normalized returns of the last five rollouts across three random seeds $\cbr{0,1,2}$.} 
\label{table:sigma_sweep} 
\centering 
\def\arraystretch{1.1}
\resizebox{\textwidth}{!}{
\begin{tabular}{l||cccccc}
\toprule
                 Task Name & $\sigma = 1 \times 10^{-2}$ & $\sigma = 3 \times 10^{-3}$ & $\sigma = 1 \times 10^{-3}$ & $\sigma = 3  \times 10^{-4}$ & $\sigma = 1 \times 10^{-4}$ & $\sigma = 0$ \\
\midrule
              maze2d-umaze &                    48.4 $\pm$ {\footnotesize 21.4} &                    48.3 $\pm$ {\footnotesize 19.7} &                    41.7 $\pm$ {\footnotesize 11.5} &                    40.1 $\pm$ {\footnotesize 16.9} &                    54.9 $\pm$ {\footnotesize 10.4} &                    50.8 $\pm$ {\footnotesize 24.0} \\
             maze2d-medium &                    58.7 $\pm$ {\footnotesize 33.6} &                     48.7 $\pm$ {\footnotesize 7.4} &                    64.0 $\pm$ {\footnotesize 23.9} &                    69.6 $\pm$ {\footnotesize 25.6} &                    46.9 $\pm$ {\footnotesize 15.5} &                     26.4 $\pm$ {\footnotesize 5.7} \\
              maze2d-large &                    87.1 $\pm$ {\footnotesize 17.9} &                    57.6 $\pm$ {\footnotesize 21.3} &                    62.4 $\pm$ {\footnotesize 13.3} &                    71.3 $\pm$ {\footnotesize 26.0} &                     61.0 $\pm$ {\footnotesize 8.6} &                    62.3 $\pm$ {\footnotesize 32.3} \\
        halfcheetah-medium &                     43.0 $\pm$ {\footnotesize 0.4} &                     43.7 $\pm$ {\footnotesize 0.3} &                     43.9 $\pm$ {\footnotesize 0.4} &                     44.1 $\pm$ {\footnotesize 0.3} &                     43.8 $\pm$ {\footnotesize 0.3} &                     44.0 $\pm$ {\footnotesize 0.4} \\
           walker2d-medium &                    56.9 $\pm$ {\footnotesize 10.4} &                     66.4 $\pm$ {\footnotesize 7.9} &                    68.8 $\pm$ {\footnotesize 10.3} &                     69.3 $\pm$ {\footnotesize 8.6} &                    64.6 $\pm$ {\footnotesize 13.8} &                     63.8 $\pm$ {\footnotesize 8.4} \\
             hopper-medium &                     23.5 $\pm$ {\footnotesize 8.4} &                    66.7 $\pm$ {\footnotesize 20.8} &                    63.3 $\pm$ {\footnotesize 21.0} &                    60.1 $\pm$ {\footnotesize 27.3} &                    74.5 $\pm$ {\footnotesize 19.3} &                    89.6 $\pm$ {\footnotesize 27.9} \\
 halfcheetah-medium-replay &                     32.1 $\pm$ {\footnotesize 2.4} &                     31.5 $\pm$ {\footnotesize 3.3} &                     32.3 $\pm$ {\footnotesize 2.1} &                     33.1 $\pm$ {\footnotesize 2.3} &                     31.2 $\pm$ {\footnotesize 1.9} &                     31.5 $\pm$ {\footnotesize 3.2} \\
    walker2d-medium-replay &                      9.8 $\pm$ {\footnotesize 2.4} &                     10.7 $\pm$ {\footnotesize 2.0} &                     10.2 $\pm$ {\footnotesize 1.8} &                     10.2 $\pm$ {\footnotesize 2.4} &                     10.9 $\pm$ {\footnotesize 1.7} &                      9.4 $\pm$ {\footnotesize 1.4} \\
      hopper-medium-replay &                     28.7 $\pm$ {\footnotesize 3.8} &                     30.1 $\pm$ {\footnotesize 2.7} &                     30.5 $\pm$ {\footnotesize 2.9} &                     29.5 $\pm$ {\footnotesize 2.5} &                     31.3 $\pm$ {\footnotesize 1.9} &                     29.2 $\pm$ {\footnotesize 1.5} \\
 halfcheetah-medium-expert &                    79.9 $\pm$ {\footnotesize 10.1} &                    74.2 $\pm$ {\footnotesize 13.0} &                    76.8 $\pm$ {\footnotesize 13.4} &                    75.8 $\pm$ {\footnotesize 10.1} &                     71.3 $\pm$ {\footnotesize 8.6} &                     70.7 $\pm$ {\footnotesize 7.9} \\
    walker2d-medium-expert &                    67.4 $\pm$ {\footnotesize 16.0} &                    63.4 $\pm$ {\footnotesize 22.2} &                    69.7 $\pm$ {\footnotesize 17.3} &                    71.2 $\pm$ {\footnotesize 22.0} &                    63.4 $\pm$ {\footnotesize 23.1} &                    77.2 $\pm$ {\footnotesize 18.4} \\
      hopper-medium-expert &                     20.5 $\pm$ {\footnotesize 6.8} &                    56.7 $\pm$ {\footnotesize 27.5} &                    79.4 $\pm$ {\footnotesize 21.9} &                    99.9 $\pm$ {\footnotesize 29.0} &                    66.7 $\pm$ {\footnotesize 19.6} &                    62.0 $\pm$ {\footnotesize 19.5} \\
                 pen-human &                     -3.3 $\pm$ {\footnotesize 0.5} &                    64.2 $\pm$ {\footnotesize 17.0} &                    46.6 $\pm$ {\footnotesize 33.5} &                    45.5 $\pm$ {\footnotesize 24.5} &                    67.8 $\pm$ {\footnotesize 13.4} &                    60.3 $\pm$ {\footnotesize 11.4} \\
                pen-cloned &                      4.5 $\pm$ {\footnotesize 1.9} &                    19.6 $\pm$ {\footnotesize 11.8} &                    23.3 $\pm$ {\footnotesize 13.2} &                    18.0 $\pm$ {\footnotesize 14.4} &                    36.6 $\pm$ {\footnotesize 18.4} &                    40.0 $\pm$ {\footnotesize 20.8} \\
                pen-expert &                    74.2 $\pm$ {\footnotesize 26.6} &                   132.8 $\pm$ {\footnotesize 11.1} &                   132.8 $\pm$ {\footnotesize 17.9} &                   141.1 $\pm$ {\footnotesize 14.8} &                   136.6 $\pm$ {\footnotesize 10.8} &                   132.0 $\pm$ {\footnotesize 19.4} \\
               door-expert &                     29.1 $\pm$ {\footnotesize 9.7} &                    104.1 $\pm$ {\footnotesize 1.6} &                    104.2 $\pm$ {\footnotesize 1.7} &                    103.4 $\pm$ {\footnotesize 3.7} &                    102.9 $\pm$ {\footnotesize 3.9} &                    102.3 $\pm$ {\footnotesize 4.8} \\
               \midrule
             Average Score &                                               41.3 &                                               57.4 &                                               59.4 &                                               61.4 &                                               60.3 &                                               59.5 \\
\bottomrule
\end{tabular}
}
\end{table}

\section{Further Discussion on Capturing Multiple Modes in the Dataset} \label{sec:compare_implicit_gauss}

We clarify that our algorithms, \textit{e.g.}, ``GAN-Joint:Basic'', do not fail on the MuJoCo tasks, such as the medium-expert and medium datasets, though they underperform some of the baselines there. 
The tested MuJoCo datasets are collected by uni-modal Markovian policy (SAC), and hence uni-modal or deterministic policies can be sufficient for good results. 
Here, capturing multiple modes does not guarantee to give better scores.
However, on non-Markovian datasets \textit{e.g.}, Maze2D and Adroit, Table~\ref{table:main} and the following Figure~\ref{fig:mode} show that capturing multiple modes, capable by our methods, are critical for good results.
As a further corroboration, the baseline method, OptiDICE, also try to capture multiple action-modes in the offline dataset by training for behavior cloning a mixture of Gaussian policy with a per-dataset-tuned number of mixtures.
This may explains its relatively good scores on the Maze2D tasks.
However, the mixture of Gaussian behavior-cloning can fail on high-dimensional yet small-size datasets, which explains its relatively inferior results on the Adroit datasets.

In \Secref{sec:abaltion_study} we note that a uni-modal stochastic policy, such as the Gaussian policy, is less flexible to capture all the rewarding actions, on which an implicit policy may fit well.
Below we visualize such a difference.

Figure~\ref{fig:toy_gaussian_policy} compares the fitting of the eight-Gaussian toy dataset by implicit policy and Gaussian policy.
Specifically, Figure~\ref{fig:toy_gaussian_truth} plots the dataset; 
Figure~\ref{fig:toy_gaussian_cgan} plots CGAN with the default implicit generator (implicit policy) fitted by the classical policy-matching approach;
Figure~\ref{fig:toy_gaussian_gcgan} plots CGAN with Gaussian generator (Gaussian policy) fitted by the classical policy-matching approach;
Figure~\ref{fig:toy_gaussian_gan} plots CGAN with implicit policy fitted by the basic state-action joint-matching strategy (\Secref{sec:basic});
Figure~\ref{fig:toy_gaussian_ggan} plots CGAN with Gaussian policy fitted by the basic state-action joint-matching strategy.
Experimental details are on Appendix~\ref{sec:toy_detail}.

\begin{figure}[H]
     \centering
     \begin{subfigure}[b]{0.19\textwidth}
         \centering
         \includegraphics[width=\textwidth]{./toy_seed0/8gaussians_true.pdf}
         \caption{Truth}
         \label{fig:toy_gaussian_truth}
     \end{subfigure}
     \hfill
     \begin{subfigure}[b]{0.19\textwidth}
         \centering
         \includegraphics[width=\textwidth]{./toy_seed0/new_new_cgan_8gaussians_fake_1600.pdf}
         \caption{CGAN}
         \label{fig:toy_gaussian_cgan}
     \end{subfigure}
     \hfill
     \begin{subfigure}[b]{0.19\textwidth}
         \centering
         \includegraphics[width=\textwidth]{./toy_seed0/gauss_cgan_8gaussians_epoch_2000.pdf}
         \caption{G-CGAN}
         \label{fig:toy_gaussian_gcgan}
     \end{subfigure}
     \hfill
     \begin{subfigure}[b]{0.19\textwidth}
         \centering
         \includegraphics[width=\textwidth]{./toy_seed0/gan_8gaussians_epoch_2000.pdf}
         \caption{GAN}
         \label{fig:toy_gaussian_gan}
     \end{subfigure}
     \hfill
     \begin{subfigure}[b]{0.19\textwidth}
         \centering
         \includegraphics[width=\textwidth]{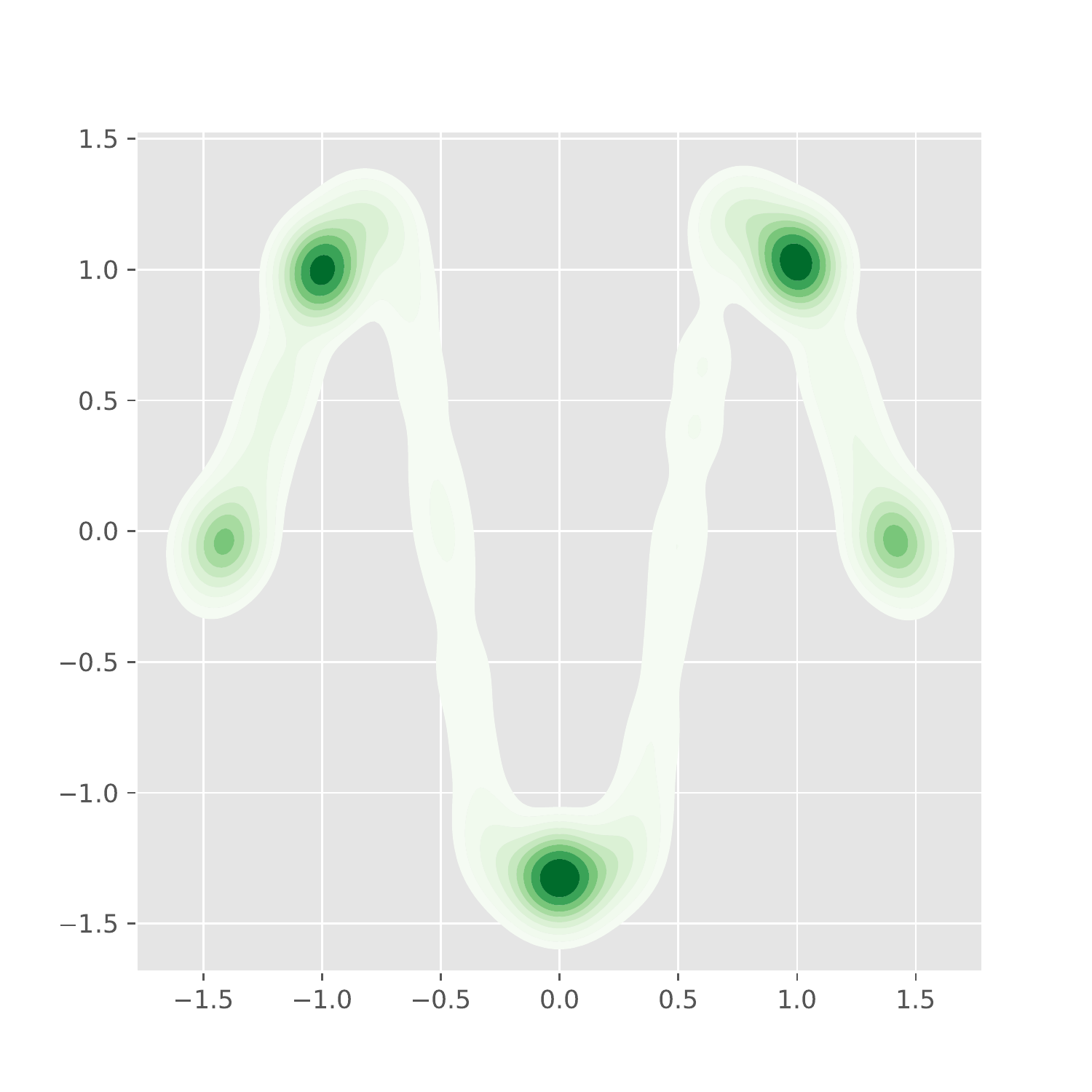}
         \caption{G-GAN}
         \label{fig:toy_gaussian_ggan}
     \end{subfigure}
     \captionsetup{font=small}
        \caption{
        \small
        Performance of approximating the behavior policy on the eight-Gaussian dataset by conditional GAN with default (implicit) generator and Gaussian generator. 
        A conditional GAN (``CGAN'') and a Gaussian-generator conditional GAN (``G-CGAN'') are fitted using the policy-matching approach.
        A conditional GAN (``GAN'') and a Gaussian-generator conditional GAN (``G-GAN'') are fitted using the basic state-action joint-matching strategy (\Secref{sec:policy_matching_reg}).
        Performance is judged by \textbf{(1)} clear concentration on the eight centers, and \textbf{(2)} smooth interpolation between centers, which implies a good and smooth fit to the behavior policy.
        }
        \label{fig:toy_gaussian_policy}
\end{figure}

We see that whatever training strategies, Gaussian policies fail to learn multi-modal state-conditional action distributions, even if needed.
Even though the Gaussian policy version of CGAN may still correctly capture some modes in the action distributions, an improvement over the mode-covering CVAE, they miss other modes.
Besides, these Gaussian policy versions interpolate less-smoothly between the centers.
In offline RL, these weaknesses is related the missing of some rewarding actions and less-predictable action-choices at unseen states.

To visualize the differences between the implicit and the Gaussian policy in the offline RL setting, we plot the kernel density estimates of the action-distribution in the ``maze2d-umaze'' dataset, where a performance difference is shown in Table~\ref{table:implicit_normal}.
Specifically, Figure~\ref{fig:mode_truth} plots the action-distribution in the offline dataset.
Figure~\ref{fig:mode_gauss} and \ref{fig:mode_implicit} respectively plot action-distributions produced by the final Gaussian policy and the final implicit policy generating Table~\ref{table:implicit_normal}.

\begin{figure}[H]
     \centering
     \begin{subfigure}[b]{0.3\textwidth}
         \centering
         \includegraphics[width=\textwidth]{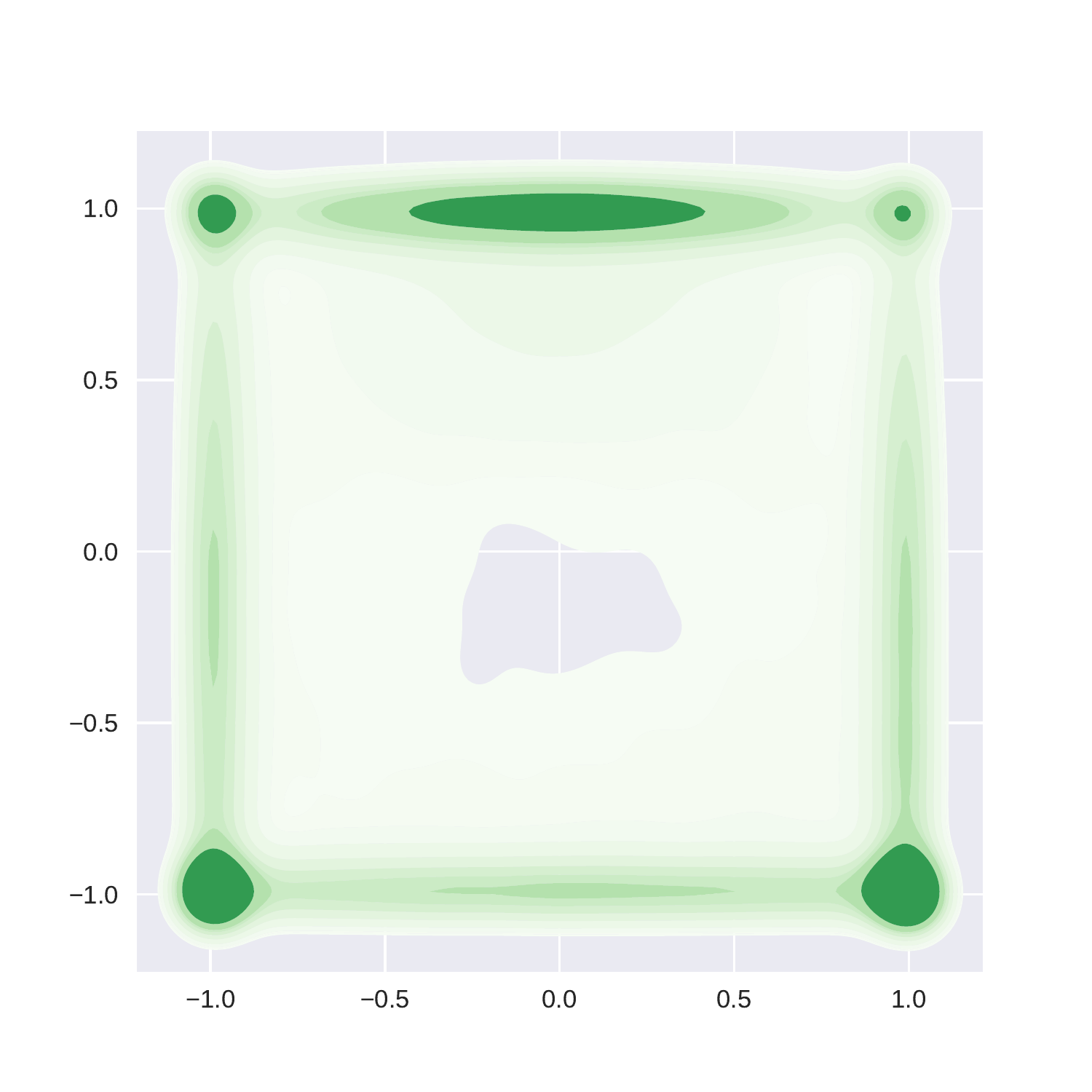}
         \caption{Truth}
         \label{fig:mode_truth}
     \end{subfigure}
     \hfill
     \begin{subfigure}[b]{0.3\textwidth}
         \centering
         \includegraphics[width=\textwidth]{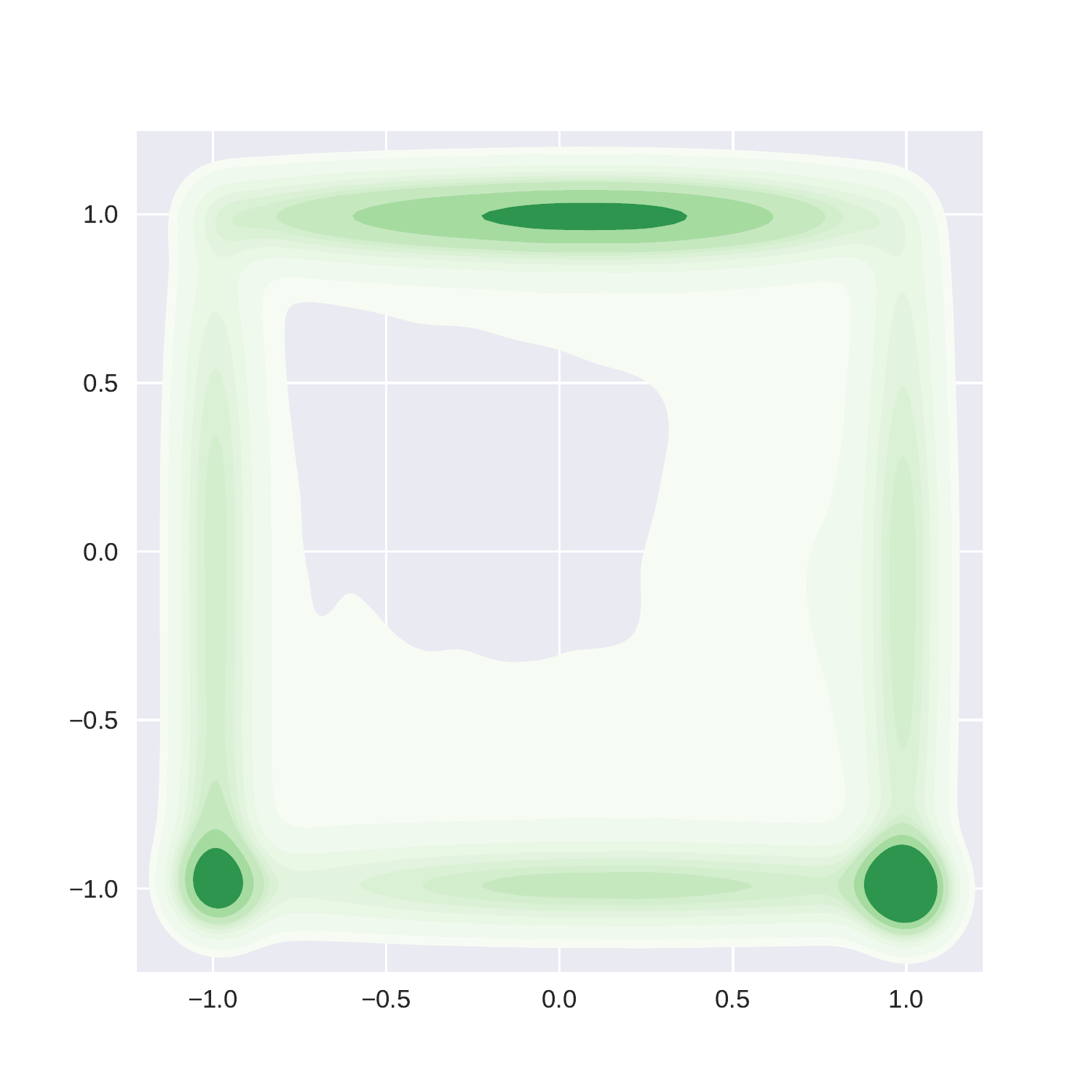}
         \caption{Gaussian Policy}
         \label{fig:mode_gauss}
     \end{subfigure}
     \hfill
     \begin{subfigure}[b]{0.3\textwidth}
         \centering
         \includegraphics[width=\textwidth]{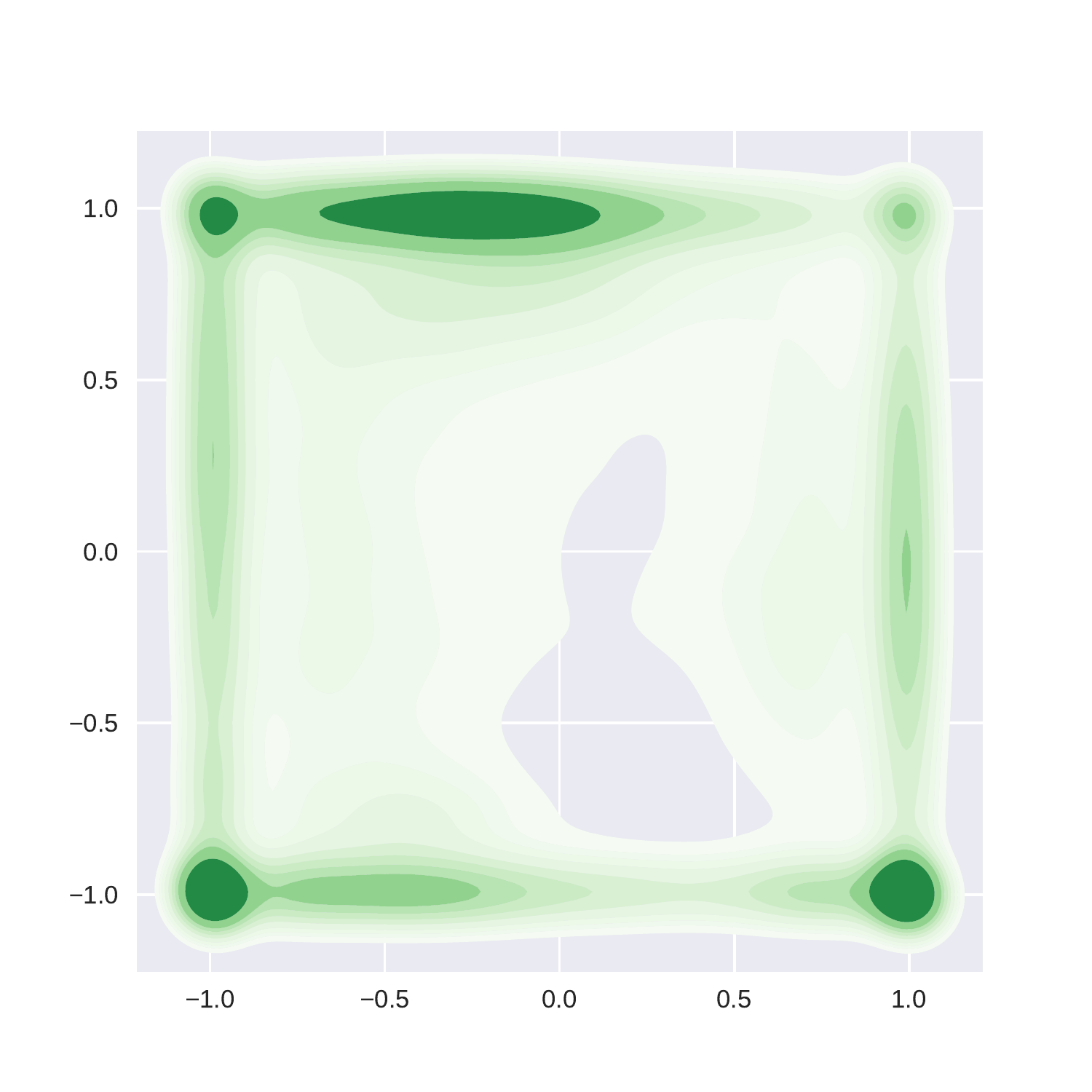}
         \caption{Implicit Policy}
         \label{fig:mode_implicit}
     \end{subfigure}
     \captionsetup{font=small}
        \caption{
        \small
        KDE plots of the action-distributions in the ``maze2d-umaze'' dataset.
        \textbf{(a)} Action distribution in the offline dataset; \textbf{(b)} Action distribution produced by the final Gaussian policy in Table~\ref{table:implicit_normal}; \textbf{(c)} Action distribution by the final implicit policy (\Secref{sec:basic}) in Table~\ref{table:implicit_normal}.
        }
        \label{fig:mode}
\end{figure}

We see from Figure~\ref{fig:mode_truth} and Figure~\ref{fig:mode_gauss} that Gaussian policy leaves out two action modes, namely, modes on the upper-left and upper-right corners.
Figure~\ref{fig:mode_implicit} shows that our implicit policy does capture all modes shown in Figure~\ref{fig:mode_truth}.
Note that ``maze2d-umaze'' is a navigation task requiring agents to reach a goal location \citep{fu2021d4rl}.
Gaussian policy thus may miss out some directions in the offline dataset pertaining to short paths to the goal state, which may explain its inferior performance on this dataset in Table~\ref{table:implicit_normal}.

\section{Full Algorithm}\label{sec:full_algo}

\begin{algorithm}[H]
\captionsetup{font=small}
\caption{\small GAN-Joint, Detailed Version}
\begin{algorithmic}
\label{alg:main}
\STATE {\bfseries Input:} Learning rate $\eta_\vtheta, \eta_\vphi, \eta_\vw$; 
target smoothing factor $\beta$; 
noise distribution $p_\vz(\vz)$; 
policy network $\pi_{\vphi}$ with parameter $\vphi$; 
critic network $Q_{\vtheta_1}$ and $Q_{\vtheta_2}$ with parameters $\vtheta_1$,$\vtheta_2$; 
discriminator network $D_\vw$ with parameter $\vw$; 
generator loss function $\gL_g$; 
standard deviation $\sigma_B = \sigma_J \triangleq \sigma$; 
number of smoothed states $N_B$;
number of epochs for warm start $N_{\mathrm{warm}}$;
policy frequency $k$. \\
\STATE \textbf{Initialization:} Initialize $\vphi$, $\vtheta_1$,$\vtheta_2$, $\vw$. Initialize $\vphi' \leftarrow \vphi$, $\vtheta_1' \leftarrow \vtheta_1$, $\vtheta_2' \leftarrow \vtheta_2$. 
Load dataset $\mathcal{D}$.
\FOR{each epoch}
\FOR{each iteration within current epoch}
\STATE Sample a mini-batch of transitions $\gB = \cbr{(\vs, \va, r, \vs')} \sim \sD$. 
\STATE \COMMENT{// Policy Evaluation}
\STATE For each $\vs' \in \gB$ sample $N_B$ $\hat \vs$ with noise standard deviation $\sigma_B$ (= $\sigma$) for state-smoothing via,
\begin{equation*}\textstyle
    \hat\vs = \vs' + \vepsilon, \vepsilon \sim \gN(\vzero, \sigma_B^2 \mI) .
\end{equation*}
\STATE Sample one corresponding actions $\hat \va \sim \pi_{\vphi'}(\boldsymbol{\cdot}\given \hat\vs)$ for each $\hat\vs$.
\STATE Calculate $\widetilde{Q}(\vs, \va)$ as 
\begin{equation*} \textstyle
     \widetilde{Q}\br{\vs, \va} \triangleq r(\vs, \va) + \gamma \frac{1}{N_B} \sum_{\br{\hat\vs,\hat\va}} \sbr{\lambda \min_{j=1,2}Q_{\vtheta'_j}\br{\hat\vs, \hat\va} + (1-\lambda) \max_{j=1,2}Q_{\vtheta'_j}\br{\hat\vs, \hat\va}}.
\end{equation*}
\STATE Minimize the critic loss with respect to $\vtheta_j, j = 1,2$, over $\br{\vs, \va} \in \gB$, with learning rate $\eta_\vtheta$,
\begin{equation*}\textstyle
    \arg\min_{\vtheta_j}\frac{1}{\abs{\gB}}\sum_{\br{\vs, \va} \in \gB} \br{ Q_{\vtheta_j}\br{\vs, \va} - \widetilde{Q}(\vs, \va) }^2 .
\end{equation*}
\STATE \COMMENT{// Policy Improvement with State-action Joint-matching}
\STATE Resample $\abs{\gB}$ new states $\Tilde{\vs} \sim \sD$ independent of $\vs$. Add state-smoothing to $\Tilde{\vs}$ with noise standard deviation $\sigma_J$ (= $\sigma$) using $\vepsilon \sim \gN\br{\vzero, \sigma_J^2\mI}, \, \Tilde{\vs} \leftarrow \Tilde{\vs} + \vepsilon$. 
\STATE Form the generator sample $\vx$ and data sample $\vy$ using 
$
    \vx \triangleq (\Tilde{\vs}, \Tilde{\va}), \; \Tilde{\va} \sim \pi_\vphi(\boldsymbol{\cdot}\given \Tilde{\vs}); \; \vy \triangleq (\vs, \va) \in \gB.
$
\STATE Calculate the generator loss $\gL_g(\vphi) = \frac{1}{\abs{\gB}} \sum_\vx\sbr{\log\br{1-D_\vw\br{\vx}}}$ using $\vx$ and discriminator $D_\vw$.
\IF{iteration count \% $k$ == 0}
\IF{epoch count $< N_{\mathrm{warm}}$}
\STATE Optimize policy with respect to $\gL_g(\vphi)$ only, with learning rate $\eta_\vphi$, \textit{i.e.},
\begin{equation*} \textstyle
    \arg\min_\vphi \alpha \cdot \gL_g(\vphi) .
\end{equation*}
\ELSE
\STATE Optimize policy with learning rate $\eta_\vphi$ for the target
\begin{equation*} \textstyle
    \arg\min_{\vphi} -\frac{1}{\abs{\gB}} \sum_{\vs \in \gB, \va \sim \pi_\vphi(\boldsymbol{\cdot}\given \vs)}\sbr{\min_{j=1,2} Q_{\vtheta_j}(\vs, \va)} + \alpha \cdot \gL_g(\vphi) .
\end{equation*}
\ENDIF
\ELSE
\STATE Skip the policy improvement step.
\ENDIF
\STATE \COMMENT{// Training the Discriminator}
\STATE Optimize discriminator $D_\vw$ to maximize $\frac{1}{\abs{\gB}} \sum_{\vy}\sbr{\log D_\vw(\vy)} + \frac{1}{\abs{\gB}} \sum_\vx \sbr{\log\br{1-D_\vw(\vx)}}$ with respect to $\vw$ with learning rate $\eta_\vw$.
\STATE \COMMENT{// Soft Update the Target Networks}
\STATE $\vphi' \leftarrow \beta \vphi + (1-\beta) \vphi'$; \; $\vtheta'_j \leftarrow \beta \vtheta_j + (1-\beta)\vtheta'_j$ for $j=1,2$. 
\ENDFOR
\ENDFOR
\end{algorithmic}
\end{algorithm}

\section{Proofs and Additional Theoretical Analysis}\label{sec:proof}
We follow the offline RL literature \citep{breakingcurse2018, algaedice2019, confoundingrobust2020, blackope2020, gendice2020} to assume the following regularity condition on the MDP structure, which ensures the ergodicity of the corresponding Markov chains and that the limiting state occupancy measures exist and equal to the stationary distributions of the chains.

\begin{assumption}[Ergodicity of MDP]
	 The MDP $\cM$ is ergodic, \textit{i.e.}, the Markov chains associated with any $\pi_b$ and any $\pi_{\vphi}$ under consideration are positive Harris recurrent \citep{harris2011}. 
\end{assumption}

Below is a useful lemma for the proof of Theorem~\ref{thm:occup_single} and Theorem~\ref{thm:occup_mix}.

\begin{lemma} \label{lemma:perturb}
	Let $\mA \in \R^{N\times N}$ be nonsingular and let $\vzero \ne \vb \in \R^N, \vx = \mA^{-1} \vb\in \R^N$. 
	Let $\Delta \mA\in \R^{N\times N}$ be an arbitrary perturbation on $\mA$.
	Assume that the norm on $\R^{N\times N}$ satisfies $\| \mA\vx \| \leq \|\mA\| \|\vx\|$  for all $\mA$ and $\vx$. If
	\begin{equation*}
		\left(\mA + \Delta \mA \right) (\vx + \Delta \vx) = \vb \text{ and } {\| \Delta \mA\| \over \|\mA\|} < {1\over \kappa(\mA)}
	\end{equation*}
	then
	\begin{equation*}
		{\| \Delta \vx\| \over \|\vx\|}  \leq {\kappa(\mA) \, {\| \Delta \mA\| \over \|\mA\|} \over 1- \kappa(\mA) \, {\| \Delta \mA\| \over \|\mA\|}} = {\kappa(\mA) \over {\|\mA\| \over \| \Delta \mA \|} - \kappa(\mA)}
	\end{equation*}
\end{lemma}

\begin{proof}[Proof of Lemma \ref{lemma:perturb}]
	Let $\widehat{\vx} = \vx + \Delta \vx, \mA\,(\vx + \Delta \vx) + \Delta \mA \,\widehat{\vx}  = \vb \implies \mA\, \Delta \vx+ \Delta \mA \,\widehat{\vx}  = \vzero \implies \Delta \vx = -\mA^{-1}\,\Delta \mA \,\widehat{\vx}$. Then,
	\begin{equation*}
		\begin{split}
			& \| \Delta \vx\| \leq \| \mA^{-1} \| \| \Delta \mA\| \left(\|\vx\| + \|\Delta \vx\|\right) = \kappa(\mA) \,{\| \Delta \mA\| \over \| \mA \|}\,\left(\|\vx\| + \|\Delta \vx\|\right) \\
			\implies & \left(1 - \kappa(\mA)\, {\| \Delta \mA\| \over \| \mA \|} \right) \,\| \Delta \vx\| \leq \kappa(\mA) \,{\| \Delta \mA\| \over \| \mA \|}\, \| \vx\| \\
			\implies & {\| \Delta \vx\| \over \|\vx\|}  \leq {\kappa(\mA) \, {\| \Delta \mA\| \over \|\mA\|} \over 1- \kappa(\mA) \, {\| \Delta \mA\| \over \|\mA\|}} = {\kappa(\mA) \over {\|\mA\| \over \| \Delta \mA \|} - \kappa(\mA)}
		\end{split}
	\end{equation*}
	since $\kappa(\mA) \,\| \Delta \mA\| \,/ \,\| \mA\| < 1$ by assumption.
\end{proof}

For the proof of Theorem~\ref{thm:occup_single} and Theorem~\ref{thm:occup_mix}, we assume the following notations.

{\bfseries Notation.} 
Denote $\mA_{-i*}$ as matrix $\mA$ with its $i$-th row removed; $\kappa(\mA)$ as the 2-norm condition number of $\mA$; $\1$ as a row vector of all ones and $\mI$ as an identity matrix, both with an appropriate dimension.
Assume that the state space $\sS$ is finite with cardinality $N$, \textit{i.e.}, $\sS = \left(\vs^1, \ldots, \vs^N\right)$.
The transition probabilities associated with policy $\pi_{\vphi}$ over $\sS$ is then an $N \times N$ matrix $\mT_{\vphi}$, whose $(i, j)$ entry is $\mT_{\vphi, (i,j)} = p_{\pi_\vphi}\left(S_{t+1}=\vs^j \given S_t = \vs^i\right) = \int_{\sA}\mathcal{P}\left(S_{t+1}=\vs^j \given S_t=\vs^i, A_t=\va_t\right) \pi_{\vphi}\left(\va_t \given \vs^i\right) \,d\va_t$, and similarly for $\mT_b$, the transition matrix associated with $\pi_b$.
Note that in this case, $d_\vphi(\vs), d_b(\vs)$ are vectors and we denote $\vd_{\vphi} \triangleq \vd_{\vphi}(\vs), \vd_b \triangleq \vd_b(\vs) \in \R^N$ and $\vd_{\vphi} = \vd_b + \Delta \vd$.

For the proof of Theorem~\ref{thm:occup_single}, notice that $d_{\vphi}(\vs,\va)=d_{\vphi}(\vs)\pi_{\vphi}(\va\given\vs)$ and similarly for $d_{b}(\vs,\va)$.
Hence, it is sufficient to show the closeness between $d_{\vphi}(\vs)$ and $d_{b}(\vs)$ when $\pi_{\vphi}(\va\given\vs)$ is close to $\pi_b(\va\given\vs)$.
Below we give our analysis for the matrix (finite state space) case. Continuous state-space cases may be analyzed similarly and are left for future work. 

\begin{theorem}[Formal Statement of Theorem~\ref{thm:occup_single_informal}] \label{thm:occup_single}
Denote 
\begin{equation*}\textstyle
	\kappa_{max} = \max_{i=2, \ldots, N+1}\kappa\left( \begin{pmatrix}
	\1 \\
	\mI - \mT^{\top}_b
\end{pmatrix}_{-i*} \right).
\end{equation*}
If
\begin{equation*}\textstyle
	\max_{i=1,\ldots, N} \left\| \pi_b\br{\boldsymbol{\cdot}\given \vs^i} - \pi_{\vphi}\br{\boldsymbol{\cdot}\given \vs^i} \right\|_1 \leq \epsilon < {1 \over \kappa_{max}}
\end{equation*}
and $\mT^{\vphi}_{i,j}, \mT^b_{i,j} > 0, \forall \, i, j \in \cbr{1,2,\ldots,N}$,
then 
\begin{equation*}\textstyle
	2 \mathrm{TV}(\vd_\vphi, \vd_b) = \| \Delta \vd \|_1 = \| \vd_\vphi - \vd_b \|_1 \leq {\epsilon~\kappa_{max} \over 1 - \epsilon~\kappa_{max}} \rightarrow 0\quad \text{as } \epsilon \rightarrow 0.
\end{equation*}
\end{theorem}

\begin{remark}
\textbf{(1)} We note that $\kappa_{max}$ is a constant for fixed $\mT_b$ and can be calculated by iteratively removing columns of $\mT_b$ and computing the SVD of the referred matrix.
\textbf{(2)} The assumption that $\mT^{\vphi}_{i,j}, \mT^b_{i,j} > 0, \forall \, i, j \in \cbr{1,\ldots,N}$ can be satisfied by substituting the zero entries in the original transition matrix with a small number and re-normalized each row of the resulting matrix, as in the PageRank algorithm \citep{pagerank1998,deeppagerank2004}.
\end{remark}

\begin{proof}[Proof of Theorem~\ref{thm:occup_single}]
	By ergodicity,  $\vd_{\vphi}(\vs), \vd_b(\vs) \in \R^N$ uniquely exist. 
	For $\vd \in \left\{\vd_{\vphi}, \vd_b\right\}, \mT \in \left\{\mT^{\vphi}, \mT^b \right\}$, stationarity implies that $\vd$ is an eigenvector of $\mT^\top$ associated with eigenvalue $1$, and furthermore
	\begin{equation} \label{eq:eigen}
		\vd^\top = \vd^\top \mT \implies \left(\mI -\mT^\top \right) \vd = \vzero \implies \underbrace{\begin{pmatrix}
			\1 \\
			\mI -\mT^\top
		\end{pmatrix}}_{\triangleq\, \widehat{\mT} \in \R^{(N+1) \times N}} \vd = \begin{pmatrix}
			1 \\
			0\\
			\vdots \\
			0
		\end{pmatrix}
	\end{equation}
	Since $\mT$ is a positive matrix, by the Perron-Frobenius theorem \citep{matrixanalysis2000}, $1$ is an eigenvalue of $\mT^\top$ with algebraic multiplicity, and hence geometry multiplicity, $1$. 
	The eigen-equation $\mT^\top \vv = 1 \vv$  has unique solution $\vv$ up to a constant multiplier.
	Since $\mT^\top \vd = \vd$, the eigenspace of $\mT^\top$ associated with eigenvalue $1$ is $\mathrm{span}\left(\{\vd\}\right)$.
	Hence $\mathrm{dim}\,\mathrm{ker}\left(\mI - \mT^\top\right) = 1 \implies \mathrm{rank}\left(\mI - \mT^\top\right)= N-1 \implies \mathrm{rank} \left(\widehat{\mT}\right) = N$.
	The reason is that if $\exists\, \vv \ne \vzero \, \,s.t.\, \widehat{\mT}\, \vv = \vzero$, then 
	\begin{equation*}
		\widehat{\mT}\, \vv = \ \begin{pmatrix}
			\1 \, \vv \\
			\left(\mI -\mT^\top\right) \vv
		\end{pmatrix} = \begin{pmatrix}
			0 \\
			\vzero
		\end{pmatrix} \implies \vv = c\, \vd \text{ for scalar } c \in \R \text{ and } \1  \, \vv = c \, \1 \, \vd = 0 \implies c = 0, 
	\end{equation*}
	and hence $\vv = c \, \vd = \vzero$ which contradicts to $\vv \ne \vzero$.
	
	Since $\mathrm{rank}\left(\widehat{\mT}\right) = N$, $\mathrm{dim}\left(\left\{\vv \in \R^{N+1}: \vv^\top \, \widehat{\mT} = \vzero \right\}\right) = 1$. 
	For such a $\vv \ne \vzero, \exists\, i \in \{2, \ldots, N+1\} \,s.t. \, v_i \ne 0$. 
	WLOG, assume $v_{N+1} \ne 0$. 
	Let $\mA \in \R^{N\times N}$ be the first $N$ rows of $\widehat{\mT}$, then $\mathrm{rank}(\mA) = N$.
	The reason is that if $\mathrm{rank}(\mA) < N$, $\exists \, \vw \ne \vzero \, s.t.\, \vw^\top \mA = \vzero$, then
	\begin{equation*}
		\begin{pmatrix}
		\vw \\
		0
	\end{pmatrix}^\top \, \widehat{\mT} = \vw^\top \mA = \vzero
	\end{equation*}
	and $v_{N+1} \ne 0 \implies \begin{pmatrix}
		\vw^\top, 0
	\end{pmatrix}^\top$ is not a constant multiple of $\vv \implies \mathrm{dim}\,\mathrm{ker}\left(\widehat{\mT}^\top \right) \geq 2$, which contradicts to the fact that $\mathrm{rank} \left(\widehat{\mT}\right) = N$. 
	Thus, we conclude that $\mathrm{rank} \left(\mA\right) = N \implies \mA$ is invertible.
	
	Let $\ve^{(1)} = \begin{pmatrix}
		1, 0, \ldots, 0
	\end{pmatrix}^\top \in \R^N$, then Eq.~\ref{eq:eigen} implies $\mA \, \vd = \ve^{(1)}$.
	Plug in $\vd_{\vphi}, \vd_b, \mT_{\vphi}, \mT_b$ and define $\mA_\vphi, \mA_b$ similarly, we have $\mA_\vphi \, \vd_\vphi = \ve^{(1)}, \mA_b \, \vd_b = \ve^{(1)}$.
	For $\mT_b$ and $\mT_\vphi$, we notice that by Jensen's inequality,
	\begin{equation*}
		\begin{split}
			\sum_{j=1}^N \left|\left(\mT_b - \mT_\vphi \right)_{ij}\right| &\leq \sum_{j=1}^N \int_{\sA}\cP\left(S_{t+1}=\vs^j \given S_t=\vs^i, A_t=\va_t\right) \left|\pi_b\left(\va_t \given \vs^i\right) - \pi_\vphi\left(\va_t \given \vs^i\right)\right| \,d\va_t \\
		&= \int_{\sA}\left(\sum_{j=1}^N \cP\left(S_{t+1}=\vs^j \given S_t=\vs^i, A_t=\va_t\right)\right) \left|\pi_b\left(\va_t \given \vs^i\right) - \pi_\vphi\left(\va_t \given \vs^i\right)\right| \,d\va_t \\
		&= \int_{\sA} \left|\pi_b\left(\va_t \given \vs^i\right) - \pi_\vphi\left(\va_t \given \vs^i\right)\right| \,d\va_t \\
		&= \left\|\pi_b\left(\boldsymbol{\cdot}\given \vs^i\right) - \pi_\vphi\left(\boldsymbol{\cdot}\given \vs^i\right) \right\|_1
		\end{split}
	\end{equation*}
	Therefore, by assumption,
	\begin{equation*}
		\left\| \mT_b - \mT_\vphi \right\|_{\infty} = \max_{i=1, \ldots, N}\sum_{j=1}^N \left|\left(\mT_b - \mT_\vphi \right)_{ij}\right| \leq  \max_{i=1, \ldots, N} \left\|\pi_b\left(\boldsymbol{\cdot}\given \vs^i\right) - \pi_\vphi\left(\boldsymbol{\cdot}\given \vs^i\right) \right\|_1 \leq \epsilon.
	\end{equation*}
	Let $\mA_\vphi = \mA_b + \Delta \mA$. 
	For $\| \Delta \mA\|_1$, we notice that
	\begin{equation*}
		\begin{split}
			\| \Delta \mA\|_1 = \| \mA_\vphi - \mA_b \|_1 &\leq \left\|\begin{pmatrix}
			\1 \\
			\mI -\mT^\top_\vphi
		\end{pmatrix} - \begin{pmatrix}
			\1 \\
			\mI -\mT^\top_b
		\end{pmatrix} \right\|_1 \\
		&= \left\|\begin{pmatrix}
			\vzero \\
			\mT^\top_b -\mT^\top_\vphi
		\end{pmatrix} \right\|_1 = \left\|\left(\mT_b - \mT_\vphi \right)^\top  \right\|_1
		= \left\|\mT_b - \mT_\vphi \right\|_{\infty} \leq \epsilon.
		\end{split}
	\end{equation*}
	Notice that matrix 1-norm satisfies $\|\mM \vv\|_1 \leq \|\mM\|_1 \, \|\vv\|_1$ for all matrix $\mM$ and vector $\vv$, that $\left\|\mA_b\right\|_1 \geq 1$ and that $\| \vd_b\|_1 =1$.
	Lemma \ref{lemma:perturb} implies that
	\begin{equation}\label{eq:stationary_bound}
		\begin{split}
			{\| \Delta \vd\|_1 \over \| \vd_b \|_1} = \| \Delta \vd\|_1 &= \| \vd_\vphi - \vd_b \|_1 \\
			&\leq {\kappa (\mA_b) \over {1\over \| \Delta \mA \|_1} - \kappa(\mA_b)} \\
			&\leq  {\kappa (\mA_b) \over {1\over \epsilon} - \kappa(\mA_b)} = {\epsilon \, \kappa (\mA_b) \over 1- \epsilon \, \kappa (\mA_b)} \\
			&\leq {\epsilon \, \kappa_{max} \over 1 - \epsilon \, \kappa_{max}} \rightarrow 0\quad \text{as } \epsilon \rightarrow 0.
		\end{split}
	\end{equation}
\end{proof}
For the statement and the proof of Theorem~\ref{thm:occup_mix}, assume that there are $K$ such data-collecting policies $\cbr{\pi_{b_k}\br{\boldsymbol{\cdot}\given\vs}}_{k=1}^K$ with corresponding mixture probabilities $\cbr{w_k}_{k=1}^K$, \textit{i.e.}, $\pi_b(\boldsymbol{\cdot}\given\vs) = \sum_{k=1}^K w_k \pi_{b_k}\br{\boldsymbol{\cdot}\given \vs}, \sum_{k=1}^K w_k=1$.
Since we collect $\sD$ by running each $\pi_{b_k}\br{\boldsymbol{\cdot}\given\vs}$ a proportion of $w_k$ of total time, we may decompose $\sD$ as $\sD = \bigcup_{k=1}^K \sD_k$, where $\sD_k$ consists of $w_k$ proportion of data in $\sD$.
Thus, $\vd_\sD(\vs) = \sum_{k=1}^K w_k \vd_{\sD_k}(\vs)$ and the targeted approximation $\vd_\vphi(\vs) \approx \vd_\sD(\vs)$ has population version $\vd_\vphi(\vs) \approx \sum_{k=1}^K w_k \vd_{b_k}(\vs) \triangleq \vd_b(\vs)$.
As before, denote $\vd_{b_k}(\vs) \in \R^N$ as the limiting state-occupancy measure induced by $\pi_{b_k}$ on $\cM$; $\mT_{b_k}$ as the transition matrix induced by $\pi_{b_k}$ over $\sS$; and $\vd_\vphi = \vd_{b_k} + \Delta \vd_k$.
\begin{theorem} \label{thm:occup_mix}
	Denote 
\begin{equation*}\textstyle
\begin{split}
    \kappa_{max, k} =  \max_{i=2, \ldots, N+1}\kappa\left( \begin{pmatrix}
	\mathbf{1} \\
	\mI - \mT^{\top}_{b_k}
\end{pmatrix}_{-i*} \right), \quad 
\kappa_{max} = \max_{k=1\ldots, K} \kappa_{max, k}.
\end{split}
\end{equation*}
If
\begin{equation*}\textstyle
	\max_{k=1, \ldots, K}\max_{i=1,\ldots, N} \left\| \pi_{b_k}\br{\boldsymbol{\cdot}\given \vs^i} - \pi_{\vphi}\br{\boldsymbol{\cdot}\given \vs^i} \right\|_1 \leq \epsilon < {1 \over \kappa_{max}}
\end{equation*}
and $\mT^{\vphi}_{i,j}, \mT^{b_k}_{i,j} > 0, \forall \, i, j \in \cbr{1,\ldots,N}, k \in \cbr{1,\ldots,K}$,
then 
\begin{equation*}\textstyle
	 \begin{split}
	     2 \mathrm{TV}(\vd_\vphi, \vd_b) = \| \vd_\vphi - \vd_b \|_1 \leq \sum_{k=1}^K w_k {\epsilon \, \kappa_{max,k} \over 1 - \epsilon \, \kappa_{max,k}} \leq {\epsilon \, \kappa_{max} \over 1-\epsilon \, \kappa_{max}} \rightarrow 0 \quad \text{as } \epsilon \rightarrow 0.
	 \end{split}
\end{equation*}
In particular, if $w_k = 1/K, \forall\, k \in \cbr{1,\ldots,K}$, then $\text{as } \epsilon \rightarrow 0$
\begin{equation*}\textstyle
	\mathbin{2 \mathrm{TV}(\vd_\vphi, \vd_b) = } \norm{\vd_\vphi - \vd_b}_1 \leq {1\over K} \sum_{k=1}^K {\epsilon \, \kappa_{max, k} \over 1-\epsilon \, \kappa_{max, k}} \rightarrow 0.
\end{equation*}
\end{theorem}

\begin{proof}[Proof of Theorem \ref{thm:occup_mix}]
	By ergodicity, $\vd_\vphi(\vs), \vd_{b_k}(\vs) \in \R^N$ uniquely exist, $\forall \, k$. 
	For $\vd\in \cbr{\vd_\vphi, \vd_{b_1}, \ldots, \vd_{b_k}}$ and $\mT \in \cbr{\mT^\vphi, \mT^{b_1}, \ldots, \mT^{b_K}}$, we follow the steps and notations in the proof of Theorem \ref{thm:occup_single} to conclude that $\mathrm{rank}\br{\mA} = N \implies \mA$ is invertible and that $\mA \vd = \ve^{(1)} \in \R^N$.
	Plugging in $\vd_\vphi, \vd_{b_k}, \mT_{\vphi}, \mT_{b_k}$ and defining $\mA_\vphi, \mA_{b_k}$ similarly, we have $\mA_\vphi \vd_\vphi = \ve^{(1)}, \mA_{b_k}\vd_{b_k} = \ve^{(1)}$.
	For the transition matrix $\mT_b$ induced by the mixture of policies $\pi_b$, we have
	\begin{equation*}
		\begin{split}
			\mT_{b,(i,j)} &= p_{\pi_b} \br{S_{t+1} = \vs^j \given S_t = \vs^i} = \int_{\sA} \cP \br{S_{t+1}=\vs^j\given S_t = \vs^i, A_t = \va_t} \pi_b\br{\va_t \given \vs^i} \, d\va_t \\
			&= \sum_{k=1}^K w_k \int_{\sA} \cP\br{S_{t+1}=\vs^j\given S_t = \vs^i, A_t = \va_t} \pi_{b_k}\br{\va_t \given \vs^i} \, d\va_t = \sum_{k=1}^K w_k \mT_{b_k, (i,j)}
		\end{split}
	\end{equation*}
	and therefore $\mT_b = \sum_{k=1}^K w_k \mT_{b_k}$.
	
	For $\mT_{b_k}$ and $\mT_\vphi$, as in the proof of Theorem \ref{thm:occup_single}, by Jensen's inequality,
	\begin{equation*}
		\sum_{j=1}^N\abs{\br{\mT_{b_k} - \mT_\vphi}_{ij}} \leq \norm{\pi_{b_k}\br{\boldsymbol{\cdot}\given\vs^i} - \pi_\vphi\br{\boldsymbol{\cdot}\given\vs^i}}_1
	\end{equation*}
	Therefore, by assumption, 
	\begin{equation*}
		\norm{\mT_{b_k} - \mT_{\vphi}}_\infty = \max_{i=1, \ldots, N}\sum_{j=1}^N \left|\left(\mT_{b_k} - \mT_\vphi \right)_{ij}\right| \leq  \max_{i=1, \ldots, N} \left\|\pi_{b_k}\left(\boldsymbol{\cdot}\given \vs^i\right) - \pi_\vphi\left(\boldsymbol{\cdot}\given \vs^i\right) \right\|_1 \leq \epsilon.
	\end{equation*}
	Let $\mA_\vphi = \mA_{b_k}  + \Delta \mA_{b_k}$.	
	For $\norm{\Delta \mA_{b_k}}_1$, we have
	\begin{equation*}
		\norm{\Delta \mA_{b_k}}_1 = \norm{\mA_\vphi - \mA_{b_k}}_1 \leq \norm{\mT^\top_{b_k} - \mT_\vphi^\top}_1 = \norm{\mT_{b_k} - \mT_\vphi}_\infty \leq \epsilon.
	\end{equation*} 
	Note that $\forall\, k, \norm{\mA_{b_k}}_1 \geq 1 + 1-\sum_{i=1}^N \mT^{b_k}_{1i}=1, \norm{\vd_{b_k}}_1 = 1$.
	Lemma \ref{lemma:perturb} implies that
	\begin{equation*}
		\begin{split}
			{\norm{\Delta \vd_k}_1 \over  \norm{\vd_{b_k}}_1} &= \norm{\Delta \vd_k}_1 = \norm{\vd_\vphi - \vd_{b_k}}_1 \leq {\epsilon\, \kappa\br{\mA_{b_k}}\over 1 - \epsilon\, \kappa\br{\mA_{b_k}}} \leq {\epsilon \, \kappa_{max. k} \over 1- \epsilon \, \kappa_{max, k}}
		\end{split}
	\end{equation*}
	Thus for the relative distance between $\vd_\vphi$ and $\vd_b$, we have
	\begin{equation*}
		\begin{split}
		\norm{\vd_b}_1 &= \sum_{i=1}^N \vd_b\br{\vs^i} = \sum_{i=1}^N\sum_{k=1}^K w_k \vd_{b_k}\br{\vs^i} = \sum_{k=1}^K w_k \sum_{i=1}^N \vd_{b_k}\br{\vs^i} = \sum_{k=1}^Kw_k = 1, \\
			{\norm{\vd_\vphi - \vd_b}_1 \over \norm{\vd_b}_1} &= \norm{\vd_\vphi - \vd_b}_1 = \norm{\sum_{k=1}^K\br{w_k \vd_\vphi - w_k \vd_{b_k}}}_1 \\
			&\leq \sum_{k=1}^K w_k \norm{\vd_\vphi - \vd_{b_k}}_1 \leq \sum_{k=1}^K w_k  {\epsilon \, \kappa_{max. k} \over 1- \epsilon \, \kappa_{max, k}} \leq {\epsilon \, \kappa_{max} \over 1-\epsilon\, \kappa_{max}} \rightarrow 0, \quad \text{ as } \epsilon \rightarrow 0.
		\end{split}
	\end{equation*}
	Plugging $w_k = 1/K, \forall \, k \in \cbr{1,\ldots,K}$ into the second to last equation, we get
	\begin{equation*}
		\norm{\vd_\vphi - \vd_b}_1 \leq {1\over K} \sum_{k=1}^K {\epsilon \, \kappa_{max,k}\over 1-\epsilon\, \kappa_{max, k}}\rightarrow 0, \quad \text{ as } \epsilon \rightarrow 0,
	\end{equation*}
	as desired.
\end{proof}
\begin{remark}
\textbf{(1)} We note that $\kappa_{max,k}$ is a constant for fixed $\mT_{b_k}$ and $\kappa_{max}$ is a constant for fixed $\cbr{\pi_{b_k}}_{k=1}^K$. 
\textbf{(2)} In general, $\vd_b^\top \mT_b = \sum_{k=1}^K w_k^2 \vd_{b_k}^\top + \sum_{i \ne j} w_iw_j \vd_{b_i}^\top \mT^{b_j} \ne \vd_b^\top$.
One sufficient condition is $\vd_{b_i}^\top \mT^{b_j} = \vd_{b_i}^\top, \forall \, i, j \implies \vd_{b_i} = \vd_{b_j}, \forall \, i, j  \implies \pi_{b_i} = \pi_{b_j}, \forall \, i, j$, similar to \citet{gail2016}. 
In such case, $\pi_b$ reduces to a single policy, not a mixture, and Theorem~\ref{thm:occup_single} applies.
\end{remark}

The formal statement of Theorem~\ref{thm:ipm_undiscount_informal} is as follows.
\begin{theorem}[Formal Statement of Theorem~\ref{thm:ipm_undiscount_informal}] \label{thm:ipm_undiscount}
    Denote $d_b\br{\vs, \va, \vs'} \triangleq d_b(\vs) \pi_b(\va \given \vs)\mathcal{P}(\vs'\given \vs, \va )$ as in \citet{breakingcurse2018}, then
    \begin{equation*}
        \resizebox{\textwidth}{!}{%
$
     \begin{aligned}
            &D_{\mathcal{G}}\br{d_\vphi\br{\vs, \va}, d_b\br{\vs, \va}} \\
            =& \sup_{g \in \mathcal{G}} \abs{\E_{(\vs, \va) \sim d_\vphi\br{\vs, \va}}\sbr{g\br{\vs, \va}} - \E_{\br{\vs, \va} \sim d_b\br{\vs, \va}}\sbr{g\br{\vs, \va}}} \\
            =& \sup_{f \in \mathcal{F}} \abs{\E_{(\vs, \va) \sim d_b\br{\vs, \va}}\sbr{f\br{\vs, \va}} - \E_{\br{\vs, \va, \vs'} \sim d_b\br{\vs, \va, \vs'}, \va'\sim \pi_\vphi\br{\cdot \given \vs'}}\sbr{f\br{\vs', \va'}}} \\
            =& \sup_{f \in \mathcal{F}} \abs{\E_{(\vs, \va) \sim d_b\br{\vs, \va}}\sbr{f\br{\vs, \va}} - \E_{\vs \sim d_b\br{\vs}, \va \sim \pi_\vphi\br{\cdot \given \vs}}\sbr{f\br{\vs, \va}}} \quad \text{(State-action joint-matching scheme Eq.~\eqref{eq:true_fake_sample})}
            \\
            =& \sup_{f \in \mathcal{F}} \abs{\E_{\vs \sim d_b(\vs)}\sbr{\E_{\va \sim \pi_b\br{ \cdot \given \vs}}\sbr{f\br{\vs, \va}} - \E_{\va \sim \pi_\vphi\br{\cdot \given \vs}}\sbr{f\br{\vs, \va}}}} \quad \text{(Classical policy-matching scheme Eq.~\eqref{eq:cond_true_fake_sample})}
        \end{aligned}
$
}
    \end{equation*}
\end{theorem}

\begin{proof}[Proof of Theorem~\ref{thm:ipm_undiscount}]
    Based on the definition of average reward and average Bellman Equation as in \citet{markovdp1994} and \citet{rlintro2018} Section 10.3, for deterministic reward function $g(\vs_t, \va_t)$ we have 
    \begin{equation*}
        \begin{split}
            R_\pi &\triangleq \lim_{T\rightarrow\infty } \E_{\tau \sim p_\pi(\tau)} \sbr{\frac{1}{T+1} \sum_{t=0}^T g(\vs_t, \va_t)} = \E_{(\vs, \va) \sim d^\pi(\vs, \va) } \sbr{g (\vs, \va)} \\
            Q^\pi(\vs, \va)&\triangleq \E_{\pi, \gP}\left[\sum_{t=0}^{\infty}\left(r_t - R_\pi \right) \given s_0=s, a_0=a\right]\\
            Q^\pi (\vs, \va) &- \E_{\vs' \sim \mathcal{P}\br{\cdot \given \vs, \va}, \va' \sim \pi\br{\cdot \given \vs'}} \sbr{Q^\pi (\vs', \va')} = g(\vs, \va) - R_\pi, \quad \forall \, \vs, \va
        \end{split}
    \end{equation*}

where $\tau$ is the trajectory.

We define $f$ as the the action-value function for policy $\pi_\vphi$ under the reward function $g(\vs,\va)$ and under the original environmental dynamics, satisfying,
\begin{equation*}
    f (\vs, \va) - \E_{\vs' \sim \mathcal{P}\br{\cdot \given \vs, \va}, \va' \sim \pi_\vphi \br{\cdot \given \vs'}} \sbr{f (\vs', \va')} = g(\vs, \va) - R_{\pi_{\vphi}}, \quad \forall \, \vs, \va.
\end{equation*}
Then for the IMP $D_\mathcal{G}\br{d_\vphi\br{\vs, \va}, d_b\br{\vs, \va}}$, we have,
\begin{equation*}
    \resizebox{1\textwidth}{!}{%
$
     \begin{aligned}
        D_{\mathcal{G}}\br{d_\vphi\br{\vs, \va}, d_b\br{\vs, \va}} &= \sup_{g \in \mathcal{G}} \abs{\E_{(\vs, \va) \sim d_\vphi\br{\vs, \va}}\sbr{g\br{\vs, \va}} - \E_{\br{\vs, \va} \sim d_b\br{\vs, \va}}\sbr{g\br{\vs, \va}}} \\
        &= \sup_{f \in \mathcal{F}} \abs{R_{\pi_{\vphi}} - \E_{\br{\vs, \va} \sim d_b\br{\vs, \va}}\sbr{f\br{\vs, \va} - \E_{\vs' \sim \mathcal{P}\br{\cdot \given \vs, \va}, \va'\sim \pi_\vphi\br{\cdot \given \vs'} }\sbr{f\br{\vs', \va'} }+ R_{\pi_{\vphi}}}} \\
        &= \sup_{f \in \mathcal{F}} \abs{\E_{\br{\vs, \va} \sim d_b\br{\vs, \va}}\sbr{f\br{\vs, \va} - \E_{\vs' \sim \mathcal{P}\br{\cdot \given \vs, \va}, \va'\sim \pi_\vphi\br{\cdot \given \vs'} }\sbr{f\br{\vs', \va'} }}} \\
            &= \sup_{f \in \mathcal{F}} \abs{\E_{(\vs, \va) \sim d_b\br{\vs, \va}}\sbr{f\br{\vs, \va}} - \E_{\br{\vs, \va, \vs'} \sim d_b\br{\vs, \va, \vs'}, \va'\sim \pi_\vphi\br{\cdot \given \vs'}}\sbr{f\br{\vs', \va'}}} \\
            &= \sup_{f \in \mathcal{F}} \abs{\E_{(\vs, \va) \sim d_b\br{\vs, \va}}\sbr{f\br{\vs, \va}} - \E_{\vs \sim d_b\br{\vs}, \va \sim \pi_\vphi\br{\cdot \given \vs}}\sbr{f\br{\vs, \va}}}\\
            &= \sup_{f \in \mathcal{F}} \abs{\E_{\vs \sim d_b(\vs)}\sbr{\E_{\va \sim \pi_b\br{ \cdot \given \vs}}\sbr{f\br{\vs, \va}} - \E_{\va \sim \pi_\vphi\br{\cdot \given \vs}}\sbr{f\br{\vs, \va}}}}.
    \end{aligned}
$
}
\end{equation*}
The second equality comes from the fact that under classical regularity conditions for the reward function class $\mathcal{G}$, reward function $g$ and its action-value function $f$ have one-to-one correspondence, with $f$ being the unique solution to the Bellman equation.
The second-to-last equality comes from the fact that for offline datasets collected by sequential rollouts, the marginal distribution of $\vs$ and $\vs'$ are the same. 
In other words, if we randomly draw $\vs \sim d_b(\vs)$, $\vs$ will almost always be the ``next state'' of some other state in the dataset.
\end{proof}
\begin{remark}
The function $f(\vs, \va)$ can be approximated using neural network under the same assumption on the reward function class $\mathcal{G}$ as the validity of neural-network approximation to the solution of the Bellman backup. 
\end{remark}

\begin{theorem}\label{thm:ipm_discount}
    For the discounted visitation frequencies $d_\vphi\br{\vs, \va}, d_b\br{\vs, \va}$, 
    \begin{equation*}
        \begin{split}
            D_{\mathcal{G}}\br{d_\vphi\br{\vs, \va}, d_b\br{\vs, \va}}
            =& \sup_{g \in \mathcal{G}} \abs{\E_{(\vs, \va) \sim d_\vphi\br{\vs, \va}}\sbr{g\br{\vs, \va}} - \E_{\br{\vs, \va} \sim d_b\br{\vs, \va}}\sbr{g\br{\vs, \va}}} \\
            =& \sup_{f \in \mathcal{F}} \abs{\E_{(\vs, \va) \sim d_b\br{\vs, \va}}\sbr{f\br{\vs, \va}} - \E_{\vs \sim d_b\br{\vs}, \va \sim \pi_\vphi\br{\cdot \given \vs}}\sbr{f\br{\vs, \va}}}
        \end{split}
    \end{equation*}
\end{theorem}

\begin{remark}
We note that as is common in GAN and IPM literature, in Theorems~\ref{thm:ipm_undiscount} and \ref{thm:ipm_discount} we assume in theory that $\sup$ can be achieved on the inner maximization of the discriminator when fixing the generator. 
Then we optimize the generator using one-step of gradient descent.
This is implemented often by $k$ (= $2$ in our paper) steps of gradient ascent for discriminator before one-step of generator updates.  
Similar theory-practice gap also appears in the analysis of actor-critic algorithms, where one usually assumes an accurate critic function has been obtained before improving the policy. 
\end{remark}

\begin{proof}[Proof of Theorem~\ref{thm:ipm_discount}]
    Recall that the discounted visitation frequency for a policy $\pi$ is defined as 
\begin{equation*}\textstyle
    d^\pi\br{\vs, \va} \triangleq (1-\gamma) \sum_{t=0}^\infty \gamma^t \Pr_\pi\br{\vs_t = \vs, \va_t = \va}.
\end{equation*}
    Denote $\mT_{b}(\vs' \given \vs) = p_{\pi_b}\left(S_{t+1}=\vs' \given S_t = \vs \right) = \int_{\sA}\mathcal{P}\left(S_{t+1}=\vs' \given S_t=\vs, A_t=\va \right) \pi_b\left(\va \given \vs\right) \,d\va$.
    
    From \citet{breakingcurse2018}, Lemma 3, for $d_b(\vs)$ we have, 
    \begin{equation*}
        \gamma \sum_\vs \mT_{b}(\vs' \given \vs) d_b(\vs) - d_b(\vs') + (1-\gamma) \mu_0(\vs') = 0, \quad \forall \, \vs',
    \end{equation*}
    where $\mu_0$ is the initial state distribution. 
    Multiply $\pi_\vphi(\va'\given \vs')$ on both sides, we get,
    \begin{equation*}
        \gamma \sum_\vs \mT_{b}(\vs' \given \vs) d_b(\vs) \pi_\vphi(\va'\given \vs') - d_b(\vs') \pi_\vphi(\va'\given \vs') + (1-\gamma) \mu_0(\vs')\pi_\vphi(\va'\given \vs') = 0, \quad \forall \, \vs', \va'.
    \end{equation*}
    Denote $d_b\br{\vs, \va, \vs'} \triangleq d_b(\vs) \pi_b(\va \given \vs)\mathcal{P}(\vs'\given \vs, \va )$, for any integrable function $f(\vs, \va)$, we multiply both sides of the above equation by $f(\vs', \va')$ and summing over $\vs', \va'$, we get
    \begin{equation*}
        \gamma \E_{\br{\vs, \va, \vs'} \sim d_b, \va'\sim \pi_\vphi\br{\cdot \given \vs'} } \sbr{f\br{\vs',\va'}} - \E_{\vs\sim d_b, \va\sim \pi_\vphi(\cdot \given \vs)}\sbr{f\br{\vs,\va}}  + (1-\gamma) \E_{\vs\sim \mu_0, \va \sim \pi_\vphi(\cdot \given \vs)} \sbr{f\br{\vs, \va}} = 0.
    \end{equation*}
    
    In estimating the IPM $D_{\mathcal{G}}\br{d_\vphi\br{\vs, \va}, d_b\br{\vs, \va}}$, for any given $g(\vs, \va)$, define 
    \begin{equation*}
        \textstyle
        f(\vs, \va) = g(\vs, \va) + \gamma \E_{\vs' \sim \mathcal{P}(\cdot \given \vs, \va), \va' \sim \pi_\vphi(\cdot \given \vs')} \sbr{f(\vs', \va')}, \forall\, \vs, \va,
    \end{equation*}
    then we have
    \begin{equation*}
     \begin{aligned}
            \E_{\vs \sim \mu_0(\vs), \va \sim \pi_\vphi(\cdot \given \vs)}\sbr{f\br{\vs, \va}} &= \E_{\vs \sim \mu_0(\vs), \va \sim \pi_\vphi(\cdot \given \vs)}\sbr{ g(\vs, \va) + \gamma \E_{\vs' \sim \mathcal{P}(\cdot \given \vs, \va), \va' \sim \pi_\vphi(\cdot \given \vs')} \sbr{f(\vs', \va')} } \\
            &= \E\sbr{\sum_{t=0}^\infty \gamma^t g\br{\vs_t, \va_t} \given \vs_0 \sim \mu_0(\vs), \va_t \sim \pi_\vphi(\cdot \given \vs_t), \vs_{t+1} \sim \mathcal{P}\br{\cdot \given \vs_t, \va_t} } \\
            &= \br{1 - \gamma}^{-1} \E_{\br{\vs, \va} \sim d_{\vphi}\br{\vs,\va}}\sbr{g\br{\vs, \va}},
        \end{aligned}
    \end{equation*}
    based on the definition of $d_{\vphi}(\vs, \va) =: d^{\pi_\vphi}\br{\vs,\va}$ stated above.

    Putting the above three equalities together, for discounted visitation frequencies $d_\vphi\br{\vs, \va}, d_b\br{\vs, \va}$, we have
    \begin{equation*}
        \resizebox{\textwidth}{!}{%
$
     \begin{aligned}
            &D_{\mathcal{G}}\br{d_\vphi\br{\vs, \va}, d_b\br{\vs, \va}} \\
            =& \sup_{g \in \mathcal{G}} \abs{\E_{(\vs, \va) \sim d_\vphi\br{\vs, \va}}\sbr{g\br{\vs, \va}} - \E_{\br{\vs, \va} \sim d_b\br{\vs, \va}}\sbr{g\br{\vs, \va}}} \\
            =& \sup_{f \in \mathcal{F}}\abs{
            (1-\gamma) \E_{\vs \sim \mu_0(\vs), \va \sim \pi_\vphi(\cdot \given \vs)}\sbr{f\br{\vs, \va}} - \E_{(\vs, \va) \sim d_b(\vs, \va)} \sbr{ f(\vs, \va) - \gamma \E_{\vs' \sim \mathcal{P}(\cdot \given \vs, \va), \va' \sim \pi_\vphi(\cdot \given \vs')} \sbr{f(\vs', \va')} }
            }  \\
            =& \sup_{f \in \mathcal{F}} \abs{\E_{(\vs, \va) \sim d_b\br{\vs, \va}}\sbr{f\br{\vs, \va}} - \E_{\vs \sim d_b\br{\vs}, \va \sim \pi_\vphi\br{\cdot \given \vs}}\sbr{f\br{\vs, \va}}}.
        \end{aligned}
$
}
    \end{equation*}
Indeed, $f\br{\vs, \va}$ is the action-value function for policy $\pi_\vphi$ under the reward $g(\vs, \va)$ and under the original environmental dynamics, and can be approximated using neural network under the classical regularity assumptions on the reward function class $\mathcal{G}$.
The second equality comes from the fact that under the classical regularity conditions for $\mathcal{G}$, reward function $g$ and its action-value function $f$ have one-to-one correspondence, with $f$ being the unique solution to the Bellman equation.
\end{proof}

\begin{proof}[Proof of Theorem~\ref{thm:jsd_lower_bound}] \label{pf:jsd_lower_bound}
    For part (1), we have
    \begin{equation*}
     \begin{aligned}
    & \mathrm{JSD}\sbr{\pi_b\br{\va \given \vs} d_b\br{\vs}, \pi_\vphi\br{\va\given \vs} d_b\br{\vs}} \\
    = &\frac{1}{2} \max_D  \left\{\mathbb{E}_{\vs \sim d_b(\vs), \va\sim\pi_b(\cdot \given \vs)} \sbr{\log D\br{\br{\vs,\va}}} + \mathbb{E}_{\tilde{\vs} \sim  d_b(\tilde{\vs}), \tilde{\va} \sim  \pi_\vphi( \cdot \given \tilde{\vs} ) } \sbr{\log(1-D((\tilde{\vs}, \tilde{\va})))} + \log4  \right\}.
\end{aligned}
    \end{equation*}
    From \citet{gan2014}, the optimal discriminator $D^*(\vs, \va)$ for fixed $\pi_\vphi$ is 
    \begin{equation*}
        D^*(\vs, \va) = \frac{\pi_b\br{\va \given \vs} d_b\br{\vs}}{\pi_b\br{\va \given \vs} d_b\br{\vs}+ \pi_\vphi\br{\va\given \vs} d_b\br{\vs}} = \frac{\pi_b\br{\va \given \vs}}{\pi_b\br{\va \given \vs} + \pi_\vphi\br{\va\given \vs} }.
    \end{equation*}
    Moreover, we have
    \begin{equation*}
        \begin{split}
            & \E_{\vs \sim d_b\br{\vs}}\left[\mathrm{JSD}\br{\pi_b(\va \given \vs}, \pi_\vphi(\va \given \vs)) \right] \\
            = &  \frac{1}{2}\left\{\mathbb{E}_{\vs\sim d_b(\vs)} \left[ \max_{D_\vs} \left\{\mathbb{E}_{\va \sim \pi_b(\cdot \given \vs)} [\log D_\vs(\va)]  + \mathbb{E}_{\tilde{\va} \sim \pi_\vphi(\cdot \given \vs)}[\log(1-D_\vs(\tilde{\va}))]  \right\} \right]   + \log 4\right\},
        \end{split}
    \end{equation*}
    where $\forall \, \vs, D_\vs$ denote the state-dependent discriminator to distinguish the state-conditional action distributions  $\pi_b\br{\cdot \given \vs}$ and $\pi_\vphi\br{\cdot \given \vs}$.
    
    For fixed $\pi_\vphi$, the optimal state-dependent discriminator for each state, $D^*_\vs(\va)$, is 
    \begin{equation*}
       D^*_\vs(\va) =  \frac{\pi_b\br{\va \given \vs}}{\pi_b\br{\va \given \vs} + \pi_\vphi\br{\va\given \vs} } = D^*(\vs, \va), \quad \forall\, \vs .
    \end{equation*}
    
    Therefore, we have
    \begin{equation*}
     \begin{aligned}
            & \mathrm{JSD}\sbr{\pi_b\br{\va \given \vs} d_b\br{\vs}, \pi_\vphi\br{\va\given \vs} d_b\br{\vs}} \\
            = &\frac{1}{2} \cbr{\mathbb{E}_{\vs \sim d_b(\vs), \va\sim\pi_b(\cdot \given \vs)} \sbr{\log D^*\br{\br{\vs,\va}}} + \mathbb{E}_{\tilde{\vs} \sim  d_b(\tilde{\vs}), \tilde{\va} \sim  \pi_\vphi( \cdot \given \tilde{\vs} ) } \sbr{\log(1-D^*((\tilde{\vs}, \tilde{\va})))} + \log4}  \\
            = & \frac{1}{2} \cbr{
            \int\int \log D^*\br{\br{\vs,\va}} d\,\pi_b(\va \given \vs) d\, d_b(\vs) + \int\int \log(1-D^*((\tilde{\vs}, \tilde{\va}))) d\,\pi_\vphi( \tilde{\va} \given \tilde{\vs} ) d\,d_b(\tilde{\vs}) + \log 4
            }\\
            = & \frac{1}{2} \cbr{
            \int\int \log D^*_\vs\br{\va} d\,\pi_b(\va \given \vs) d\, d_b(\vs) + \int\int \log(1-D^*_\vs(\tilde{\va})) d\,\pi_\vphi( \tilde{\va} \given \vs ) d\,d_b(\vs) + \log 4 
            }\\
            =&  \frac{1}{2} \cbr{
            \int \sbr{\int \log D^*_\vs\br{\va} d\,\pi_b(\va \given \vs)  + \int \log(1-D^*_\vs(\tilde{\va})) d\,\pi_\vphi( \tilde{\va} \given \vs )} d\,d_b(\vs) + \log 4 
            }\\
            =&  \frac{1}{2}\left\{\mathbb{E}_{\vs\sim d_b(\vs)} \left[ \max_{D_\vs} \left\{\mathbb{E}_{\va \sim \pi_b(\cdot \given \vs)} [\log D_\vs(\va)]  + \mathbb{E}_{\tilde{\va} \sim \pi_\vphi(\cdot \given \vs)}[\log(1-D_\vs(\tilde{\va}))]  \right\} \right]   + \log 4\right\} \\
            =& \E_{\vs \sim d_b\br{\vs}}\left[\mathrm{JSD}\br{\pi_b(\va \given \vs}, \pi_\vphi(\va \given \vs)) \right].
        \end{aligned}
    \end{equation*}
    
    For part (2), in theory, for the state-action joint-matching scheme we have
    \begin{equation*}
     \begin{aligned}
            & \mathrm{JSD}\sbr{\pi_b\br{\va \given \vs} d_b\br{\vs}, \pi_\vphi\br{\va\given \vs} d_b\br{\vs}} \\
    = &\frac{1}{2} \max_D  \left\{\mathbb{E}_{\vs \sim d_b(\vs), \va\sim\pi_b(\cdot \given \vs)} \sbr{\log D\br{\br{\vs,\va}}} + \mathbb{E}_{\tilde{\vs} \sim  d_b(\tilde{\vs}), \tilde{\va} \sim  \pi_\vphi( \cdot \given \tilde{\vs} ) } \sbr{\log(1-D((\tilde{\vs}, \tilde{\va})))} + \log4  \right\},\\
\approx & \frac{1}{2} \max_D  \left\{\frac{1}{N}\sum_{i=1}^N [\log D((\vs_i,\va_i))] + \frac{1}{N}\sum_{j=1}^N [\log(1-D((\tilde{\vs}_j, \tilde{\va}_j))) ] + \log4  \right\}
        \end{aligned}
    \end{equation*}
where we draw 
$
(\vs_i,\va_i)\stackrel{iid}\sim d_b(\vs)\pi_b(\va \given \vs)
$ for $i=1,\ldots,N$
and 
$
(\tilde \vs_j,\tilde \va_j)\stackrel{iid}\sim d_b(\vs)\pi_\vphi(\va \given \vs)
$ for $j=1,\ldots,N$.

By contrast, for the standard policy-matching scheme, in theory we have
\begin{equation*}
    \begin{split}
        & \E_{\vs \sim d_b\br{\vs}}\left[\mathrm{JSD}\br{\pi_b(\va \given \vs}, \pi_\vphi(\va \given \vs)) \right] \\
            = &  \frac{1}{2}\left\{\mathbb{E}_{\vs\sim d_b(\vs)} \left[ \max_{D_\vs} \left\{\mathbb{E}_{\va \sim \pi_b(\cdot \given \vs)} [\log D_\vs(\va)]  + \mathbb{E}_{\tilde{\va} \sim \pi_\vphi(\cdot \given \vs)}[\log(1-D_\vs(\tilde{\va}))]  \right\} \right]   + \log 4\right\}\\
\ge &  \frac{1}{2} \max_{D_{\vtheta}}  \left\{\mathbb{E}_{\vs\sim d_b(\vs)} \left[ \mathbb{E}_{\va\sim\pi_b(\cdot \given \vs)} [\log D_{\vtheta}(\va \given \vs)]  + \mathbb{E}_{\tilde{\va} \sim \pi_\vphi(\cdot \given \vs)}[\log(1-D_{\vtheta}(\tilde{\va} \given \vs))]  \right]   + \log 4\right\}
\\
\approx& \frac{1}{2} \max_{D_{\vtheta}}  \left\{\frac{1}{N}\sum_{i=1}^N [\log D_{\vtheta}(\va_i\given \vs_i) +\log(1-D_{\vtheta}(\tilde{\va}_i \given \vs_i)) ] + \log4  \right\}
    \end{split}
\end{equation*}
where we draw $(\vs_i,\va_i,\tilde \va_i)\stackrel{iid}\sim d_b(\vs) \pi_b(\va\given \vs) \pi_\vphi(\tilde \va \given \vs)$ for $i=1,\ldots,N$. 
Note that the inequality arises due to amortizing all state-dependent optimal discriminator $D_\vs^*$ into a single parametric discriminator $D_{\vtheta}(\cdot \given \vs)$ and exchanging the orders of expectation and maximization. 
Thus, in theory, the classical policy-matching scheme is optimizing towards a lower bound of its desired objective.
\end{proof}

\begin{theorem} \label{thm:fd_lower_bound}
    For the integral probability metrics $D_{\mathcal{G}}$, under the state-action joint-matching scheme, the discriminator is optimized towards estimating the desired $D_{\mathcal{G}}$; while under the classical policy-matching scheme, the discriminator is optimized towards estimating a lower bound of the desired $D_{\mathcal{G}}$.
\end{theorem}
\begin{proof}[Proof of Theorem~\ref{thm:fd_lower_bound}]
    Here we assume dealing with general IPM, $D_{\mathcal{G}}$.
    
    The goal of discriminator-learning under the joint-matching scheme is
    \begin{equation*}
        \begin{split}
            & D_{\mathcal{G}}\sbr{\pi_b\br{\va \given \vs} d_b\br{\vs}, \pi_\vphi\br{\va\given \vs} d_b\br{\vs}} \\
    = & \sup_{g \in \mathcal{G}}  \abs{\mathbb{E}_{\vs \sim d_b(\vs), \va\sim\pi_b(\cdot \given \vs)} \sbr{g\br{\br{\vs,\va}}} - \mathbb{E}_{\tilde{\vs} \sim  d_b(\tilde{\vs}), \tilde{\va} \sim  \pi_\vphi( \cdot \given \tilde{\vs} ) } \sbr{g((\tilde{\vs}, \tilde{\va}))}},\\
\approx & \sup_{g \in \mathcal{G}} \abs{\frac{1}{N}\sum_{i=1}^N [g((\vs_i,\va_i))] - \frac{1}{N}\sum_{j=1}^N [g((\tilde{\vs}_j, \tilde{\va}_j)) ]}
        \end{split}
    \end{equation*}
where we draw 
$
(\vs_i,\va_i)\stackrel{iid}\sim d_b(\vs)\pi_b(\va \given \vs)
$ for $i=1,\ldots,N$
and 
$
(\tilde \vs_j,\tilde \va_j)\stackrel{iid}\sim d_b(\vs)\pi_\vphi(\va \given \vs)
$ for $j=1,\ldots,N$.
Thus the joint-matching scheme optimizes the discriminator towards the desired objective.

In theory, the goal of discriminator-learning under the classical policy-matching scheme is
\begin{equation*}
    \begin{split}
        & \E_{\vs \sim d_b\br{\vs}}\left[D_{\mathcal{G}}\br{\pi_b(\va \given \vs}, \pi_\vphi(\va \given \vs)) \right] \\
            = &  \mathbb{E}_{\vs\sim d_b(\vs)} \left[ \sup_{g_\vs \in \mathcal{G}} \abs{\mathbb{E}_{\va \sim \pi_b(\cdot \given \vs)} [g_\vs(\va)]  - \mathbb{E}_{\tilde{\va} \sim \pi_\vphi(\cdot \given \vs)}[g_\vs(\tilde{\va})]} \right]   \\
\ge &  \sup_{g_\vtheta \in \mathcal{G}} \cbr{ \mathbb{E}_{\vs\sim d_b(\vs)} \abs{\mathbb{E}_{\va\sim\pi_b(\cdot \given \vs)} [g_\vtheta(\va \given \vs)]  - \mathbb{E}_{\tilde{\va} \sim \pi_\vphi(\cdot \given \vs)}[g_\vtheta(\tilde{\va} \given \vs)]} }
\\
\ge &  \sup_{g_\vtheta \in \mathcal{G}} \cbr{ \abs{\mathbb{E}_{\vs\sim d_b(\vs)} \sbr{\mathbb{E}_{\va\sim\pi_b(\cdot \given \vs)} [g_\vtheta(\va \given \vs)]  - \mathbb{E}_{\tilde{\va} \sim \pi_\vphi(\cdot \given \vs)}[g_\vtheta(\tilde{\va} \given \vs)]}} }
\\
\approx& \sup_{g_\vtheta \in \mathcal{G}} \abs{\frac{1}{N}\sum_{i=1}^N \sbr{g_\vtheta(\va_i\given \vs_i) - g_\vtheta(\tilde{\va}_i \given \vs_i)}  }
    \end{split}
\end{equation*}
where we draw $(\vs_i,\va_i,\tilde \va_i)\stackrel{iid}\sim d_b(\vs) \pi_b(\va\given \vs) \pi_\vphi(\tilde \va \given \vs)$ for $i=1,\ldots,N$ and $g_\vs$ is the state-dependent witness function for distinguishing $\pi_b\br{\cdot \given \vs}$ and $\pi_\vphi\br{\cdot \given \vs}$. 
Note that the inequality arises due to amortizing all state-dependent optimal $g_\vs$ into a single parametric witness function $g_{\vtheta}(\cdot \given \vs)$ and exchanging the orders of expectation and maximization. 
Thus, in theory, the classical policy-matching scheme optimizes the discriminator towards estimating a lower bound of the desired objective.
\end{proof}

\section{Technical Details} \label{sec:tech_detail}
\subsection{Toy Experiment} \label{sec:toy_detail}
Denote the total sample size as $N_{\mathrm{total}}$, we follow the convention to construct the eight-Gaussian dataset as in Algorithm~\ref{alg:eight_gaussian}.
Here we use $N_{\mathrm{total}}=2000$.

\begin{algorithm}[H]
\captionsetup{font=small}
\caption{\small Constructing the Eight-Gaussian Dataset}
\begin{algorithmic}
\label{alg:eight_gaussian}
    \STATE {\bfseries Input:} Total sample size $N_{\mathrm{total}}$.
    \STATE {\bfseries Output:} Generated dataset $\sD_{\mathrm{Gaussian}}$.
    \WHILE{Dataset size $< N_{\mathrm{total}}$}
    \STATE Draw a random center $\vc$ uniformly
    \begin{equation*} \textstyle
        \vc \sim \cbr{\br{\sqrt{2}, 0}, \br{-\sqrt{2}, 0}, \br{0, \sqrt{2}}, \br{0, -\sqrt{2}}, \br{1,1}, \br{1,-1}, \br{-1,1}, \br{-1,-1}}.
    \end{equation*}
    \STATE Sample datapoint $\vx = (x,y) \sim \gN\br{\vc, 2 \times 10^{-4} \cdot \mI_2}$.
    Store $\vx$ in the dataset.
    \ENDWHILE
\end{algorithmic}
\end{algorithm}

We are interested in the 2-D eight-Gaussian dataset because \textbf{(a)} the conditional distribution of $p(y\given x)$ is multi-modal in many $x$; and \textbf{(b)} interpolation is needed to fill-in the blanks between Gaussian-centers, where a smooth-interpolation into a circle is naturally expected.

To rephrase this dataset into offline reinforcement learning setting, we define $x$ as state and the corresponding $y$ as action.
Note that in the behavior cloning task, the information of reward, next state, and the episodic termination is not required.
Hence, the generated dataset $\sD_{\mathrm{Gaussian}}$ can serve as an offline RL dataset readily applicable to train behavior cloning policies. 

In order to compare the ability to approximate the behavior policy by the KL loss and the JSD loss, the Gaussian policy and the implicit policy, the classical policy-matching scheme and the proposed state-action joint-matching, we fit a conditional VAE (``CVAE''), a Gaussian generator conditional GAN (``G-CGAN'') and a conditional GAN (``CGAN'') using the policy-matching approach similar to \citet{brac2019}.
We fit a conditional GAN (``GAN'') using basic state-action joint-matching strategy.
As discussed in \Secref{sec:basic}, the major distinction between ``CGAN'' and ``GAN'' is that the former uses the same states in constructing the generator samples and the data samples while the later resamples states.

The network architecture of our conditional VAE is as follows.
\begin{multicols}{2}
[Conditional Variational Auto-encoder (CVAE) in Toy Experiment]
\begin{minipage}{0.45\textwidth}
Encoder
\begin{verbatim}
Linear(state_dim+action_dim, H)
BatchNorm1d(H)
ReLU
Linear(H, H//2)
BatchNorm1d(H//2)
ReLU
mean = Linear(H//2, latent_dim)
log_std = Linear(H//2, latent_dim)
\end{verbatim}
\end{minipage}

\begin{minipage}{0.45\textwidth}
Decoder
\begin{verbatim}
Linear(state_dim+latent_dim, H)
BatchNorm1d(H)
ReLU
Linear(H, H//2)
BatchNorm1d(H//2)
ReLU
Linear(H//2, action_dim)
\end{verbatim}
\end{minipage}
\end{multicols}
with hidden dimension ${\tt H} = 100$ and latent dimension ${\tt latent\_dim} = 50$.
CVAE is trained for $1200$ epochs with a mini-batch size of $100$ and random seed $0$, using the mean-squared-error as the reconstruction loss, and the Gaussian-case closed-form formula in \citet{vae2013} for the KL term.

The network architecture of our conditional GAN, used in ``CGAN'' and ``GAN,'' is as follows.
\begin{multicols}{2}
[Conditional Generative Adversarial Nets (CGAN) in Toy Experiment]
\begin{minipage}{0.45\textwidth}
Generator
\begin{verbatim}
Linear(state_dim+z_dim, H)
BatchNorm1d(H)
ReLU
Linear(H, H//2)
BatchNorm1d(H//2)
ReLU
Linear(H//2, action_dim)
\end{verbatim}
\end{minipage}

\begin{minipage}{0.45\textwidth}
Discriminator
\begin{verbatim}
Linear(state_dim+action_dim, H)
LeakyReLU(0.1)
Linear(H, H//2)
LeakyReLU(0.1)
Linear(H//2, 1)
\end{verbatim}
\end{minipage}
\end{multicols}
where the structure of {\tt BatchNorm1d, LeakyReLU} follows \citet{dcgan2016}.
Here we again use ${\tt H} = 100, {\tt z\_dim} = 50$.
Conditional GAN is trained for $2000$ epochs with a mini-batch size of 100 and random seed $0$.
We follow \cite{dcgan2016} to train CGAN using Adam optimizer with $\beta_1 = 0.5$.

The network architecture of our Gaussian-generator version of conditional GAN is as follows.
\begin{verbatim}
Generator
Linear(state_dim, H)
BatchNorm1d(H)
ReLU
Linear(H, H//2)
BatchNorm1d(H//2)
ReLU
mean = Linear(H//2, action_dim), log_std = Linear(H//2, action_dim)
\end{verbatim} 
with the discriminator and other technical details the same as CGAN.
This Gaussian-generator version of CGAN is again trained for $2000$ epochs with a mini-batch size of $100$, random seed $0$, and $\beta_1 = 0.5$ in the Adam optimizer.

Our test set is formed by a random sample of $2000$ new states ($x$) from $\sbr{-1.5, 1.5}$ together with the states in the training set.
The performance on the test set thus shows both the concentration on the eight centers and the smooth interpolation between centers, which translates into a good and smooth fit to the behavior policy.
Figure~\ref{fig:toy} shows the training set (``Truth'') and the kernel-density-estimate plot of each methods.

\subsection{Reinforcement Learning Experiments}\label{sec:rl_details}

{\bfseries Computing Facility.}
Our experiments are run on a computing server that has four Nvidia GeForce GTX 1080 Ti GPUs.

{\bfseries Datasets.}
We use the continuous control tasks provided by the D4RL dataset \citep{fu2021d4rl} to conduct algorithmic evaluations.
Due to limited computational resources, we select therein the ``medium-expert,'' ``medium-replay,'' and ``medium'' datasets for the Hopper, HalfCheetah, Walker2d tasks in the Gym-MuJoCo domain, which are commonly used 
benchmarks in prior work \citep{bcq2019, bear2019, brac2019, cql2020}.
We follow the literature \citep{mabe2021,decisiontrans2021,iql2021} to not test on the ``random'' and ``expert'' datasets as they are known as less practical \citep{bremen2021} and can be respectively solved by directly using standard off-policy RL algorithms \citep{rem2020} and the behavior cloning algorithms.
We note that in offline RL applications, one typically know the quality of the offline datasets, \textit{e.g.}, whether it is collected by random or expert policy.
Further, a comprehensive benchmarking of prior offline-RL algorithms on the ``expert” datasets is currently unavailable in the literature, which is out of the scope of this paper.
Apart from the Gym-MuJoCo domain, we also consider the Maze2D domain\footnote{We use the tasks ``maze2d-umaze,'' ``maze2d-medium,'' and ``maze2d-large.''} for the non-Markovian data-collecting policy, and the Adroit tasks\footnote{We use the tasks ``pen-human,'' ``pen-cloned,'' ``pen-expert,'' and ``door-expert.''} \citep{adroit2018} for their sparse reward-signal and high dimensionality.

{\bfseries Evaluation Protocol.}
In all the experiments, we follow \citet{fu2021d4rl} to use the ``v0'' version of the datasets in the Gym-MuJoCo and Adroit domains.
In our preliminary study, we find that the results of some baseline algorithms can be unstable across epochs in some datasets, even towards the end of training.
To reduce the instability in evaluation, for our algorithm, we report the mean and standard deviation of the last five rollouts across five random seeds $\cbr{0,1,2,3,4}$.
For the baselines that we rerun, we follow \citet{fu2021d4rl} to rerun under three random seeds $\cbr{0,1,2}$ and under the recommended hyperparameter setting, including per-dataset tuned hyperparameters if available.
We run our method for $1000$ epochs, where each epoch consists of $1000$ mini-batch stochastic gradient descent steps.
We rollout our method and baselines for $10$ episodes after each epoch of training.

\textbf{Terminal states.} In practice, the rollouts contained in the offline dataset have finite horizon, and thus special treatment is needed per appearance of the terminal states in calculating the Bellman update target.
We follow the standard treatment \citep{dqn2013,rlintro2018} to define the update target $y$ as 
\begin{equation*}
    \widetilde{Q}\br{\vs,\va} = \begin{cases}
    r\br{\vs,\va} + \gamma \widetilde{Q}'\br{\vs', \va'} & \text{ if $\vs$ is a non-terminal state} \\
    r\br{\vs,\va} & \text{ if $\vs$ is a terminal state}
    \end{cases},
\end{equation*}
where $\widetilde{Q}'\br{\vs', \va'}$ refers to the expectation term in Eq.~\eqref{eq:critic_target} for basic algorithm (\Secref{sec:basic}) or the expectation term in Eq.~\eqref{eq:critic_smooth_target} for the enhanced versions with state-smoothing at the Bellman Backup (\Secref{sec:enhance_comp}).

\textbf{Implicit policy implementation.} For simplicity, we follow \citet{samplegenerative2016} to choose the noise distribution $p_\vz\br{\vz}$ as the multivariate standard normal distribution, where the dimension of $\vz$ is conveniently chosen as $\mathrm{dim}\br{\vz} = {\tt min(10, \, state\_dim // 2)}$.
To sample from the implicit policy, for each state $\vs$, we first sample independently $\vz \sim \gN\br{\vzero, \mI}$.
We then concatenate $\vs$ with $\vz$ and feed the resulting $\sbr{\vs, \vz}$ into the deterministic policy network to generate stochastic actions.
To sample from a small region around the next state $\vs'$ (\Secref{sec:enhance_comp}), we keep the original $\vs'$ and repeat it additionally $N_B$ times.
For each of the $N_B$ replications, we add an independent Gaussian noise $\vepsilon \sim \gN(\vzero, \sigma_B^2 \mI)$.
The original $\vs'$ and its $N_B$ noisy replications are then fed into the implicit policy to sample the corresponding action.

Due to limited computational resources, we leave a fine-tuning of the noise distribution $p_\vz\br{\vz}$, the network architectures, and the optimization hyperparameters for future work, which also leaves room for further improving our results.

\textbf{Warm-start step.} For a more stable training of the policy, we adopt the warm start strategy \citep{cql2020,idac2020}.
Specifically, in the first $N_{\mathrm{warm}}$ epochs, the policy is trained to minimize $\gL_g$ only. 
The learning rate $\eta_\vphi$ in the warm-start step is the same as the following epochs that also maximize the expected Q-values.

\subsubsection{GAN Joint Matching} \label{sec:tech_gan}
In approximately matching the JSD between the current and the behavior policies via GAN, a crucial step is to stably and effectively train the GAN structure.
With training techniques developed over the years, GAN can be stably trained with satisfactory mode coverage on data with moderate dimension, \textit{e.g.}, Figure~\ref{fig:toy}.
We adopt the following tricks from literature.
\begin{itemize}[leftmargin=*]
    \item To provide stronger gradients early in training, rather than training the policy $\pi_\vphi$ to minimize
    \begin{equation*} \textstyle
        \E_{\vx} \sbr{\log\br{1-D_\vw(\vx)}}
    \end{equation*}
    we follow \citet{gan2014} to train $\pi_\vphi$ to maximize
    \begin{equation*} \textstyle
        \E_{\vx} \sbr{\log\br{D_\vw(\vx)}}
    \end{equation*}
    \item Motivated by \citet{dcgan2016}, we use LeakyReLU activation in both the generator and discriminator, with default {\tt negative\_slope=0.01}.
    \item To stabilize the training, we follow \citet{dcgan2016} to use a reduced momentum term $\beta_1 = 0.4$ in the Adam optimizer \citep{adam2014}.
    \item We follow \citet{dcgan2016} to use actor and discriminator learning rate $\eta_\vphi = \eta_\vw = 2 \times 10^{-4}$.
    \item To avoid overfitting of the discriminator, we are motivated by \citet{improvetechgan2016} and \citet{gantutorial2017} to use one-sided label smoothing with soft and noisy labels.
    Specifically, the labels for the data sample $\vy$ is replaced with a random number between $0.8$ and $1.0$, instead of the original $1$.
    No label smoothing is applied for the generator sample $\vx$, and therefore their labels are all $0$.
    \item The loss function for training the discriminator in GAN is the Binary Cross Entropy between the labels and the outputs from the discriminator.
\end{itemize}

Furthermore, motivated by TD3 \citep{td32018} and GAN (\Secref{sec:gan_prelim}), we update $\pi_\vphi(\boldsymbol{\cdot}\given \vs)$ once per $k$ updates of the critics and discriminator. 

Table \ref{table:gan_param} shows the hyperparameters for our GAN joint-matching framework.
Note that \emph{several simplifications are made to minimize hyperparameter tuning}, such as fixing $\eta_\vphi = \eta_\vw$ as in \citet{dcgan2016} and $\sigma_B = \sigma_J \triangleq \sigma$.

We comment that many of these hyperparameters can be set based on literature, for example, we use $\eta_\vphi = \eta_\vw = 2 \times 10^{-4}$ as in \citet{dcgan2016}, $\eta_\vtheta = 3 \times 10^{-4}$ and $N_{\mathrm{warm}} = 40$ as in \citet{cql2020}, $\lambda=0.75$ as in \citet{bear2019}, and policy frequency $k=2$ as in \citet{td32018}.
Unless specified otherwise, the same hyperparameters are used across all datasets.

\begin{table}[ht]
\captionsetup{font=small}
\caption{\small Default Hyperparameters for GAN joint matching.} \label{table:gan_param}
\centering
\begin{tabular}{@{}ll@{}}
\toprule
Hyperparameter & Value \\ \midrule
Optimizer      & Adam \cite{adam2014}  \\
Learning rate $\eta_\vtheta$ &  $3 \times 10^{-4}$  \\ 
Learning rate $\eta_\vphi$, $\eta_\vw$ & $2 \times 10^{-4}$      \\ 
Log Lagrange multiplier $\log \alpha$ for non-Adroit datasets & $4.0$ \\
Log Lagrange multiplier $\log \alpha$ for Adroit datasets & $8.0$ \\
Evaluation frequency & $10^3$ \\
Training iterations & $10^6$ \\
Batch size & $512$ (as in \citet{optidice2021}) \\
Discount factor & $0.99$ \\
Target network update rate $\beta$ & $0.005$ \\
Weighting for clipped double Q-learning $\lambda$ & $0.75$ \\
Noise distribution $p_\vz(\vz)$ & $\gN\br{\vzero, \mI}$ \\
Standard deviations for state smoothing $\sigma_B = \sigma_J\triangleq \sigma$ & $3 \times 10^{-4}$ \\
Number of smoothed states in Bellman backup $N_B$ & $50$ \\
Number of epochs for warm start $N_{\mathrm{warm}}$ & $40$ \\
Policy frequency $k$ & $2$ \\
Random seeds & $\cbr{0,1,2,3,4}$ \\
\bottomrule
\end{tabular}
\end{table}

Below we state the network architectures of the actor, critic, and discriminator.
Note that we use a pair of critic networks with the same architecture to perform clipped double Q-learning.

\begin{multicols}{2}
  \begin{minipage}{0.45\textwidth}
    Actor
\begin{verbatim}
Linear(state_dim+noise_dim, 400)
LeakyReLU
Linear(400, 300)
LeakyReLU
Linear(300, action_dim)
max_action * tanh
\end{verbatim}
  \end{minipage}
  
\begin{minipage}{0.45\textwidth}
Critic 
\begin{verbatim}
Linear(state_dim+action_dim, 400)
LeakyReLU
Linear(400, 300)
LeakyReLU
Linear(300, 1)
\end{verbatim}
\end{minipage}
\end{multicols}

\begin{multicols}{2}
[Discriminator in GAN]
\begin{minipage}{0.45\textwidth}
\begin{verbatim}
Linear(state_dim+action_dim, 400)
LeakyReLU
Linear(400, 300)
LeakyReLU
Linear(300, 1)
Sigmoid
\end{verbatim}
\end{minipage}

\begin{minipage}{0.45\textwidth}
Note that 
all the LeakyReLU activation uses the default {\tt negative\_slope=0.01}.
\end{minipage}
\end{multicols}

\subsubsection{Construction of the Penalty Coefficient in GAN-Joint-$\alpha$}  \label{sec:tech_alpha}
We combined Eq.~\ref{eq:policy_target} with the definition of the penalty coefficient in TD3+BC \citep{td3bc2021} as
\begin{equation*} \label{eq:policy_target_td3} 
    \arg\min_{\vphi} - \lambda \E_{\vs \sim \sD} \E_{\va \sim \pi_\vphi(\boldsymbol{\cdot}\given \vs)}\sbr{\min_{j=1,2} Q_{\vtheta_j}(\vs, \va)} + \gL_g(\vphi), \quad \lambda = \frac{\alpha}{Q_{avg}}
\end{equation*}
where we use \emph{$\alpha=10$ across all datasets}.
$Q_{avg}$ is soft-updated based on each mini-batch $\gB$ as
\begin{equation*}
    Q_{avg} = \beta \cdot \frac{1}{\abs{\gB}} \sum_{\br{\vs, \va} \in \gB} \abs{Q\br{\vs, \va}} + (1 - \beta) \cdot Q_{avg}.
\end{equation*}
Here we modify the update scheme of $Q_{avg}$ in \citet{td3bc2021} to allow for soft-update.

\subsubsection{Results of CQL} \label{sec:tech_cql}
We note that the official CQL GitHub repository does not provide hyperparameter settings for the Maze2D and Adroit domain of tasks.
For datasets in these two domains, we train a CQL agent using five hyperparameter settings: four recommended Gym-MuJoCo settings and one recommended Ant-Maze setting.
We then calculate the average normalized-return over the random seeds $\cbr{0,1,2}$ for each hyperparameter settings and per-dataset select the best results from these five settings.
We comment that this per-dataset tuning may give CQL some advantage on the Maze2D and Adroit domains, and is a compensation for the missing of recommended hyperparameter settings.
For the Gym-MuJoCo domain, we use the recommentation by \citet{cql2020}.

\subsubsection{Ablation Study on Gaussian Policy}\label{sec:tech_gaussian_policy}
The network architecture of the Gaussian policy variant that we used in the ablation study (\Secref{sec:abaltion_study}) follows the common practice \citep{sac2018,cql2020}.
\begin{verbatim}
Gaussian Policy
Linear(state_dim, 400)
LeakyReLU
Linear(400, 300)
LeakyReLU
mean = Linear(300, action_dim)
log_std = Linear(300, action_dim)
\end{verbatim}
Critics and discriminator are the same as the implicit policy variant (Appendix \ref{sec:tech_gan}).

For action-selection from the Gaussian policy, a given state $\vs$ is first mapped to the mean $\vmu(\vs)$ and standard deviation vector $\vsigma(\vs)$. 
A raw action is sampled as $\va_{\mathrm{raw}} \sim \gN\br{\vmu(\vs), \mathrm{diag}\br{\vsigma^2(\vs)}}$.
Finally, $\va_{\mathrm{raw}}$ is mapped into the action space as $\mathrm{max\_action} \times \tanh(\va_{\mathrm{raw}})$.

For fair comparison, other technical details, including the training procedure and hyperparameter setting, are exactly the same as the implicit policy case (Appendix \ref{sec:tech_gan}).

\end{document}